\documentclass[10pt]{article}
\newif\ifsubmission
\submissionfalse
\newif\ifshowack
\showackfalse
\ifsubmission
  \usepackage{tmlr}
\else
  \usepackage[preprint]{tmlr}
\fi
\usepackage{hyperref}
\usepackage{url}

\usepackage[T1]{fontenc}
\usepackage{microtype}
\usepackage{graphicx}
\usepackage{subcaption}
\usepackage{booktabs,multirow,makecell}
\usepackage{amsmath,amssymb,amsthm,mathtools,bm}
\usepackage{xcolor}
\usepackage{enumitem}
\usepackage{adjustbox}
\usepackage{algorithm}
\makeatletter\let\AND\@undefined\makeatother   %
\usepackage{algorithmic}
\usepackage{aliascnt}
\usepackage[capitalise,nameinlink]{cleveref}
\definecolor{sfgblue}{HTML}{2A78D6}
\definecolor{aoiorange}{HTML}{EB6834}
\definecolor{graytext}{HTML}{52514E}

\newtheorem{theorem}{Theorem}
\newaliascnt{proposition}{theorem}\newtheorem{proposition}[proposition]{Proposition}\aliascntresetthe{proposition}
\newaliascnt{lemma}{theorem}\aliascntresetthe{lemma}
\newaliascnt{corollary}{theorem}\newtheorem{corollary}[corollary]{Corollary}\aliascntresetthe{corollary}
\crefname{proposition}{Proposition}{Propositions}
\crefname{lemma}{Lemma}{Lemmas}
\crefname{corollary}{Corollary}{Corollaries}
\theoremstyle{definition}
\newtheorem{assumption}{Assumption}
\newtheorem{remark}{Remark}
\newtheorem{principle}{Principle}
\newtheorem{definition}{Definition}
\usepackage{tcolorbox}
\tcolorboxenvironment{principle}{colback=sfgblue!5,colframe=sfgblue!60,boxrule=0.5pt,arc=1mm,left=4pt,right=4pt,top=2pt,bottom=2pt}
\crefname{definition}{Definition}{Definitions}
\crefname{assumption}{Assumption}{Assumptions}
\crefname{principle}{Principle}{Principles}
\newcommand{\E}{\mathbb E}
\newcommand{\sfg}{SFG}
\newcommand{\sfgl}{SFG-L}
\newcommand{\sfgcut}{SFG-Cut}

\title{Discount Yourself: Self-inclusive Freshness-weighted Gossip for Decentralised Learning under Client Outages}

\author{\name Qiyuan Wang \email Qiyuan.Wang@glasgow.ac.uk \\
        \name Ke Xiao \email Ke.Xiao@glasgow.ac.uk \\
        \name Christos Anagnostopoulos \email Christos.Anagnostopoulos@glasgow.ac.uk \\
        \addr University of Glasgow}
\ifsubmission\def\pdfauthors{Anonymous}\else\def\pdfauthors{Qiyuan Wang, Ke Xiao, Christos Anagnostopoulos}\fi
\hypersetup{
  hidelinks,
  pdftitle={Discount Yourself: Self-inclusive Freshness-weighted Gossip for Decentralised Learning under Client Outages},
  pdfauthor={\pdfauthors},
  pdfsubject={Decentralised federated learning under client outages},
  pdfkeywords={decentralised learning, federated learning, client outages, staleness, aggregation, gossip learning}
}

\begin{document}
\maketitle

\begin{abstract}
Decentralised federated learning (DFL) trains personalised models on devices that average with their neighbours rather than
with a central server. Real devices, however, are often offline for hours or days, and a client returning from an outage can
recover only by aggregating the models its neighbourhood holds. We observe that the stalest of these models is usually the
client's \emph{own}, a case that existing staleness-aware rules, which discount only the models a client receives, do not
address. We formalise aggregation after an outage as estimating the state of a fully available network and derive three
principles. First, the receiver must be discounted like any neighbour. Every rule that weights it at least as much as each
neighbour has an error that grows without bound with the outage, whereas discounting all candidates by relative staleness keeps
the error bounded. Second, staleness must be measured as inherited lag rather than age, since local training resets a model's
age but not the lag it absorbed from stale neighbours. We introduce the \emph{effective lag}, which tracks this quantity at the
cost of one scalar per message. Third, the discount must be steep, because absorbed lag persists in every later round. These
principles define Self-inclusive Freshness-weighted Gossip (SFG), a family of drop-in aggregation rules with the communication
and computation of decentralised federated averaging, whose guarantees hold for the stale receiver itself. On three image
benchmarks and eight outage processes, including real smartphone availability traces, the lag-based members SFG-L and SFG-Cut
consistently outperform decentralised federated averaging and recent age-of-information rules, and ablations show that each
principle is necessary. The analysis also explains why tuned limited merging from gossip learning is effective.
\end{abstract}

\section{Introduction}
\label{sec:intro}

Federated learning (FL) trains models on data that never leaves the devices that hold it, such as phones, vehicles and edge
nodes \citep{bonawitz2019towards,kairouz2021advances}. These devices are available only part of the time. In the public FLASH
smartphone trace \citep{yang2021characterizing}, 30\% of devices never satisfy the usual ``screen off, charging and on WiFi''
participation rule. Devices that satisfy it at least 5\% of the time do so only 18\% of the time on average, and their outages
last 7.6 hours (Appendix~\ref{app:flash}). For on-device learning, outages are the norm rather than the exception.

Server-based FL has a mature theory of such partial participation
\citep{gu2021mifa,jhunjhunwala2022fedvarp,wang2024fedau,xiang2024fedawe}. Its remedies share one convenience. A server always
holds a fresh global model, and a returning client can simply download it.

\emph{Decentralised} federated learning (DFL) gives up this convenience. Clients average models with their neighbours on a
communication graph and personalise them by local training \citep{lian2017dpsgd,koloskova2020unified,sun2023dfedavgm,dai2022dispfl}.
There is no coordinator and no single point of failure. A client that returns after $\tau$ rounds offline, however, has nothing
fresh to download. It can only recover by averaging with its neighbours, who may be stale themselves. In the same round, it shares its own outdated model with them. The convergence of gossip learning is known to degrade with
the number of such inaccessible nodes and the duration of their absence \citep{liu2024gossipinaccessibility}.

\textbf{The stalest model in the room is your own.} Who, then, is stale when a client returns? Its neighbours have mostly kept
training while it was away. The stalest candidate in its aggregation is therefore usually the returning client itself. On
three real smartphone traces, this holds in 85--93\% of returns (Appendix~\ref{app:stalest}).

Existing staleness-aware rules were not built for this case (\cref{fig:teaser}). Age-of-information (AoI) aggregation, such as
DFL-AA (Decentralized Federated Learning with Adaptive AoI-weighted Aggregation), discounts the models a client \emph{receives} and keeps the receiver's own weight fixed \citep{hewa2026dflaa}. Gossip
learning weights models by update counts inherited on merging \citep{hegedus2019gossip}. Limited merging discards any model,
possibly the receiver's own, whose count lags by more than a threshold \citep{danner2023limited}. These rules are heuristics.
None of them says whether the receiver should be discounted, which measure of staleness to use, or how steeply to discount.
Without those answers, one cannot predict which rule will work, or why.

\begin{figure}[t]
\centering
\includegraphics[width=\textwidth]{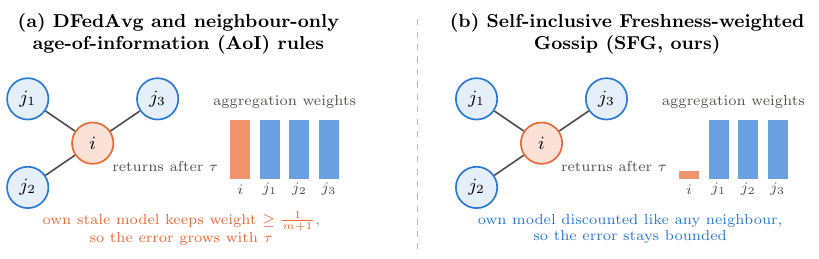}
\caption{\textbf{The stalest model in the room is your own.} A client $i$ returns from an outage of $\tau$ rounds to neighbours
that kept training. (a) Decentralised federated averaging (DFedAvg) and neighbour-only age-of-information (AoI) rules keep the
receiver's own weight at least as large as each neighbour's. Its stale model therefore always enters the aggregate, and the
error grows without bound in $\tau$ (\cref{thm:stale}). (b) Self-inclusive Freshness-weighted Gossip (\sfg), the rule family we
propose, treats the receiver as one more candidate and discounts it by its staleness relative to the freshest candidate. This
keeps every receiver's error bounded (\cref{thm:sfg}).}
\label{fig:teaser}
\end{figure}

\textbf{This paper.} We view aggregation after an outage as an estimation problem. From candidate models of different
staleness, the receiver should recover the collaborative state that a fully available network would have reached. This view
turns the three open questions into three principles, each backed by a formal result.
\begin{enumerate}[leftmargin=*,itemsep=1pt,topsep=2pt,label=(\roman*)]
\item \emph{Whom to discount?} Everyone, the receiver included, by staleness \emph{relative} to the freshest candidate
(\emph{discount yourself}).
\item \emph{By what measure?} By the lag a model has \emph{inherited} through averaging, which local training does not remove.
Age is the wrong measure, because local training resets it (\emph{measure inherited staleness}).
\item \emph{How steeply?} More steeply than a single aggregation suggests. The receiver carries what it absorbs into later
rounds (\emph{discount steeply}).
\end{enumerate}

\textbf{Our method.} From these principles we derive \textbf{Self-inclusive Freshness-weighted Gossip (\sfg)}, a family of
aggregation rules defined formally in \cref{def:sfg}. Every member weights each candidate, the receiver included, by a kernel of
its staleness relative to the freshest candidate. \textbf{\sfg{}} applies a geometric discount to model ages. \textbf{\sfgl{}}
applies it to effective lags. \textbf{\sfgcut{}} keeps only the candidates within a small lag margin of the freshest. \sfgl{} and
\sfgcut{} satisfy all three principles, and they are the rules we propose for practice. The family is a drop-in replacement for
the averaging weights of DFedAvg. It adds one scalar to each message, which is about one millionth of the CNN we train, and
needs no extra communication round and no extra local computation.

\textbf{Contributions and evidence.} Each claim below is tied to the result that supports it.
\begin{itemize}[leftmargin=*,itemsep=2pt,topsep=2pt]
\item \textbf{Self-inclusion separates bounded from unbounded error} (\cref{sec:selfdiscount}). Under a simple model of
staleness, the minimum-error weights are self-inclusive and depend only on relative staleness (\cref{prop:blue}). Every rule
that keeps the receiver's own weight at least as large as each neighbour's has an error that grows without bound in the
receiver's outage (\cref{thm:stale}). This class includes DFedAvg and every neighbour-only AoI rule. \sfg{}, in contrast, keeps
every receiver's error within a constant of the freshest candidate's (\cref{thm:sfg}).
\item \textbf{A new staleness statistic, the effective lag} (\cref{sec:lags}). We identify \emph{age laundering}. A model's age
is reset whenever its client trains, but the lag inherited from stale neighbours is not. The effective lag is computable
locally from one scalar per message and bounds the persistent error. It also provably separates cases that any age-based rule
must confuse (\cref{prop:lag}).
\item \textbf{A horizon analysis that explains steep discounting} (\cref{sec:horizon}). Persistent error survives local
training, so the error that matters is summed over future rounds. The optimal weights then steepen and, under nested staleness,
form a cut (\cref{prop:horizon}). Measured on real traces, the effective rate is 1.5--4 times the one-step rate.
\item \textbf{The \sfg{} family} (\cref{sec:method}). Three drop-in rules with the communication and computation of DFedAvg,
all covered by our error bounds. To our knowledge, they are the first staleness-aware aggregation rules for decentralised learning with a
guarantee that holds for the stale receiver itself, and they compose with a link-reliability correction and with catch-up
training.
\item \textbf{A large, paired evaluation} (\cref{sec:experiments}). We test on CIFAR-10, CIFAR-100 and FEMNIST (Federated
Extended MNIST) under five synthetic outage processes and three real smartphone traces, with 120 paired runs per rule.
\sfgcut{} improves mean personalised accuracy over DFedAvg by $+4.0$ percentage points (pp) in 119 of 120 runs. It closes
56--60\% of the gap to outage-free training, while the best AoI rule gains $+1.8$\,pp. With a link correction, the gain is
$+4.4$\,pp in 120 of 120 runs. With catch-up training, it reaches $+5.1$\,pp. Without any extra computation, \sfgcut{} gains 2.3
times as much as catch-up training does at $1.71\times$ the computation (\cref{fig:headline}a). Removing any single principle
costs 1.5--2.1\,pp, and a random cut of the same size gains nothing (\cref{fig:headline}c). In all ten settings of
\cref{fig:robust}, which include a six-day trace, learning-rate re-tuning, link failures, three alternative backbones and a
pretrained ResNet-18, the best \sfg{} rule gains 2.2--4.6 times as much as the strongest AoI baseline (\cref{fig:robust}).
\end{itemize}

Beyond the rules themselves, the analysis gives a unified account of prior work. Tuned limited merging from gossip learning is
strong because it satisfies all three principles. AoI rules are weaker because they measure staleness by age and discount the
receiver at best implicitly (\cref{tab:positioning}). All appendices referred to in the text follow the references.

\section{Related work}
\label{sec:related}

\textbf{Unavailability in server-based FL.} MIFA reuses stale updates of unavailable devices \citep{gu2021mifa}. FedVARP
reduces the variance of partial participation \citep{jhunjhunwala2022fedvarp}. FedAU estimates participation statistics online \citep{wang2024fedau}, building on a unified analysis of arbitrary client
participation \citep{wang2022unified}, and FedAWE amplifies the updates of returning clients \citep{xiang2024fedawe}.
Correlated participation has been modelled as a Markov chain \citep{sun2025debiasing,rodio2023cafed}. Asynchronous FL discounts
stale updates \citep{xie2019fedasync,nguyen2022fedbuff}, and age-based scheduling decides which clients upload
\citep{yang2020agebased}. All of these methods keep a fresh model at the server, and the server is never stale. The question of
how a stale receiver should weight \emph{itself} therefore never arises. We transplant two server-FL ideas to DFL as baselines:
model memory, the analogue of MIFA, and update amplification, as in FedAWE.

\textbf{Decentralised and personalised learning.} Decentralised stochastic gradient descent (SGD) converges at rates governed by
the mixing of the gossip matrices \citep{lian2017dpsgd,koloskova2020unified,kong2021consensus,vogels2022beyond}, also over
random links \citep{boyd2006randomized,kar2009distributed}. Asynchronous and sporadic variants tolerate bounded staleness or intermittent computation and communication
\citep{lian2018adpsgd,assran2019sgp,bornstein2023swift,liu2024aedfl,zehtabi2025dspodfl}. Decentralised and partial-model
personalisation methods \citep{sun2023dfedavgm,dai2022dispfl,shi2023dfedsam,liu2024dfedpgp,long2025dadpfl,
arivazhagan2019fedper,collins2021fedrep,li2021ditto} do not address outages. \sfg{} is complementary to all of them, because it
changes only the averaging weights. We evaluate it with local momentum, Ditto-style personal models and personal heads.

\textbf{Staleness-aware aggregation.} DFL-AA discounts received models by their age of information relative to the newest
neighbour \citep{hewa2026dflaa}. It corrects for lossy links by inverse-probability weighting (IPW) and keeps the receiver's
weight fixed. Gossip learning merges models in proportion to their update counts, and a merged model inherits the larger count
\citep{hegedus2019gossip,hegedus2021decentralized}. Limited merging discards the less-updated model, possibly the local one, when
counts differ by more than a threshold \citep{danner2023limited}. These are the closest prior rules.

Our analysis connects them to \sfg{}. Through the identity $n_i=t-\ell_i$ (\cref{sec:lags}), \sfgl{} is a smooth, self-inclusive
relative of limited merging. Its counts are inherited by averaging rather than by the maximum. \sfgcut{} applies limited
merging's hard margin to these mean-inherited counts. This change of inheritance is what places \sfgcut{} under our error bounds,
because a deficit in maximum-inherited counts does not bound a lag difference (\cref{prop:counts}). Effective lag is the analogue, for averaged
models, of the version age of information in gossip networks \citep{yates2021gossip,buyukates2022version}.

\textbf{Inverse-variance fusion.} Inverse-variance weights that include the receiver are classical in diffusion adaptation
and consensus \citep{tu2011optimal,zhao2012diffusion,sayed2014adaptation,kamal2013information,julier1997nondivergent}, where
variance comes from observation noise. We compare against the relative-variance combiner of \citet{zhao2012diffusion} and find
that, in personalised learning, its distance-based estimates mix staleness with data heterogeneity.

\textbf{Positioning.} \Cref{tab:positioning} summarises which principles each rule satisfies. Relative to gossip learning, our
contribution is not a better heuristic but an explanation. We show \emph{why} self-inclusion, inherited staleness and steepness
are needed. We derive a staleness statistic for which the resulting rules carry error guarantees. We then verify each principle
by ablation.

\begin{table}[t]
\centering
\caption{\textbf{Which rules satisfy which principle.} Self-inclusive: the receiver's own model is discounted by its
staleness. Inherited: staleness survives local training. Steep: a geometric discount with a small $\gamma$, or a tight margin.
Bound: the aggregation error of a returning receiver is bounded uniformly in its outage under our staleness models.
$^\ast$IPW raises the weight of rarely live links, which acts as an implicit self-discount for rarely available receivers
(\cref{sec:mechanism}). $^\dagger$Steep only with a tight threshold.}
\label{tab:positioning}
\small
\begin{tabular}{lcccc}
\toprule
Rule & Self-inclusive & Inherited & Steep & Bound \\
\midrule
DFedAvg \citep{sun2023dfedavgm} & -- & -- & -- & -- \\
Neighbour-only AoI & -- & -- & \checkmark & -- \\
DFL-AA \citep{hewa2026dflaa} & implicit$^\ast$ & -- & \checkmark & -- \\
Gossip-learning model age \citep{hegedus2019gossip} & \checkmark & max & -- & -- \\
Limited merging \citep{danner2023limited} & \checkmark & max & \checkmark$^\dagger$ & -- \\
\midrule
\sfg{} (ours) & \checkmark & -- & \checkmark & \checkmark\ (\cref{thm:sfg}) \\
\sfgl{} (ours) & \checkmark & mean & \checkmark & \checkmark\ (\cref{prop:lag}) \\
\sfgcut{} (ours) & \checkmark & mean & \checkmark & \checkmark\ (\cref{prop:horizon}) \\
\bottomrule
\end{tabular}
\end{table}

\section{Problem setting}
\label{sec:setup}

\textbf{Clients and outages.} $N$ clients sit on an undirected communication graph $\mathcal G=(V,E)$ and run synchronous
rounds $t=0,1,\dots,T-1$. Client $i$ holds a personal model $x_i^t\in\mathbb R^d$ and private data $\mathcal D_i$. In round $t$,
it is available with indicator $a_i^t\in\{0,1\}$, drawn from an exogenous outage process. We summarise a process by its
unavailability $u$, the long-run fraction of client-rounds spent offline. An unavailable client neither trains nor
communicates. Its model is frozen and continues to serve inference, and we evaluate this frozen model.

\textbf{Aggregation.} An available client, the \emph{receiver}, forms the candidate set of itself and its available neighbours,
\begin{equation}\label{eq:cands}
S_i^t=\{i\}\cup\{j\in\mathcal N(i):a_j^t=1\}.
\end{equation}
It aggregates
\begin{equation}\label{eq:agg}
x_i^{t+\frac12}=\sum_{j\in S_i^t}w_{ij}^t\,x_j^t,\qquad \sum_{j\in S_i^t}w_{ij}^t=1,\quad w_{ij}^t\ge0.
\end{equation}
It then runs $K$ local SGD steps on $\mathcal D_i$ from $x_i^{t+\frac12}$, which personalises by fine-tuning. Decentralised
federated averaging (DFedAvg, \citealp{sun2023dfedavgm}) uses uniform weights, $w_{ij}^t=1/|S_i^t|$. The only design choice we
study is the weights $w_{ij}^t$. Everything else is held fixed, including what is communicated and how much is computed.

\textbf{Age.} The age $\tau_j^t\in\mathbb N$ of model $x_j^t$ is the number of rounds since it was last trained. A client that
trained in the previous round has age $0$. A client returning from an outage of $\tau$ rounds has age $\tau$. Ages are a
deterministic function of the outage trace. Each client can piggyback its age on its model as an integer counter.

\textbf{Aggregation rules.} We compare rules of the form $w_{ij}^t\propto h_{ij}^t$ that use only information available at the
receiver. \emph{Neighbour-only} AoI rules set $h_{ii}=1$ and $h_{ij}=h(\tau_j)\le1$ for neighbours, for example
$h_{ij}=\gamma^{\tau_j}$. The other baselines are described in \cref{sec:related} and Appendix~\ref{app:details}.

\textbf{Metrics.} We report the mean and the 10th percentile of \emph{personalised} test accuracy over all $N$ clients at the
end of training. Each client is evaluated on a test set drawn from its own data distribution. The 10th percentile tracks the
clients that benefit least. We also report communication and local computation.

\section{What aggregation needs under outages}
\label{sec:theory}

We now derive the three principles. Throughout, we fix an available receiver $i$ and write $w_j$ for $w_{ij}^t$. We measure the
quality of an aggregate by its squared distance to the \emph{collaborative state}, the model that a fully available network
would hold. This estimation view gives optimal weights in closed form, against which any rule can be compared.

The analysis proceeds in three steps. \Cref{sec:selfdiscount} characterises the error of a single aggregation and isolates what
an outage changes. \Cref{sec:lags} replaces ages by a dynamic notion of staleness. \Cref{sec:horizon} extends the objective
across rounds. All proofs are in Appendix~\ref{app:proofs}, which also lists the statements we checked numerically on random
instances.

\subsection{A model of staleness}

\begin{assumption}[Staleness]\label{ass:staleness}
There is a reference trajectory $(c_r)_{r\ge0}$, the collaborative state a fully available network would follow. Every
candidate at round $t$ satisfies
\begin{equation*}
x_j^t=c_{t-\tau_j}+\varepsilon_j .
\end{equation*}
The dispersions $\varepsilon_j$ are zero-mean, mutually uncorrelated and uncorrelated with $(c_r)$, with
$\E\|\varepsilon_j\|^2=\sigma^2$. The state drifts with lag covariance
\begin{equation*}
\E\langle c_t-c_{t-a},\,c_t-c_{t-b}\rangle=\delta^2\kappa(a,b),
\end{equation*}
where $\kappa\ge0$, and $\varphi(a):=\kappa(a,a)$ is non-decreasing and unbounded with $\varphi(0)=0$.
\end{assumption}

The assumption separates the two forces that make a candidate imperfect. The dispersion $\sigma^2$ is the spread of fresh
models around the collaborative state, caused by data heterogeneity and SGD noise. The drift $\delta^2$ is how far the state
moves per round. Their ratio
\begin{equation}\label{eq:rho}
\rho=\delta^2/\sigma^2
\end{equation}
measures how costly one round of staleness is, relative to that spread. Two canonical cases are \emph{random-walk} drift (RW),
with mutually uncorrelated increments, $\kappa(a,b)=\min(a,b)$ and $\varphi(a)=a$, and \emph{persistent} drift, with
$\kappa(a,b)=ab$ and $\varphi(a)=a^2$.

Write $e_j=x_j^t-c_t$ for the error of candidate $j$. The candidate errors have covariance
\begin{equation}\label{eq:cov}
\Sigma_{jk}=\E\langle e_j,e_k\rangle=\sigma^2\big(\mathbf 1[j=k]+\rho\,\kappa(\tau_j,\tau_k)\big),
\end{equation}
so a candidate's mean squared error grows with its age:
\begin{equation}\label{eq:vtau}
v(\tau_j)=\sigma^2\big(1+\rho\,\varphi(\tau_j)\big).
\end{equation}
The model is deliberately simple. \Cref{sec:lags} makes staleness dynamic, and \cref{sec:experiments} tests the predictions.

\subsection{Principle 1: discount yourself}
\label{sec:selfdiscount}

\begin{proposition}[Minimum-error aggregation]\label{prop:blue}
The weights that minimise $\E\|\sum_{j\in S}w_jx_j^t-c_t\|^2$ subject to $\mathbf 1^\top w=1$ are
\begin{equation}\label{eq:blue}
w^\star=\frac{\Sigma^{-1}\mathbf 1}{\mathbf 1^\top\Sigma^{-1}\mathbf 1}.
\end{equation}
Under RW drift, moreover, the following hold.
\begin{enumerate}[nosep,leftmargin=*]
\item[(i)] \emph{Relative staleness.} $w^\star$ depends on the ages only through $\Delta_j=\tau_j-\tau_{\min}$, where
$\tau_{\min}=\min_{k\in S}\tau_k$. The minimum error is
\begin{equation*}
\delta^2\tau_{\min}+\frac{1}{\mathbf 1^\top\Sigma'^{-1}\mathbf 1},\qquad
\Sigma'_{jk}=\sigma^2\big(\mathbf 1[j=k]+\rho\min(\Delta_j,\Delta_k)\big).
\end{equation*}
\item[(ii)] \emph{Self-inclusion.} If at most one candidate is older than the youngest, then for every $j\in S$,
\emph{including the receiver $j=i$},
\begin{equation*}
w^\star_j\propto\frac{1}{1+\rho\Delta_j}.
\end{equation*}
\item[(iii)] \emph{No outages.} If all candidates have the same age, $w^\star$ is uniform, which is DFedAvg.
\end{enumerate}
\end{proposition}

Part (i) says that the drift since the freshest candidate was trained is shared by all candidates. No weighting can remove it,
so only relative staleness matters. Part (ii) says that the receiver's own model is just another candidate, discounted by its
staleness like any neighbour. Part (iii) says that optimal aggregation reduces to DFedAvg when nothing is stale, so any gain must
come from outages.
Parts (i) and (ii) both fail for AoI rules, which keep the receiver's weight fixed. This matters in practice. On real traces, a
returning client is strictly the stalest candidate in 85--93\% of returns (Appendix~\ref{app:stalest}). This is precisely the
case that neighbour-only discounting cannot handle.

\begin{theorem}[Stale receivers]\label{thm:stale}
Let receiver $i$ have age $\tau_i$ and $m\ge1$ live neighbours, and let \cref{ass:staleness} hold. Consider any non-negative rule
whose self-weight is at least each neighbour weight, $w_{ii}\ge w_{ij}$. Its aggregation error satisfies
\begin{equation}\label{eq:stalelb}
\E\big\|x_i^{t+\frac12}-c_t\big\|^2\;\ge\;\frac{v(\tau_i)}{(m+1)^2},
\end{equation}
which grows without bound in the outage length. This class includes DFedAvg and every neighbour-only discount
$w_{ij}\propto h(\tau_j)\le h_{ii}=1$. With uncorrelated errors and equally aged, fresher neighbours, such a discount is even
worse than uniform averaging (Appendix~\ref{app:proofs}).
\end{theorem}

The theorem is a negative result about an entire class of rules, not about one rule. No choice of neighbour discount can
repair a rule that refuses to discount the receiver. Inverse-probability weighting, as in DFL-AA, escapes the theorem. Raising
the weight of rarely live links acts as an implicit self-discount for a rarely available receiver (\cref{sec:mechanism}).

\paragraph{A self-inclusive rule.} \Cref{prop:blue} suggests weighting every candidate, the receiver included, by a decreasing
function of its \emph{relative} staleness $s_j$. Self-inclusive Freshness-weighted Gossip (\sfg) uses a geometric discount:
\begin{equation}\label{eq:sfg}
w_{ij}^t=\frac{\gamma^{\,s_j-s_{\min}}}{\sum_{k\in S_i^t}\gamma^{\,s_k-s_{\min}}},\qquad j\in S_i^t,
\end{equation}
with $s_{\min}=\min_{k\in S_i^t}s_k$ and $\gamma\in[0,1)$, using the convention $0^0=1$. The staleness $s_j$ is a scalar sent
with the model. It is the age $\tau_j$ for \sfg{} and the effective lag $\ell_j$ of \cref{sec:lags} for \sfgl{}. The geometric
form is robust. It keeps the error bounded for every sub-exponential drift law, whereas hyperbolic weights can fail under fast
drift (\cref{prop:robust} in Appendix~\ref{app:proofs}).

\begin{theorem}[SFG bounds the error of every receiver]\label{thm:sfg}
Let \cref{ass:staleness} hold with RW drift, and let $s=\tau$. For every candidate set and every configuration of ages,
\begin{align}
\E\big\|x_i^{t+\frac12}-c_t\big\|^2&\le\sigma^2\textstyle\sum_jw_j^2+\delta^2\big(\tau_{\min}+\sum_jw_j\Delta_j\big)\nonumber\\
&\le\sigma^2\textstyle\sum_jw_j^2+\delta^2\Big(\tau_{\min}+\dfrac{m}{e\ln(1/\gamma)}\Big).\label{eq:sfgbound}
\end{align}
The receiver's own age therefore enters only through $\tau_{\min}\le\tau_i$. \sfg{} coincides with neighbour-only AoI
discounting whenever the receiver is fresh ($\tau_i=0$). It coincides with DFedAvg whenever all candidates have the same age,
in particular without outages.
\end{theorem}

\begin{corollary}[Exact errors for a returning client]\label{cor:exact}
Under RW drift, consider a receiver of age $\tau$ with $m$ fresh neighbours. In units of $\sigma^2$, its aggregation error is
\begin{equation}\label{eq:exact}
\underbrace{\frac{m+1+\rho\tau}{(m+1)^2}}_{\text{DFedAvg and AoI}},\qquad
\underbrace{\frac{m+\gamma^{2\tau}(1+\rho\tau)}{(m+\gamma^\tau)^2}}_{\text{\sfg}},\qquad
\underbrace{\Big(m+\frac1{1+\rho\tau}\Big)^{-1}}_{\text{optimum}}.
\end{equation}
The first grows linearly in $\tau$. The second tends to the optimum.
\end{corollary}

\Cref{thm:stale,thm:sfg} together make the case for self-inclusion. It is what separates unbounded from bounded error. It also
costs nothing in the benign cases. When the receiver is fresh, \sfg{} reduces to AoI. When no one is stale, it reduces to
DFedAvg. For long outages, \sfg{} approaches the optimum (\cref{fig:mech}a).

The first inequality of \eqref{eq:sfgbound} holds for any non-negative weights. Every rule confined to candidates within an
age margin $L$ of the least stale, such as a hard cut, therefore obeys the bound with $L$ in place of $m/(e\ln(1/\gamma))$.
\Cref{prop:lag}(iv) is the analogue on lags. Cuts on maximum-inherited counts are not covered, because a count deficit does not
bound a lag difference (\cref{prop:counts}).

\begin{principle}[Discount yourself]\label{pr:self}
Weight every candidate, the receiver included, by a decreasing function of its staleness relative to the freshest candidate.
\end{principle}

\subsection{Principle 2: measure inherited staleness, not age}
\label{sec:lags}

Which staleness should the weights use? A model's age is reset whenever its client trains. The lag the model \emph{inherited}
by averaging with stale models is not reset. Local training moves the model forward at the pace of the collaborative state,
but it cannot recover progress that never reached the model.

Consider a fresh client that averages with a neighbour returning from a long outage. After one round of training, its age is
zero. Yet part of the neighbour's missing progress now lives in its model and will be passed on. We call this \emph{age
laundering}. To quantify it, we let the state evolve with the network.

\begin{assumption}[Dynamic staleness]\label{ass:dynamic}
The collaborative state is the counterfactual full-availability trajectory,
\begin{equation*}
c_{t+1}=c_t+\xi_t .
\end{equation*}
The increments are mutually uncorrelated and uncorrelated with the dispersions, with $\E\|\xi_t\|^2=\delta^2$. An available
client's model after aggregation and local training is
\begin{equation*}
x_i^{t+1}=c_{t+1}+p_i^{t+1}+\varepsilon_i^{t+1},\qquad p_i^{t+1}=\sum_jw_{ij}^tp_j^t .
\end{equation*}
The persistent error $p_i$ is inherited: training keeps pace with the state but does not remove inherited error. The dispersion
$\varepsilon_i^{t+1}$ is regenerated by local training. It has variance $\sigma^2$ and is independent of the past and of the other
clients' dispersions. An unavailable client is frozen, so $p_i^{t+1}=p_i^t-\xi_t$.
\end{assumption}

Under this assumption the persistent error is a weighted sum of missed increments,
\begin{equation}\label{eq:missed}
p_i^t=-\sum_{r<t}A_{ir}^t\,\xi_r,\qquad A^t_{ir}\in[0,1],
\end{equation}
where $A^t_{ir}$ is the fraction of increment $r$ that client $i$ has missed. The matrix $A$ evolves linearly. It is multiplied
by the aggregation matrix $\widetilde W^t$, which holds the weights of available clients and identity rows otherwise, and then
gains a new column for the round:
\begin{equation}\label{eq:Arec}
A^{t+1}=\widetilde W^tA^t,\qquad A_{\cdot t}^{t+1}=\mathbf 1[a^t=0].
\end{equation}

\begin{definition}[Effective lag]\label{def:lag}
The \emph{effective lag} of client $i$ is $\ell_i^t=\sum_rA_{ir}^t$, the number of rounds of collaborative progress its model is
missing.
\end{definition}

\begin{proposition}[Age laundering]\label{prop:lag}
Under \cref{ass:dynamic}, the following hold.
\begin{enumerate}[nosep,leftmargin=*]
\item[(i)] \emph{Lags are computable.} The lags obey
\begin{equation}\label{eq:lagrec}
\ell^{t+1}=\widetilde W^t\ell^t+\mathbf 1[a^t=0],
\end{equation}
so every client tracks its lag from its candidates' lags with one scalar per message. Moreover,
$\E\|p_i^t\|^2=\delta^2\sum_r(A_{ir}^t)^2\le\delta^2\ell_i^t$, with equality when the client's missed increments are not mixed.
\item[(ii)] \emph{Ages under-report lag, and age-blind rules pay for it.} Always $\tau_i^t\le\ell_i^t$. A client that absorbed
lag by averaging is reported as fresh after it trains. Moreover, take any $D>0$ and any graph with an edge $a$--$b$ such that $a$
has a neighbour not adjacent to $b$. Then $a$ and $b$ can have the same ages and availability histories but lags $vD$ and $D$,
where $v<1$ is the weight the rule gives stale candidates at their return. Consider any rule whose weights depend only on the
candidates' roles, ages and availability histories. In one of the two aggregations of $a$ and $b$ with each other, it incurs
persistent error at least $\delta^2(1-v)D/4$ more than keeping only the less-lagged candidate.
\item[(iii)] \emph{Averaging does not remove lag.} If the aggregation weights are doubly stochastic, then
\begin{equation*}
\textstyle\sum_i\ell_i^{t+1}=\sum_i\ell_i^t+|\{i:a_i^t=0\}|.
\end{equation*}
Doubly stochastic averaging redistributes the lag injected by outages. It does not remove it.
\item[(iv)] \emph{SFG on lags.} With $s=\ell$ in \eqref{eq:sfg}, every available receiver's aggregation error is at most
\begin{equation*}
\sigma^2\textstyle\sum_jw_j^2+\delta^2\Big(\ell_{\min}+\dfrac{m}{e\ln(1/\gamma)}\Big),
\end{equation*}
uniformly in its own lag.
\item[(v)] \emph{Only the persistent error carries over.} The aggregation error splits into two terms. The dispersion term
$\sigma^2\sum_jw_j^2$ is regenerated by local training in the next round. The persistent error is inherited by every later
round and satisfies
\begin{equation}\label{eq:pers}
\E\Big\|\sum_jw_jp_j\Big\|^2=\delta^2\,w^\top Gw,\qquad G=AA^\top .
\end{equation}
Always $0\le G_{jk}\le \Lambda_{jk}:=\min(\ell_j,\ell_k)$. Equality holds when the missed increments are nested, that is, when
the rows $A_j$ lie in $\{0,1\}$ with supports ordered by inclusion.
\end{enumerate}
\end{proposition}

Part (ii) is the key separation. Some configurations cannot be distinguished by any age-based rule, yet call for different
weights. The cost of confusing them grows linearly with the outage. Part (iii) explains why plain averaging cannot ``wash out''
outages. It only spreads the injected lag across the network.

Age laundering has two consequences. First, models reported as fresh carry hidden lag, so a rate fitted to ages is too small
(\cref{sec:rate}). Second, rules should run on inherited staleness. The lag also has a simple count interpretation. Define
\begin{equation}\label{eq:counts}
n_i^t:=t-\ell_i^t .
\end{equation}
Then $n_i\leftarrow\sum_jw_{ij}n_j$ at aggregation, and $n_i\leftarrow n_i+1$ per training round. So $n_i$ counts the
training rounds a model has absorbed, and \sfgl{} weights candidates by $\gamma^{\,n_{\max}-n_j}$.

This links \sfgl{} to gossip learning. Gossip learning's model age is also an inherited count, but it is merged by the maximum
\citep{hegedus2019gossip}. Limited merging drops candidates whose count lags by more than a threshold \citep{danner2023limited}.
Inherited counts are why those rules outperform AoI rules. The maximum, however, discards exactly the information that matters:
a model that absorbed a stale neighbour inherits the larger count and looks fresh.

\begin{proposition}[Maximum-inherited counts do not certify lag]\label{prop:counts}
Let $n_j$ be update counts inherited by the maximum on merging. For every $D>0$ and every $v\in(0,1]$, there are two
availability histories with the following properties. At some round, a receiver $i$ has the same candidates, with the same
max-inherited counts and the same ages, in both histories. One candidate $j$ has the largest count in both. Its effective lag is
$0$ in the first history and $vD$ in the second, where $v$ is the weight $j$ gave its own model when it returned from an outage of
$D$ rounds. Consequently, any rule whose weights depend on candidates only through their roles, ages and max-inherited counts
gives $j$ the same weight $w_j$ in both histories. If $w_j>0$, its persistent error in the second history grows without bound in
$D$:
\begin{equation}\label{eq:countsbound}
\E\Big\|\sum_kw_kp_k\Big\|^2\;\ge\;\delta^2\,w_j^2\,v^2D .
\end{equation}
This includes limited merging with any margin $L$ and \sfg{} weights on counts. \sfgcut{} drops $j$ in the second history as
soon as $vD>\ell_{\min}+L$, and its persistent error obeys \cref{prop:horizon}(iii) in both.
\end{proposition}

\Cref{prop:counts} makes the difference between counts and lags precise. In a network where \emph{every} client runs limited
merging with margin $L$, a returning client keeps its own model only if $D\le L$, so the lag hidden behind a count is at most
$vL$. This explains why tuned limited merging performs well in our experiments (\cref{sec:main}). Its protection, however, is a
property of the whole network's history, which the counts cannot reveal. It can fail as soon as some clients merge differently,
for instance during a partial deployment, or when some clients run an older or a different aggregation rule. The lag-based bounds of \cref{prop:lag}(iv) and
\cref{prop:horizon}(iii) hold for every weight history, because the effective lag certifies itself. This is the substance of our
claim that \sfgcut{} carries a guarantee that limited merging does not.

\begin{principle}[Measure inherited staleness]\label{pr:lag}
Measure a candidate's staleness by the lag it has inherited through averaging, not by the time since it last trained.
\end{principle}

\subsection{Principle 3: discount steeply}
\label{sec:horizon}

\Cref{prop:lag}(v) shows that the single-aggregation objective of \cref{prop:blue} is myopic. Local training regenerates the
dispersion of an aggregate. Its persistent error, however, stays in every later model of the receiver.

To account for this, fix all other aggregations. Let $\Phi^s$ be the product of the aggregation matrices of rounds
$t+1,\dots,s-1$, with identity rows for clients that do not aggregate. Let $\psi_s=\Phi^s_{ii}$ be the weight that the receiver's
model at round $s$ puts on its own aggregate. We score the dispersion once, as in \cref{prop:blue}, and the persistent error at
every remaining round:
\begin{equation}\label{eq:J}
J(w)=\sigma^2\|w\|^2+\sum_{s=t+1}^{T}\E\|p_i^s\|^2 .
\end{equation}

\begin{proposition}[Horizon]\label{prop:horizon}
Under \cref{ass:dynamic},
\begin{equation}\label{eq:Jform}
J(w)=\sigma^2\Big[w^\top\big(I+H\rho\,G\big)w+2\rho\,g^\top w\Big]+\text{const}.
\end{equation}
Here $H=\sum_{s=t+1}^T\psi_s^2\ge1$ is the receiver's \emph{self-return}. The import term is $g_j=\langle A_j,B\rangle\ge0$, where
$B=\sum_s\psi_s\sum_{k\ne i}\Phi^s_{ik}A_k^{t+1}$ collects the missed fractions the receiver will import from others.
\begin{enumerate}[nosep,leftmargin=*]
\item[(i)] Without imported lag ($B=0$), $J$ is minimised over $\mathbf 1^\top w=1$ by the generalised least-squares (GLS)
weights for $I+H\rho\,G$. Under nesting, these are the weights of \cref{prop:blue} with $(\rho,\tau)$ replaced by $(H\rho,\ell)$.
They tend to uniform weights on the least-lagged candidates as $H\to\infty$.
\item[(ii)] Under nesting, the minimiser over the simplex is a cut. Its weights are non-increasing in lag. It keeps every
candidate below some lag, including the least-lagged candidate $j_0$. Every kept candidate has $\rho(g_j-g_{j_0})<w_{j_0}\le1$.
If imports are uniform on the candidates' missed increments ($B_r=\beta$), then $g_j-g_{j_0}=\beta\Delta_j$, and the margin is
below $1/(\rho\beta)$.
\item[(iii)] For $w\ge0$, with $\Lambda_{jk}=\min(\ell_j,\ell_k)$,
\begin{equation*}
w^\top Gw\le w^\top\Lambda w\le\textstyle\sum_jw_j\ell_j .
\end{equation*}
A rule confined to a lag margin $L$ therefore has $w^\top Gw\le\ell_{\min}+L$.
\end{enumerate}
\end{proposition}

The intuition is that persistent error \emph{returns}. The receiver keeps its aggregate through its next outages and gets part
of it back from its neighbours. The horizon therefore counts persistent error $H$ times, while dispersion is counted once. Three
consequences follow.
First, the effective rate exceeds the one-step rate $\rho$, which steepens the optimal discount (part i). Second, when lag
absorbed now would add to lag imported later, the optimum truncates into a cut (part ii). Third, every rule confined to a margin
inherits an error bound (part iii).

\textbf{How steep, in practice?} Dispersion is not fully regenerated by local training either. If it is counted $H_\varepsilon$
times, the effective rate becomes, to first order,
\begin{equation}\label{eq:effrate}
\rho_{\mathrm{eff}}\approx\rho\,\lambda_G\,\frac{H}{H_\varepsilon},
\end{equation}
where $\lambda_G\approx1$ converts persistent differences into lag differences. Measured along \sfgl{} trajectories, this rate is
0.8--9, against a one-step estimate of 0.4--1.9 (Appendix~\ref{app:ext-theory}). The theory therefore establishes \emph{that} the
discount should be steep. A cut additionally needs the nesting and import conditions of part (ii), which fail on the long-outage
traces, so we select its margin empirically.

\begin{principle}[Discount steeply]\label{pr:steep}
Discount more steeply than a single aggregation suggests, down to a hard cut. The receiver carries the lag it absorbs into later
rounds.
\end{principle}

\section{The SFG family of aggregation rules}
\label{sec:method}

We now turn the principles into the method we propose. \Cref{def:sfg} states it formally, \cref{alg:sfg} gives one round,
and \cref{tab:family} lists its three members.

\begin{definition}[Self-inclusive Freshness-weighted Gossip, \sfg]\label{def:sfg}
Let every model carry a scalar staleness $s_j\ge0$, let $h:[0,\infty)\to[0,1]$ be a non-increasing kernel with $h(0)=1$, and let
$\kappa_j>0$ be optional count weights. An available receiver $i$ aggregates over its candidate set $S_i^t$, \emph{itself
included}, with
\begin{equation}\label{eq:sfgfamily}
w_{ij}^t=\frac{\kappa_j\,h\big(s_j-s_{\min}\big)}{\sum_{k\in S_i^t}\kappa_k\,h\big(s_k-s_{\min}\big)},
\qquad s_{\min}=\min_{k\in S_i^t}s_k,\qquad j\in S_i^t .
\end{equation}
The three members are
\begin{equation*}
\underbrace{s=\tau,\ h(\Delta)=\gamma^\Delta}_{\text{\sfg{}}},\qquad
\underbrace{s=\ell,\ h(\Delta)=\gamma^\Delta}_{\text{\sfgl{}}},\qquad
\underbrace{s=\ell,\ h(\Delta)=\mathbf 1[\Delta\le L],\ \kappa_j=1+t-\ell_j}_{\text{\sfgcut{}}},
\end{equation*}
with $\kappa\equiv1$ unless stated. Lags are inherited by \eqref{eq:lagupdate}.
\end{definition}

Two features separate \sfg{} from prior staleness-aware rules. The receiver's own weight is computed by the same kernel as every
neighbour's, so the receiver is discounted whenever it is stale (\cref{pr:self}). The kernel acts on staleness \emph{relative} to
the freshest candidate, so a uniformly stale neighbourhood is averaged as DFedAvg would average it, and only differences in
staleness change the weights. The members differ in the staleness they use (\cref{pr:lag}) and in how steeply they discount
(\cref{pr:steep}).

\begin{algorithm}[t]
\caption{Self-inclusive Freshness-weighted Gossip (one round, client $i$). \sfg{} uses ages ($s=\tau$). \sfgl{} and \sfgcut{}
use effective lags ($s=\ell$). Within the margin, the count weights $1+t-s_j$ of \sfgcut{} are uniform up to $O(L/t)$.}
\label{alg:sfg}
\begin{algorithmic}[1]
\REQUIRE model $x_i$, staleness $s_i$, discount $\gamma\in[0,1)$ ($0^0=1$) or margin $L$, local steps $K$
\IF{$i$ is available}
  \STATE send $(x_i,s_i)$ to, and receive $(x_j,s_j)$ from, live neighbours $j$
  \STATE $S\gets\{i\}\cup\{\text{live neighbours}\}$, \quad $s_{\min}\gets\min_{j\in S}s_j$
  \STATE $h_j\gets\gamma^{s_j-s_{\min}}$ (\sfg, \sfgl) or $h_j\gets(1{+}t{-}s_j)\,\mathbf 1[s_j-s_{\min}\le L]$ (\sfgcut), $j\in S$
  \STATE $w_j\gets h_j/\sum_{k\in S}h_k$, \quad $x_i\gets\sum_{j\in S}w_jx_j$
  \STATE \textbf{if} lags: $s_i\gets\sum_{j\in S}w_js_j$ \hfill\COMMENT{lag is inherited}
  \STATE run $K$ local SGD steps on $\mathcal D_i$ from $x_i$, \quad \textbf{if} \sfg: $s_i\gets0$
\ELSE
  \STATE $s_i\gets s_i+1$ \hfill\COMMENT{model frozen}
\ENDIF
\end{algorithmic}
\end{algorithm}

\begin{table}[t]
\centering
\caption{\textbf{The \sfg{} family.} All members discount the receiver like any neighbour and use staleness relative to the
freshest candidate. Principles refer to \cref{pr:self,pr:lag,pr:steep}.}
\label{tab:family}
\small
\begin{tabular}{lllll}
\toprule
Rule & Staleness $s_j$ & Unnormalised weight $h_j$ & Principles & Guarantee \\
\midrule
\sfg{} & age $\tau_j$ & $\gamma^{\,s_j-s_{\min}}$ & 1, 3 & \cref{thm:sfg} \\
\sfgl{} & effective lag $\ell_j$ & $\gamma^{\,s_j-s_{\min}}$ & 1, 2, 3 & \cref{prop:lag}(iv) \\
\sfgcut{} & effective lag $\ell_j$ & $(1+t-s_j)\,\mathbf 1[s_j-s_{\min}\le L]$ & 1, 2, 3 & \cref{prop:horizon}(iii) \\
\bottomrule
\end{tabular}
\end{table}

\textbf{\sfg{}} weights candidates geometrically by relative age. It satisfies \cref{pr:self} and is the simplest member. It
needs only an integer counter.

\textbf{\sfgl{}} replaces ages by effective lags. Each client keeps a real-valued lag $\ell_i$, initially 0, and sends it with
its model. After aggregating, it inherits the weighted mean of its candidates' lags. Every round offline adds 1:
\begin{equation}\label{eq:lagupdate}
\ell_i\leftarrow\sum_{j\in S_i^t}w_{ij}^t\,\ell_j\quad\text{(available)},\qquad
\ell_i\leftarrow\ell_i+1\quad\text{(offline)}.
\end{equation}

\textbf{\sfgcut{}} keeps only the candidates whose lag is within a margin $L$ of the smallest, the receiver included. It weights
them by their mean-inherited counts:
\begin{equation}\label{eq:cut}
w_{ij}^t\propto(1+t-\ell_j)\,\mathbf 1\big[\ell_j-\ell_{\min}\le L\big],\qquad j\in S_i^t .
\end{equation}
Within the margin, these weights are nearly uniform. \sfgcut{} is limited merging's hard threshold \citep{danner2023limited}
applied to effective lags instead of maximum-inherited counts. This substitution is what brings the cut under the bounds of
\cref{prop:lag,prop:horizon}.

\textbf{Efficiency.} \sfg{} is as cheap as DFedAvg in every resource that matters. Each message carries one extra scalar next to
a model of $d$ parameters, a relative overhead of $1/d$, which is about $10^{-6}$ for our CNN ($d=0.88$M) and below $10^{-7}$ for
ResNet-18 ($d=11.2$M). Computing the weights takes $O(|S_i^t|)$ operations, against the $O(|S_i^t|\,d)$ operations of averaging
the models themselves. \sfg{} needs no extra communication round, no global information, no stored copies of past models and no
extra local computation. It coincides with DFedAvg whenever no candidate is stale.

Because it changes only the weights in \eqref{eq:agg}, it plugs into any decentralised method that averages models. This
includes local momentum (DFedAvgM), Ditto-style personal models and personal heads.

\textbf{Extensions.} Two orthogonal additions compose with every member. The first, \emph{+\,IPW}, divides each neighbour's weight by an online estimate of the probability that its link is live, as
in DFL-AA \citep{hewa2026dflaa}. The receiver's own model has estimate 1. The second, \emph{+\,catch-up}, lets a returning client
run $K(1+\min(\tau_i,3))$ local steps. It trades extra computation for accuracy.

\textbf{Which rule to use.} Our default is \sfgcut{} with $L\in\{1,2\}$. It has the largest average gain of the deployable rules
we tested, and $L=2$ has a better worst case (\cref{sec:main}). \sfgl{} with $\gamma\in[0.2,0.5]$ is the smooth alternative and
is preferable with local momentum. Under heterogeneous availability, we recommend adding IPW. With it, \sfgl{} lifts every
availability quartile (\cref{sec:mechanism}). The discount $\gamma=0.5$ and the margin were fixed on a pilot seed. Neither was
tuned per dataset or per outage process.

\section{Experiments}
\label{sec:experiments}

Our experiments answer five questions. \textbf{Q1 (effectiveness):} do the \sfg{} rules improve personalised accuracy under outages, and by how much relative to AoI and
gossip-learning rules? \textbf{Q2 (necessity):} is each principle needed? \textbf{Q3 (mechanism):} why does self-inclusive
discounting work, and which clients benefit? \textbf{Q4 (theory meets practice):} does the theory predict the right staleness
measure and discount rate? \textbf{Q5 (robustness):} do the gains survive changes of protocol, model and network?

\subsection{Experimental setup}
\label{sec:expsetup}

\textbf{Training and data.} $N=100$ clients on a 6-nearest-neighbour geometric graph run 300 rounds of $K=10$ local SGD steps
(batch 32, learning rate 0.05). The main model is a convolutional neural network (CNN) with 0.88M parameters, and a pretrained
ResNet-18 with GroupNorm (ResNet-18-GN) serves as a larger-model check. We use CIFAR-10 and CIFAR-100 with Dirichlet label skew
($\alpha=0.3$), and FEMNIST with its natural writer split from the LEAF benchmark \citep{caldas2018leaf}. Each client is evaluated
on a test set drawn from its own distribution.

\textbf{Outages.} We use eight outage processes. Five are synthetic, with unavailability calibrated to $\pm0.005$:
independent and identically distributed (i.i.d.) outages at $u=0.6$, bursty two-state Gilbert--Elliott chains with mean outage
length 10 at $u\in\{0.6,0.8\}$, and heterogeneous per-client rates at $u\in\{0.6,0.77\}$.
Three come from real smartphone availability, in 10-minute rounds. FLASH \citep{yang2021characterizing} gives a relaxed regime
($u\approx0.62$) and a strict one ($u\approx0.82$, mean outage 7.6\,h). STUNner \citep{szabo2019stunner} gives a ``plugged in
and on WiFi'' regime ($u\approx0.82$, mean outage 8.6\,h). The real traces contain the long, correlated, diurnal outages that
synthetic processes miss (\cref{fig:traces}).

\textbf{Baselines and protocol.} We compare with DFedAvg, local training, server-FL remedies transplanted to DFL (model memory, the
analogue of MIFA, and FedAWE-style update amplification), the relative-variance combiner of diffusion adaptation, the AoI rules
(neighbour-only AoI, AoI\,+\,IPW and DFL-AA) and the gossip-learning rules (model age and limited merging). Random controls and
oracles with global information test specific hypotheses (Appendix~\ref{app:details}). All rules share data partitions, graph,
initialisation, outage trace and minibatch stream per seed, so every comparison is paired over seeds 1--5. Each rule is evaluated
on 24 dataset--outage cells and 120 paired runs, with $\gamma=0.5$ fixed on a pilot seed. Gains are in percentage points (pp) of
mean personalised accuracy over DFedAvg. Intervals are 95\% stratified $t$-intervals over the 15 dataset--seed units, and tests
are exact sign-flip permutation tests, Holm-corrected over 70 comparisons. When the two disagree, we call the difference
unresolved.

\textbf{What do outages cost?}\label{sec:anatomy} Undiscounted stale parameters are a first-order cost. A client that cannot train but keeps sharing
its frozen model costs DFedAvg \emph{more} than a full outage (10.3--12.6\,pp against 7.1--9.6\,pp at $u=0.8$), and \sfg{} removes
most of that cost (1.2--2.2\,pp). Outage structure matters little at a matched rate, with one exception that limits every
staleness-based rule. When neighbours fail and return together, a returning client finds no fresher candidate, and \sfg{}'s gain
falls from $+3.8$ to $+0.7$\,pp (Appendix~\ref{app:ext-anatomy}).

\subsection{Q1: the SFG rules deliver the largest gains}
\label{sec:main}

\Cref{fig:headline} summarises the answer and \cref{tab:main} gives the numbers. Every rule, the controls and the per-cell results
are in \cref{tab:mainfull,tab:cells,fig:main}.

\begin{figure*}[t]
\centering
\includegraphics[width=\textwidth]{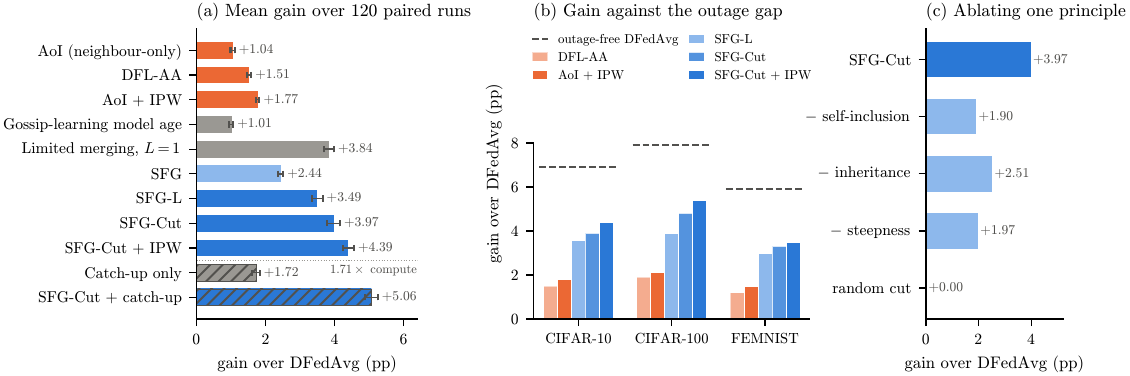}
\caption{\textbf{Effectiveness, efficiency and necessity.} (a) Gain in mean personalised accuracy over DFedAvg, with 95\%
stratified $t$-intervals over 120 paired runs (\cref{tab:main}). Orange: age-of-information rules. Grey: gossip learning and
catch-up training. Blue: the \sfg{} family. Hatched rules use $1.71\times$ the local computation of DFedAvg, and all others use
exactly as much. (b) Gain per dataset, averaged over the eight outage processes. The dashed line marks outage-free DFedAvg, so the
height of each bar relative to it is the share of the outage gap that the rule recovers. (c) Removing one principle from
\sfgcut{} ($L=1$): never dropping the receiver, cutting on ages instead of lags, or a loose margin ($L=10$). A random cut that
keeps the same number of candidates gains nothing (\cref{tab:inherit,tab:selfincl}).}
\label{fig:headline}
\end{figure*}

\textbf{Headline.} \sfgcut{} ($L=1$) improves mean personalised accuracy over DFedAvg by $+3.97$\,pp ($[3.79,4.15]$). It wins
119 of 120 paired runs and all 24 cells, and closes 56--60\% of the gap between DFedAvg with and without outages. The weakest
clients gain most, $+4.75$\,pp at the 10th percentile. \sfgl{} gains $+3.49$\,pp and even \sfg{}, on plain ages, gains
$+2.44$\,pp in all 120 runs. With the link correction of IPW, \sfgcut{} gains $+4.39$\,pp in 120 of 120 runs and closes 59--69\% of
the gap. The ranking is the same on every dataset (\cref{fig:headline}b).

\textbf{Accuracy for free.} These gains come at the communication and computation of DFedAvg, plus one scalar per message.
Spending computation is a much weaker remedy. Catch-up training, which gives returning clients up to four times the local steps,
gains $+1.72$\,pp at $1.71\times$ the computation. \sfgcut{} gains 2.3 times as much at no extra cost, and the two remedies combine
to $+5.06$\,pp.

\textbf{Baselines.} The AoI rules recover a quarter of the gap or less: neighbour-only AoI gains $+1.04$\,pp, DFL-AA $+1.51$\,pp and
AoI\,+\,IPW $+1.77$\,pp. The lag-based \sfg{} rules roughly double the best of them. Reusing stale information without discounting
it is harmful. Model memory costs $-15.6$\,pp and update amplification $-1.6$\,pp (Appendix~\ref{app:ext-main}).

\textbf{Gossip learning, explained.} Tuned limited merging on gossip learning's maximum-inherited counts gains $+3.84$\,pp at
$L=1$. Our principles explain why. It is self-inclusive, it uses inherited staleness, and with a tight threshold it is steep.
Without steepness, the gain shrinks to $+2.55$\,pp ($L=10$), and gossip learning's plain count weights gain only $+1.01$\,pp.
\sfgcut{} is 0.13\,pp above limited merging, an unresolved difference. What separates them is the guarantee. The effective lag
certifies itself, whereas the protection of limited merging rests on every client running it (\cref{prop:counts}).

\begin{table*}[t]
\centering
\caption{\textbf{Main results} on 3 datasets, 8 outage processes (five synthetic, three real) and 5 seeds. Accuracy is averaged
over the eight processes. Gain: paired difference to DFedAvg in percentage points, mean over 120 runs, in mean and
10th-percentile (p10) client accuracy, with a 95\% stratified $t$-interval over dataset--seed units. Wins: paired runs with a
positive gain. Cells: dataset--process cells with a positive mean gain. Local steps: computation relative to DFedAvg.
$^\dagger$Gossip-learning update counts, inherited by the maximum on merging. All 25 rules, including the controls and the
oracle, are in \cref{tab:mainfull}. Bold: best value in each column, separately for rules without and with catch-up training (extra computation). The oracle is excluded.}
\label{tab:main}
\small
\setlength{\tabcolsep}{3.9pt}
\adjustbox{max width=\linewidth}{%
\begin{tabular}{l ccc cc ccc}
\toprule
 & \multicolumn{3}{c}{Mean accuracy (\%)} & \multicolumn{2}{c}{Gain vs DFedAvg (pp)} & Wins & Cells & Local \\
\cmidrule(lr){2-4}\cmidrule(lr){5-6}
Method & CIFAR-10 & CIFAR-100 & FEMNIST & mean [95\% CI] & p10 & /120 & $>0$ & steps \\
\midrule
DFedAvg & 66.8 & 25.2 & 80.8 & -- & -- & -- & -- & 1.00 \\
\midrule
\multicolumn{9}{l}{\emph{Baselines}} \\
AoI (neighbour-only, ages) & 68.0 & 26.2 & 81.8 & +1.04 [0.98, 1.10] & +1.19 & 113 & 22 & 1.00 \\
DFL-AA (relative AoI + IPW) & 68.3 & 27.1 & 82.0 & +1.51 [1.45, 1.58] & +1.71 & 115 & 24 & 1.00 \\
AoI + IPW (absolute AoI) & 68.6 & 27.3 & 82.3 & +1.77 [1.72, 1.82] & +2.06 & 120 & 24 & 1.00 \\
Limited merging$^\dagger$, $L{=}1$ & 70.6 & 29.7 & 84.1 & +3.84 [3.70, 3.98] & +4.53 & 119 & 24 & 1.00 \\
\midrule
\multicolumn{9}{l}{\emph{Proposed: the \sfg{} family}} \\
\sfg{} (ages) & 69.4 & 28.0 & 82.9 & +2.44 [2.37, 2.51] & +2.77 & 120 & 24 & 1.00 \\
\sfgl{} (lags) & 70.4 & 29.1 & 83.8 & +3.49 [3.33, 3.65] & +4.09 & 118 & 24 & 1.00 \\
\sfgl{} + IPW (lags) & 71.0 & 30.2 & \textbf{84.3} & +4.21 [4.08, 4.33] & +4.91 & 120 & 24 & 1.00 \\
\sfgcut{}, $L{=}1$ & 70.7 & 30.0 & 84.1 & +3.97 [3.79, 4.15] & +4.75 & 119 & 24 & 1.00 \\
\sfgcut{}, $L{=}1$, + IPW & \textbf{71.2} & \textbf{30.6} & \textbf{84.3} & \textbf{+4.39} [4.23, 4.56] & \textbf{+5.17} & 120 & 24 & 1.00 \\
\midrule
\multicolumn{9}{l}{\emph{With catch-up training (extra computation)}} \\
Catch-up local steps & 68.8 & 26.9 & 82.4 & +1.72 [1.61, 1.83] & +2.11 & 117 & 24 & 1.71 \\
\sfgcut{} + catch-up & \textbf{72.0} & \textbf{31.0} & \textbf{85.1} & \textbf{+5.06} [4.88, 5.25] & \textbf{+6.18} & 120 & 24 & 1.71 \\
\bottomrule
\end{tabular}
}
\end{table*}

\subsection{Q2: each principle is necessary}
\label{sec:ablation}

\textbf{Removing any principle costs 1.5--2.1\,pp.} We start from \sfgcut{} and remove one principle at a time
(\cref{fig:headline}c, all $p_{\text{Holm}}<0.01$). Never dropping the receiver costs $-2.07$\,pp, cutting on ages instead of lags
$-1.46$\,pp, and a loose margin ($L=10$) $-2.00$\,pp. A placebo cut that keeps the same number of candidates at random gains
$+0.00$\,pp, so the benefit lies in \emph{which} candidates are kept, not in sparsity. Geometric weights confirm the first
principle. Applied to neighbours only, they gain 1.5--2.5\,pp less on ages, lags and counts alike (\cref{tab:selfincl}).

\textbf{Both halves of self-inclusion, and lags where staleness is widespread.} Discounting only stale neighbours or only the
receiver's own stale model each gives about 40\% of \sfg{}'s gain. Doing both gives more than the sum (\cref{fig:mech}c). Lags add
$+1.05$\,pp over ages ($p_{\text{Holm}}=0.004$), concentrated where many neighbours are stale at once. Under bursty outages at
$u=0.8$, \sfgl{} gains $+6.8$\,pp against \sfg{}'s $+3.8$\,pp (Appendix~\ref{app:ext-ablation}).

\subsection{Q3: why self-inclusive discounting works, and who benefits}
\label{sec:mechanism}
\begin{figure*}[t]
\centering
\includegraphics[width=\textwidth]{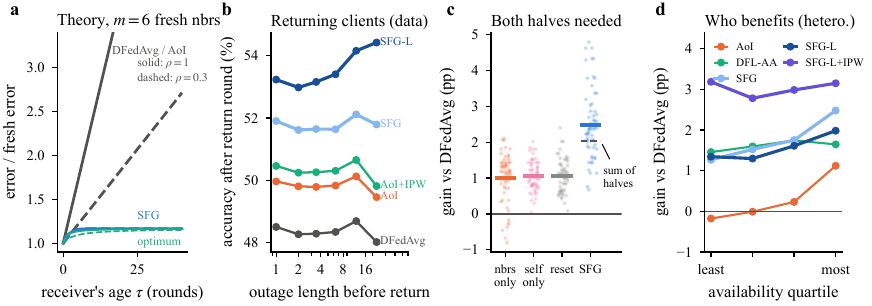}
\caption{\textbf{Why self-inclusive discounting works.} (a) Aggregation error of a returning client relative to a fresh one
(\cref{cor:exact}, six fresh neighbours). (b) Accuracy of returning clients after the local training of their return round,
against their outage length (27 runs per rule). (c) Decomposition on bursty $u{=}0.8$, heterogeneous $u{\approx}0.77$, and FLASH
relaxed and strict (60 paired runs). Discounting neighbours or oneself alone gives 40--42\% of \sfg{}'s gain. (d) Gain by
availability quartile under heterogeneous rates (30 paired runs, s.e.\ 0.12--0.32\,pp, \cref{tab:quartiles}).}
\label{fig:mech}
\end{figure*}

\textbf{Returning clients recover faster.} Right after aggregation, \sfg{}'s returning clients fit their own data \emph{worse} than
DFedAvg's, because discounting the stale model gives up its personalised layers. After the $K$ local steps of the same round,
however, their models are the best at every outage length (\cref{fig:mech}b, Appendix~\ref{app:returns}). Local training quickly
restores the personal part of a model but cannot restore stale shared features, and self-inclusive discounting exploits exactly
this asymmetry.

\textbf{The least available clients benefit too.} Their own model is usually the stalest. Under heterogeneous rates,
neighbour-only AoI does not help the least-available quartile ($-0.18$\,pp), as \cref{thm:stale} predicts, while \sfgl{}\,+\,IPW
gains $+2.8$ to $+3.2$\,pp in every quartile (\cref{fig:mech}d). On the real traces, \sfgl{}'s gain is flat across quartiles
(Appendix~\ref{app:ext-mechanism}).

\subsection{Q4: ages versus lags, and the discount rate}
\label{sec:rate}

\textbf{The weights derived from ages are too mild.} \Cref{prop:blue} also prescribes a \emph{rate}, through weights
$1/(1+\rho\Delta)$ with $\rho$ identifiable from pairwise distances. We implement this as precision-weighted gossip (PWG), which
estimates $\rho$ online from ages. It gains only $+1.48$\,pp, against $+2.44$\,pp for \sfg{} (\cref{tab:extra}). Two effects
predicted by our analysis explain the gap, and each accounts for about half of it (Appendix~\ref{app:rate}).

\textbf{Age laundering explains the first half.} Measured against effective lags instead of ages, distances between live models
grow 1.4--3.1$\times$ faster with staleness, in every rule--dataset combination we logged (\cref{tab:laglaw}). Feeding lags
instead of ages to the \emph{same} estimator (PWG-L) lifts the gain from $+1.48$ to $+2.44$\,pp in 102 of 120 paired runs, halfway
to \sfgl{} ($+3.49$\,pp).

\textbf{The horizon explains the second half.} The receiver carries its persistent error for $H\ge3$ rounds, and for 8--23 rounds
under long or frequent outages, while the dispersion it absorbs fades faster. The effective rate \eqref{eq:effrate} is
therefore 0.8--9 rather than the one-step estimate (\cref{tab:hmeasured,tab:carry}). Raising PWG-L's rate to $H\rho=3$, 10 and 30,
as \cref{prop:horizon} prescribes, lifts its gain to $+3.35$, $+3.56$ and $+3.59$\,pp (\cref{tab:horizon}).

\textbf{A deployable rule matches the oracle.} An oracle that knows the true persistent-error Gram matrix $G$ and minimises the
horizon objective gains $+3.87$\,pp, within 0.1\,pp of \sfgcut{}. A rule that uses one scalar per neighbour thus matches an
oracle with global information, which is, in our view, the strongest evidence that the three principles capture what matters.
The theory identifies the shape of good weights rather than a unique rule. In a pre-registered test on held-out seeds
(Appendix~\ref{app:prereg}), parameters measured per process do no better than constants and trail \sfgcut{}, which is why we
select the margin empirically.

\subsection{Q5: robustness and generality}
\label{sec:robust}

\begin{figure}[t]
\centering
\includegraphics[width=0.92\textwidth]{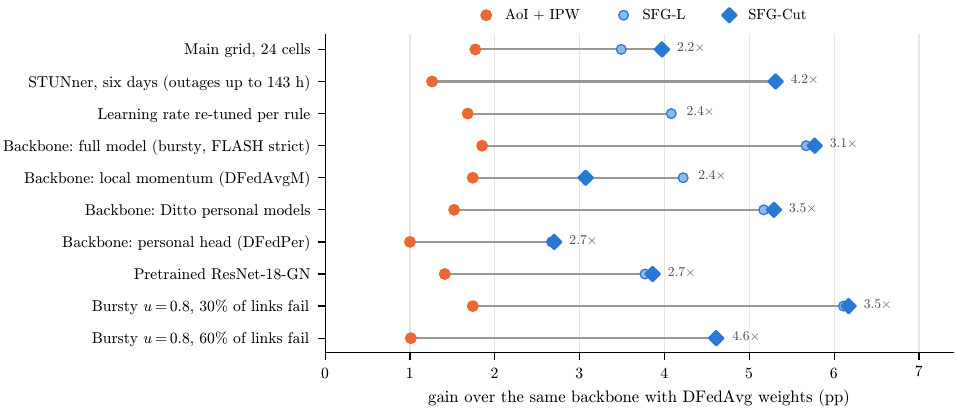}
\caption{\textbf{The advantage holds across protocols, backbones, models and networks.} Gain over the same backbone with DFedAvg
weights, for the strongest AoI baseline and the two lag-based \sfg{} rules. The label gives the ratio of the best \sfg{} gain to the
AoI\,+\,IPW gain. Backbone rows use bursty $u{=}0.8$ and FLASH strict. Link rows use bursty $u{=}0.8$. ResNet-18-GN is the mean
over the six cells of \cref{tab:resnet}. \sfgcut{} was not run in the learning-rate sweep. Sources: \cref{tab:main,tab:week,tab:lr,tab:backbones,tab:resnet,tab:links}.}
\label{fig:robust}
\end{figure}

\Cref{fig:robust} condenses this section. In each of ten settings, the best \sfg{} rule gains 2.2--4.6 times as much as the
strongest AoI baseline under the same protocol. The settings cover learning rates re-tuned per rule, a six-day real trace with
outages of up to 143 hours, three alternative backbones, a pretrained ResNet-18 and unreliable links. The gains also hold for
time-averaged accuracy, on Erd\H{o}s--R\'enyi and regular graphs, for $N\in\{50,300\}$ and for Dirichlet $\alpha\in\{0.1,1\}$
(Appendix~\ref{app:ext-robust}). They are smallest on a ring ($+1.6$\,pp), where each client has only two neighbours to choose
from.

\section{Discussion and limitations}
\label{sec:discussion}

\textbf{What we achieved.} We answered a question that decentralised learning on real devices cannot avoid: how a client that
has just returned from an outage should weight the models in front of it, its own included. The answer is three principles,
each backed by a formal result and confirmed by an ablation. They yield the \sfg{} family, whose lag-based members are the
strongest deployable rules in our benchmark at the cost of one scalar per message, and they explain why AoI rules help little and
why limited merging needs a tight threshold.

\textbf{When the rules help less.} The rules rely on fresher candidates being available. Under regional outages, neighbours fail and return together. A returning client then has no fresher candidate, and \sfg{}'s
gain falls to $+0.7$\,pp (\cref{sec:anatomy}). On a ring, the gain is $+1.6$\,pp. Under i.i.d.\ outages, most returns come after
one round, and the receiver is the stalest candidate in only 37\% of returns (Appendix~\ref{app:stalest}). Under heterogeneous
availability, \sfgcut{} helps the least-available quartile least among the self-inclusive rules. At $u=0.6$, \sfgl{} needs IPW to
match DFL-AA. We therefore recommend IPW whenever availability is heterogeneous.

\textbf{Limitations of the theory.} The analysis concerns the error of aggregation under stylised staleness models with
mean-inherited lags. It treats personal deviations as noise. It establishes that the discount must be steep. The cut structure,
however, requires nesting and import conditions that fail on the long-outage traces, so the margin $L$ is chosen empirically.
Convergence of training under staleness-dependent mixing remains outside the analysis, although the mixing matrices
are exogenous to the optimisation (\cref{rem:convergence}). For rules on maximum-inherited counts such as limited merging, we
show that the counts cannot certify a bound (\cref{prop:counts}), but we do not bound their error when every client runs them. Finally, the advantage of steep rules under i.i.d.\ outages persists against
a Ditto reference and remains unexplained.

\textbf{Limitations of the evaluation.}
The simulation is synchronous. The real availability traces are mapped at random onto a
synthetic geometric graph, whereas real deployments would couple availability with topology. The experiments cover image
classification with a small CNN, a pretrained ResNet-18 check, and up to 300 clients.

The pilot that documents the margin $L$ was run after the test-seed results for $L\in\{1,2\}$ were known. The margin was
therefore not chosen blind. We report both values throughout, and they differ by at most 0.16\,pp on average. Several close
comparisons also remain statistically unresolved, notably \sfgcut{} against tuned limited merging and against the oracle QP.

\textbf{Future work.} Several directions follow naturally. A convergence analysis of training under staleness-dependent mixing is the most important. Rules could also adapt the margin
online from the measured horizon. Asynchronous and event-driven protocols and availability-aware topologies are natural
extensions. Finally, the staleness scalar should be protected against manipulation by malicious clients.

\section{Conclusion}
\label{sec:conclusion}

This paper studied aggregation in decentralised learning when clients go offline, and identified the case that matters most:
after an outage, the stalest model a client aggregates is usually its own. Treating aggregation as estimation of the state of a
fully available network, we derived three principles that a good rule must satisfy. It must discount every candidate, the
receiver included, by staleness relative to the freshest candidate. It must measure staleness as the lag a model has inherited
through averaging, which local training does not remove, rather than as its age. It must discount steeply, because absorbed lag
persists in every later round. Each principle is supported by a formal result and confirmed by an ablation that removes it.

These principles define Self-inclusive Freshness-weighted Gossip (\sfg), a family of aggregation rules that changes only the
averaging weights of DFedAvg and adds one scalar per message. Its lag-based members, \sfgl{} and \sfgcut{}, are to our knowledge
the first staleness-aware rules for decentralised learning whose error guarantees hold for the stale receiver itself. Across
three datasets, eight outage processes and the robustness settings of \cref{fig:robust}, they recover about half or more of the
accuracy lost to outages, roughly twice as much as age-of-information rules, without additional communication rounds or local
computation. Combined with catch-up training, they recover up to three quarters.

The analysis also offers a unifying view of prior work. It explains why the inherited counts of gossip learning help and why
limited merging needs a tight threshold. The main open questions, discussed in \cref{sec:discussion}, are a convergence theory
for staleness-dependent mixing and a principled choice of the cut margin. More broadly, our results suggest that robustness to
intermittent availability depends less on reusing stale information than on knowing how stale each model is, including one's
own.

\section*{Broader impact statement}
This work makes decentralised learning more robust to the intermittent availability of real devices. This can broaden
participation in on-device learning. Freshness weighting can, however, benefit frequently available devices most. Availability
correlates with device quality, charging habits and connectivity. We therefore report per-availability-quartile effects.
\sfgcut{} helps the least-available clients least among the self-inclusive rules, while \sfgl{} with IPW lifts every quartile.
We recommend auditing these effects in deployment. Self-reported staleness metadata must also be protected against
manipulation.

\section*{Reproducibility statement}
All comparisons are paired over seeds 1--5, with data partitions, graph, initialisation, outage trace and minibatch stream
shared across rules (\cref{sec:experiments}). Every hyperparameter, baseline and trace-processing step is specified in
\cref{app:details}. Three files are provided as ancillary files with this submission: \texttt{comparisons.csv} lists every
comparison reported in the paper with its interval, win count and both tests; \texttt{horizon\_derive.py} is the
trace-only replay behind the horizon measurements of \cref{prop:horizon}; and \texttt{c26b\_prereg.md} is the
pre-registration of the predictive test. The FLASH trace is public \citep{yang2021characterizing}, and the STUNner trace is
available from its authors for non-profit use \citep{szabo2019stunner}.

\bibliography{refs}
\bibliographystyle{tmlr}

\appendix

\section{Extended results for Sections 4 and 6}
\label{app:ext}

This appendix gives the full versions of analyses that the main paper summarises.

\subsection{The horizon in practice}
\label{app:ext-theory}
This subsection expands \cref{sec:horizon}.

\textbf{How steep, in practice?} Dispersion is not fully regenerated either. Suppose a fraction $\theta_\varepsilon^2$ survives
each round of local training. The dispersion is then counted $H_\varepsilon$ times instead of once, and $H_\varepsilon$ can be
below 1. To first order, the effective rate becomes
\begin{equation*}
\rho_{\mathrm{eff}}\approx\rho\,\lambda_G\,\frac{H}{H_\varepsilon},
\end{equation*}
where $\lambda_G\approx1$ converts persistent differences into lag differences (Appendix~\ref{app:proofs}). On our models,
$\theta_\varepsilon^2\approx0.8$, while nearly all persistent error survives. Along \sfgl{} trajectories, $H$ is 3.5--23 and
$H/H_\varepsilon$ is 1.5--4 (Appendix~\ref{app:carry}). The effective rate is thus 0.8--9, against a one-step estimate
$\hat\rho\approx0.4$--1.9. It is steeper than one-step weights, though far milder than a horizon that ignores the carry-over.

The import term of part (ii) is weaker. It truncates the weights only if lag absorbed now adds to the lag the receiver will
import. Its slope depends on the neighbours' rule and on the training length. It is negative on the long-outage traces, where
nesting fails. The theory therefore establishes \emph{that} the discount should be steep. A cut additionally requires the nesting
and import conditions of part (ii), so we select its margin empirically.

\subsection{Towards a convergence analysis}
\label{app:ext-convergence}
\begin{remark}[Towards a convergence analysis]\label{rem:convergence}
Our results bound the error of each aggregation, not the convergence of training. Two properties of \sfg{} make such an analysis
tractable. First, its weights depend only on the outage trace, the graph and the rule, never on the models (Appendix~\ref{app:hderive}).
The mixing matrices are therefore exogenous to the optimisation, as in analyses of decentralised SGD with time-varying gossip
\citep{koloskova2020unified}. Second, they coincide with DFedAvg's whenever no candidate is stale, so any analysis of DFedAvg
applies between outages. The remaining obstacle is that the matrices are row-stochastic but not doubly stochastic. Fresher models
receive more total weight than staler ones, which is the intended effect, but it means that the network average is not preserved.
Corrections for column sums, as in push-sum methods \citep{assran2019sgp}, are one route to a convergence guarantee. We leave this
analysis to future work.
\end{remark}

\subsection{What do outages cost?}
\label{app:ext-anatomy}
This subsection expands \cref{sec:anatomy}.

Before comparing rules, we ask where the damage comes from.

\textbf{The rate matters, the structure does not.} At matched unavailability, \emph{which} clients fail together barely changes
DFedAvg's loss. The difference is at most 0.78\,pp per dataset over 540 runs (Appendix~\ref{app:anatomy}).

\textbf{Sharing a stale model is worse than silence.} What an outage disables matters a great deal. A client that cannot train
but keeps sharing its frozen model costs DFedAvg \emph{more} than a full outage: 10.3--12.6\,pp against 7.1--9.6\,pp at $u=0.8$.
Under \sfg{}, whose receivers discount that model, the same client costs only 1.2--2.2\,pp. Undiscounted stale parameters are
therefore a first-order cost. An aggregation rule alone can remove it.

\textbf{A limit of any staleness-based rule.} When neighbours fail and return together, as in regional outages, a returning
client finds no fresher candidate. \sfg{}'s gain then falls from $+3.8$ to $+0.7$\,pp.

\subsection{Main results, all rules}
\label{app:ext-main}
This subsection expands \cref{sec:main}.

\Cref{tab:mainfull} lists every rule of the main grid, including the controls and the oracle.

\begin{table*}[h]
\centering
\caption{\textbf{Main results, all rules} on 3 datasets, 8 outage processes (five synthetic, three real) and 5 seeds. Accuracy is averaged
over the eight processes. Gain: paired difference to DFedAvg in percentage points, mean over 120 runs, in mean and
10th-percentile (p10) client accuracy, with a 95\% stratified $t$-interval over dataset--seed units. Wins: paired runs with a
positive gain. Cells: dataset--process cells with a positive mean gain. Local steps: computation relative to DFedAvg. All
rules send the same models as DFedAvg. The staleness-aware rules add one scalar, and model memory stores one model per
neighbour. $^\dagger$Gossip-learning update counts, inherited by the maximum on merging. $^\ast$Oracle consistency check that
requires global information unavailable to a receiver. Further variants are
in \cref{tab:extra,tab:inherit,tab:horizon}. Bold: best value in each column, separately for rules without and with catch-up training (extra computation). The oracle is excluded.}
\label{tab:mainfull}
\small
\setlength{\tabcolsep}{3.9pt}
\adjustbox{max width=\linewidth}{%
\begin{tabular}{l ccc cc ccc}
\toprule
 & \multicolumn{3}{c}{Mean accuracy (\%)} & \multicolumn{2}{c}{Gain vs DFedAvg (pp)} & Wins & Cells & Local \\
\cmidrule(lr){2-4}\cmidrule(lr){5-6}
Method & CIFAR-10 & CIFAR-100 & FEMNIST & mean [95\% CI] & p10 & /120 & $>0$ & steps \\
\midrule
\multicolumn{9}{l}{\emph{Baselines and controls}} \\
Local only & 63.6 & 23.3 & 73.3 & $-$4.23 [$-$4.52, $-$3.95] & $-$5.08 & 5 & 1 & 1.00 \\
DFedAvg & 66.8 & 25.2 & 80.8 & -- & -- & -- & -- & 1.00 \\
Model memory & 56.0 & 15.8 & 54.3 & $-$15.60 [$-$16.05, $-$15.14] & $-$17.43 & 0 & 0 & 1.00 \\
Update amplification ($\times2$) & 65.3 & 23.8 & 78.9 & $-$1.63 [$-$1.72, $-$1.54] & $-$3.31 & 9 & 0 & 1.00 \\
Relative-variance combiner & 65.6 & 23.7 & 77.5 & $-$2.03 [$-$2.17, $-$1.89] & $-$2.30 & 0 & 0 & 1.00 \\
Copy a random candidate & 65.5 & 25.5 & 79.1 & $-$0.97 [$-$1.17, $-$0.77] & $-$1.70 & 24 & 4 & 1.00 \\
Placebo cut (random, same size) & 66.8 & 25.9 & 80.3 & +0.00 [$-$0.13, 0.13] & $-$0.10 & 71 & 13 & 1.00 \\
\midrule
\multicolumn{9}{l}{\emph{Age-of-information rules}} \\
AoI (neighbour-only, ages) & 68.0 & 26.2 & 81.8 & +1.04 [0.98, 1.10] & +1.19 & 113 & 22 & 1.00 \\
DFL-AA (relative AoI + IPW) & 68.3 & 27.1 & 82.0 & +1.51 [1.45, 1.58] & +1.71 & 115 & 24 & 1.00 \\
AoI + IPW (absolute AoI) & 68.6 & 27.3 & 82.3 & +1.77 [1.72, 1.82] & +2.06 & 120 & 24 & 1.00 \\
\midrule
\multicolumn{9}{l}{\emph{Gossip learning}} \\
Gossip-learning model age$^\dagger$ & 67.8 & 26.1 & 82.0 & +1.01 [0.95, 1.06] & +1.12 & 116 & 24 & 1.00 \\
Limited merging$^\dagger$, $L{=}10$ & 69.4 & 27.8 & 83.4 & +2.55 [2.42, 2.68] & +3.06 & 119 & 24 & 1.00 \\
Limited merging$^\dagger$, $L{=}1$ & 70.6 & 29.7 & 84.1 & +3.84 [3.70, 3.98] & +4.53 & 119 & 24 & 1.00 \\
SFG weights on counts$^\dagger$ & 70.5 & 29.4 & 84.0 & +3.66 [3.51, 3.80] & +4.31 & 120 & 24 & 1.00 \\
SFG weights on counts$^\dagger$ + IPW & 71.1 & 30.2 & \textbf{84.4} & +4.27 [4.16, 4.38] & +5.10 & 120 & 24 & 1.00 \\
\midrule
\multicolumn{9}{l}{\emph{Proposed: the \sfg{} family}} \\
\sfg{} (ages) & 69.4 & 28.0 & 82.9 & +2.44 [2.37, 2.51] & +2.77 & 120 & 24 & 1.00 \\
\sfg{} + IPW (ages) & 69.8 & 28.8 & 83.2 & +2.97 [2.86, 3.07] & +3.32 & 120 & 24 & 1.00 \\
\sfgl{} (lags) & 70.4 & 29.1 & 83.8 & +3.49 [3.33, 3.65] & +4.09 & 118 & 24 & 1.00 \\
\sfgl{} + IPW (lags) & 71.0 & 30.2 & 84.3 & +4.21 [4.08, 4.33] & +4.91 & 120 & 24 & 1.00 \\
\sfgcut{}, $L{=}1$ & 70.7 & 30.0 & 84.1 & +3.97 [3.79, 4.15] & +4.75 & 119 & 24 & 1.00 \\
\sfgcut{}, $L{=}1$, + IPW & \textbf{71.2} & \textbf{30.6} & 84.3 & \textbf{+4.39} [4.23, 4.56] & \textbf{+5.17} & 120 & 24 & 1.00 \\
\midrule
\multicolumn{9}{l}{\emph{Oracle}} \\
Oracle horizon-GLS, $H\rho{=}3$$^\ast$ & 70.6 & 29.5 & 83.9 & +3.71 [3.57, 3.85] & +4.34 & 120 & 24 & 1.00 \\
\midrule
\multicolumn{9}{l}{\emph{With catch-up training (extra computation)}} \\
Catch-up local steps & 68.8 & 26.9 & 82.4 & +1.72 [1.61, 1.83] & +2.11 & 117 & 24 & 1.71 \\
\sfgl{} + catch-up & 71.7 & 30.3 & 84.9 & +4.65 [4.49, 4.81] & +5.59 & 120 & 24 & 1.71 \\
\sfgcut{} + catch-up & \textbf{72.0} & \textbf{31.0} & \textbf{85.1} & \textbf{+5.06} [4.88, 5.25] & \textbf{+6.18} & 120 & 24 & 1.71 \\
\bottomrule
\end{tabular}
}
\end{table*}

\textbf{AoI rules recover a quarter of the gap or less.} Neighbour-only AoI gains $+1.04$\,pp and DFL-AA $+1.51$\,pp. AoI\,+\,IPW on
absolute ages gains $+1.77$\,pp, and resetting a returning client's model $+1.0$\,pp. Every \sfg{} variant beats every AoI variant
on average. The lag-based \sfg{} rules roughly double the best AoI gain.

\textbf{Stale information must be discounted, not reused.} Reusing undiscounted stale models is harmful. Model memory costs
$-15.6$\,pp and update amplification $-1.6$\,pp. Discounting rescues model memory, turning $-15.6$\,pp into $+2.87$\,pp
(Appendix~\ref{app:anatomy}). The relative-variance combiner costs $-2.0$\,pp. In personalised learning, its distances mix
staleness with data heterogeneity.

\textbf{Gossip learning, explained.} Limited merging on gossip learning's maximum-inherited counts gains $+3.84$\,pp at $L=1$ and
$+3.70$\,pp at $L=2$. \sfg{} weights on the same counts gain $+3.66$\,pp. Our principles explain this. Tuned limited merging is
self-inclusive and uses inherited staleness. With a tight threshold, it is also steep. Removing steepness shrinks the gain.
Gossip learning's plain count weights gain only $+1.01$\,pp, and limited merging with a loose threshold ($L=10$) gains
$+2.55$\,pp.

\sfgcut{} is 0.13\,pp above limited merging at $L=1$. We leave this difference unresolved. The advantage of \sfgcut{} is that it
is covered by our bounds. Outside i.i.d.\ outages, the steep rules are on average within 0.2\,pp of \sfgl{}. Under i.i.d.\
outages, they gain up to 2.9\,pp more, and the oracles of \cref{sec:rate} up to 3.2\,pp more.

\textbf{Outage-free DFedAvg is a reference, not a ceiling.} Under i.i.d.\ outages, the steepest rules even exceed outage-free
DFedAvg. \sfgcut{} with $L=1$ does so also at tuned learning rates. We trace this to a mixing effect outside our staleness model,
with three pieces of evidence. Averaging two random candidates gains as much without any outages ($+1.36$\,pp). Copying a random candidate gains $+0.86$\,pp
without outages but costs $-0.97$\,pp under them. Against an outage-free Ditto reference, the excess disappears.
The lead of the steep rules over \sfgl{} under i.i.d.\ outages persists with Ditto (2.3\,pp) and remains unexplained
(\cref{tab:perprocess}, Appendix~\ref{app:iid}).

\textbf{Composability.} The link correction of IPW is orthogonal to freshness weighting. It adds $+0.42$ to $+0.72$\,pp to
\sfg{}, \sfgl{} and \sfgcut{}. \sfgcut{}\,+\,IPW gains $+4.39$\,pp ($[4.23,4.56]$), wins all 120 runs and closes 59--69\% of the
gap. \sfgl{}\,+\,IPW ($+4.21$\,pp) and \sfg{} on counts\,+\,IPW ($+4.27$\,pp) are within 0.2\,pp, an unresolved difference.
Catch-up local steps help on their own, gaining $+1.72$\,pp at $1.71\times$ the computation. \sfgcut{}\,+\,catch-up gains
$+5.06$\,pp, wins all 120 runs and closes 73--75\% of the gap.

\subsection{Ablations}
\label{app:ext-ablation}
This subsection expands \cref{sec:ablation}.

\textbf{Both halves of self-inclusion are needed.} We decompose \sfg{} on four processes: bursty $u{=}0.8$, heterogeneous
$u{\approx}0.77$, and FLASH relaxed and strict (60 paired runs, \cref{fig:mech}c, \cref{tab:decomp}). Discounting only stale neighbours (AoI) gains $+0.99$\,pp over DFedAvg. Discounting only the receiver's own stale model gains
$+1.05$\,pp, as does discarding it on return. Doing both, as \sfg{} does, gains $+2.48$\,pp.
The combination is worth more than the sum of the halves ($+2.04$\,pp), and \sfg{} beats self-only discounting in 60 of 60 runs.
A constant, age-independent self-weight gains only 0.2--0.7\,pp, and as much without outages as with them. \sfg{} adds $+2.1$ to
$+3.5$\,pp on top of this mixing benefit (\cref{tab:controls}).

\textbf{Lags help where many neighbours are stale at once.} \sfgl{} gains $+1.05$\,pp more than \sfg{}
($p_{\text{Holm}}=0.004$). It leads in 18 of 24 cells, with cell-level $p_{\text{Holm}}=0.08$. The gain concentrates where lags and
ages disagree most (\cref{sec:rate}). Under bursty outages at $u=0.8$, \sfgl{} gains $+6.8$\,pp against \sfg{}'s $+3.8$\,pp. On FLASH strict and STUNner, it gains
$+4.5$\,pp against $+1.9$ and $+2.1$\,pp. On the other five processes, the two rules are within 0.5\,pp.
Under heterogeneous rates at $u=0.6$, \sfgl{} ($+1.2$\,pp) trails \sfg{}, DFL-AA ($+1.7$\,pp) and limited merging ($+1.9$\,pp).
Adding IPW recovers it ($+2.5$\,pp).

\subsection{Who benefits}
\label{app:ext-mechanism}
This subsection expands \cref{sec:mechanism}.

\textbf{The least available clients benefit too.} Rarely available clients are the ones whose own model is usually the stalest
(\cref{fig:mech}d, \cref{tab:quartiles}). Under heterogeneous rates, neighbour-only AoI does not help the least-available quartile
($-0.18$\,pp), as \cref{thm:stale} predicts. AoI\,+\,IPW and DFL-AA, in contrast, lift it as much as \sfg{} and \sfgl{}: $+1.33$
and $+1.46$\,pp, against $+1.27$ and $+1.35$\,pp. The link correction is an implicit self-discount, outside the scope of
\cref{thm:stale}.

Among self-inclusive rules, limited merging with $L=2$ lifts that quartile most ($+1.83$\,pp), and \sfgcut{} with $L=1$ least
($+1.21$\,pp). With IPW, \sfgcut{} trails \sfgl{}\,+\,IPW by 0.5--1.0\,pp in every quartile, while \sfgl{}\,+\,IPW gains $+2.8$ to
$+3.2$\,pp in each. On the real traces, \sfgl{}'s gain is flat across quartiles, at $+3.7$ to $+3.9$\,pp. Age-based rules, in
contrast, help frequently available clients most.

\subsection{Robustness and generality}
\label{app:ext-robust}
This subsection expands \cref{sec:robust}.

\textbf{Protocol.} The gains hold for time-averaged accuracy (\cref{tab:time}). They also hold with learning rates re-tuned per
rule, where \sfgl{} gains $+4.1$\,pp in 36 of 36 runs and AoI\,+\,IPW $+1.7$\,pp (\cref{tab:lr}).

\textbf{A six-day real trace.} On a six-day STUNner window, outages last up to 143 hours. \sfgcut{} gains $+5.3$\,pp there, and
$+5.8$\,pp with IPW. \sfgl{} gains $+5.3$\,pp, against $+0.6$\,pp for DFL-AA (\cref{tab:week}). On this longer trace, the
advantage of the self-inclusive rules over DFL-AA widens rather than shrinks.

\textbf{Backbones and a larger model.} Under bursty and FLASH-strict outages, \sfgl{} gains $+4.2$\,pp with local momentum,
$+5.2$\,pp with Ditto-style personal models and $+2.7$\,pp with a personal head. AoI\,+\,IPW gains $+1.0$ to $+1.7$\,pp in the same
settings (\cref{tab:backbones}). With a pretrained ResNet-18-GN, \sfgl{} gains $+3.8$\,pp in 18 of 18 runs, against $+1.4$\,pp
for AoI\,+\,IPW, and $+5.9$\,pp with catch-up (\cref{tab:resnet}). \sfgcut{} matches \sfgl{} across these settings, except with
local momentum ($+3.1$ against $+4.2$\,pp, \cref{tab:backbones,tab:links,tab:resnet}).

\textbf{Graphs, scale, heterogeneity and links.} On CIFAR-10, \sfgl{} gains $+3.5$ to $+8.6$\,pp on Erd\H{o}s--R\'enyi and
regular graphs, with $N\in\{50,300\}$ and Dirichlet $\alpha\in\{0.1,1\}$. It gains $+4.6$ to $+6.1$\,pp when 30--60\% of links fail
(\cref{tab:settings,tab:links}). On a ring, where each client has only two neighbours to choose from, it gains $+1.6$\,pp. With
two classes per client, it gains $+2.1$\,pp.

\section{Proofs}
\label{app:proofs}

\paragraph{Notation.} Throughout, $S$ is the candidate set of receiver $i$, with $|S|=m+1$. We write $e_j=x_j^t-c_t$, and
$\Sigma$ is the error covariance \eqref{eq:cov}. For weights with $\mathbf 1^\top w=1$, we have $\sum_jw_jx_j^t-c_t=\sum_jw_je_j$.
The aggregation error is therefore
\begin{equation}\label{eq:aggerr}
\mathcal E(w):=\E\Big\|\sum_jw_je_j\Big\|^2=w^\top\Sigma w .
\end{equation}

\paragraph{Numerical checks.} We also checked the following statements numerically on random instances: \cref{prop:blue},
\cref{thm:stale,thm:sfg}, \cref{cor:exact}, \cref{prop:lag}(i)--(iii) and (v), \cref{prop:horizon}, and \cref{prop:ft}(ii)--(iii).
The instances draw random $m$, $\rho$, $\gamma$, ages, weights and availability sequences. We used 300--20{,}000 instances per
statement with absolute tolerance $10^{-9}$, and found no violation.

\paragraph{Covariance \eqref{eq:cov}.} The candidate error decomposes as
\begin{equation*}
e_j=\varepsilon_j-\big(c_t-c_{t-\tau_j}\big).
\end{equation*}
Cross terms between $\varepsilon$ and the drift vanish by \cref{ass:staleness}. The dispersions give
$\E\langle\varepsilon_j,\varepsilon_k\rangle=\sigma^2\mathbf 1[j=k]$, and the drift gives
$\E\langle c_t-c_{t-\tau_j},c_t-c_{t-\tau_k}\rangle=\delta^2\kappa(\tau_j,\tau_k)$. Under RW drift, $c_t-c_{t-\tau}$ is a sum of
the last $\tau$ increments. Two such sums share exactly $\min(\tau_j,\tau_k)$ uncorrelated increments. \hfill$\square$

\paragraph{Proof of Proposition~\ref{prop:blue}.}
$\Sigma$ is positive definite because $\sigma^2>0$. We minimise the convex quadratic $w^\top\Sigma w$ over the affine set
$\mathbf 1^\top w=1$. A Lagrange multiplier gives $\Sigma w=\lambda\mathbf 1$, hence
\begin{equation*}
w^\star=\frac{\Sigma^{-1}\mathbf 1}{\mathbf 1^\top\Sigma^{-1}\mathbf 1},\qquad
\mathcal E(w^\star)=\frac{1}{\mathbf 1^\top\Sigma^{-1}\mathbf 1}.
\end{equation*}

\emph{(i)} Under RW, split the drift of each candidate $j\in S$ at the youngest candidate:
\begin{equation*}
c_t-c_{t-\tau_j}=\big(c_t-c_{t-\tau_{\min}}\big)+\big(c_{t-\tau_{\min}}-c_{t-\tau_j}\big).
\end{equation*}
The two terms are sums over disjoint sets of increments, so they are uncorrelated. The first is common to all candidates and
has second moment $\delta^2\tau_{\min}$. Hence
\begin{equation*}
\Sigma=\delta^2\tau_{\min}\mathbf 1\mathbf 1^\top+\Sigma',
\end{equation*}
with $\Sigma'$ as stated. For every $w$ with $\mathbf 1^\top w=1$, this gives $w^\top\Sigma w=\delta^2\tau_{\min}+w^\top\Sigma'w$.
The minimiser is that of $w^\top\Sigma'w$, which depends on ages only through $\Delta_j$.

\emph{(ii)} The off-diagonal entry $\Sigma'_{jk}=\sigma^2\rho\min(\Delta_j,\Delta_k)$ vanishes unless both $\Delta_j>0$ and
$\Delta_k>0$. If at most one candidate has $\Delta_j>0$, then $\Sigma'=\sigma^2\,\mathrm{diag}(1+\rho\Delta_j)$, and
$w^\star_j\propto1/(1+\rho\Delta_j)$. Nothing distinguishes $j=i$ from a neighbour.

\emph{(iii)} Equal ages give $\Delta\equiv0$, so $\Sigma'=\sigma^2I$ and $w^\star$ is uniform. \hfill$\square$

\paragraph{Proof of Theorem~\ref{thm:stale}.}
All entries of $\Sigma$ are non-negative because $\kappa\ge0$, and $w\ge0$. Dropping the off-diagonal terms therefore gives
\begin{equation*}
w^\top\Sigma w\;\ge\;\sum_jw_j^2\Sigma_{jj}\;\ge\;w_{ii}^2\,v(\tau_i).
\end{equation*}
This step is what excludes signed or extrapolating weights. Since $w_{ii}\ge w_{ij}$ for all $j$ and the $m+1$ weights sum to
one, $w_{ii}\ge1/(m+1)$. As $v(\tau_i)=\sigma^2(1+\rho\varphi(\tau_i))$ and $\varphi$ is unbounded, the bound grows without limit
in $\tau_i$.

The same holds on lags under \cref{ass:dynamic}. An outage of $\tau$ rounds puts $\tau$ unit entries in the receiver's row of
$A$, so $G_{ii}\ge\tau$. Since $G\ge0$ and $w\ge0$, the persistent error of the aggregate satisfies
\begin{equation*}
\delta^2\,w^\top Gw\;\ge\;\delta^2w_{ii}^2G_{ii}\;\ge\;\frac{\delta^2\tau}{(m+1)^2}.
\end{equation*}

\emph{Worse than uniform.} Consider uncorrelated errors and equally aged, fresher neighbours. We show that every neighbour-only
discount $g\in(0,1)$ is strictly worse than uniform averaging. Write the unnormalised weights as $(1,g_1,\dots,g_m)$, with
normaliser $Z=1+\sum_lg_l$. For uncorrelated errors with variances $v_0$ (receiver) and $v_l$ (neighbours),
\begin{equation*}
\mathcal E(g)=\frac{v_0+\sum_lg_l^2v_l}{Z^2}.
\end{equation*}
With $g_l\equiv g$ and $v_l\equiv v$, this becomes
\begin{equation*}
\mathcal E(g)=\frac{v_0+mg^2v}{(1+mg)^2},\qquad \mathcal E'(g)=\frac{2m(gv-v_0)}{(1+mg)^3}.
\end{equation*}
The derivative is negative for $g<v_0/v$. Hence $\mathcal E$ is strictly decreasing on $(0,v_0/v]\supset(0,1]$, and every $g<1$
is worse than $g=1$. The minimiser is $g^\star=v_0/v>1$. Weights $(1,g^\star,\dots)$ are proportional to $(1/g^\star,1,\dots)$,
so the optimal rule discounts the receiver by $v/v_0$.

More generally, for uncorrelated errors,
\begin{equation*}
\frac{\partial\mathcal E}{\partial g_k}=\frac{2\big[g_kv_kZ-(v_0+\sum_lg_l^2v_l)\big]}{Z^3},\qquad
\frac{\partial\mathcal E}{\partial g_k}\Big|_{g\equiv1}=\frac{2\big[(m+1)v_k-\sum_{l=0}^mv_l\big]}{(m+1)^3}.
\end{equation*}
Starting from uniform weights, an infinitesimal down-weighting of candidate $k$ therefore lowers the error if and only if its
error exceeds the candidate average. \hfill$\square$

\paragraph{Proof of Theorem~\ref{thm:sfg}.}
By the decomposition in the proof of \cref{prop:blue}(i), $\mathcal E(w)=\delta^2\tau_{\min}+w^\top\Sigma'w$. Using
$\min(\Delta_j,\Delta_k)\le\Delta_j$, $w\ge0$ and $\sum_kw_k=1$,
\begin{align*}
w^\top\Sigma'w&=\sigma^2\sum_jw_j^2+\delta^2\sum_{j,k}w_jw_k\min(\Delta_j,\Delta_k)\\
&\le\sigma^2\sum_jw_j^2+\delta^2\sum_{j,k}w_jw_k\Delta_j
=\sigma^2\sum_jw_j^2+\delta^2\sum_jw_j\Delta_j .
\end{align*}
Under \sfg{}, $w_j=\gamma^{\Delta_j}/Z$ with $Z=\sum_k\gamma^{\Delta_k}\ge1$, because the youngest candidate contributes
$\gamma^0=1$. At most $m$ candidates have $\Delta_j\ge1$. Moreover,
\begin{equation*}
\max_{x>0}\,x\gamma^x=\frac{1}{e\ln(1/\gamma)},
\end{equation*}
so $\sum_jw_j\Delta_j\le m/(e\ln(1/\gamma))$. This gives \eqref{eq:sfgbound}.

For the special cases, suppose first that $\tau_i=0$. Then $\tau_{\min}=0$, and the \sfg{} weights are $\gamma^{\tau_j}$ for
neighbours and $1$ for the receiver. These are exactly the AoI weights. If all ages are equal, the rule is uniform.
\hfill$\square$

\paragraph{Proof of Corollary~\ref{cor:exact}.}
The receiver has age $\tau$ and its $m$ neighbours have age $0$. Hence $\tau_{\min}=0$ and
$\Sigma'=\sigma^2\mathrm{diag}(1+\rho\tau,1,\dots,1)$. For weights proportional to $(h,1,\dots,1)$,
\begin{equation*}
\mathcal E=\sigma^2\,\frac{h^2(1+\rho\tau)+m}{(h+m)^2}.
\end{equation*}
DFedAvg and AoI have $h=1$, and \sfg{} has $h=\gamma^\tau$. The optimum is $h=1/(1+\rho\tau)$ by \cref{prop:blue}(ii). Substituting
gives
\begin{equation*}
\mathcal E^\star=\sigma^2\,\frac{\tfrac1{1+\rho\tau}+m}{\big(\tfrac1{1+\rho\tau}+m\big)^2}=\sigma^2\Big(m+\frac1{1+\rho\tau}\Big)^{-1}.
\end{equation*}
\hfill$\square$

\begin{proposition}[Robustness to the staleness law]\label{prop:robust}
Let \cref{ass:staleness} hold. Consider shift-invariant rules $w_j\propto h(\Delta_j)$ with $h(0)=1$ and $h$ non-increasing.
\begin{enumerate}[nosep,leftmargin=*]
\item[(i)] If $\sup_{\Delta}h(\Delta)\varphi(\tau_{\min}+\Delta)<\infty$, the aggregation error is bounded uniformly in the
receiver's age.
\item[(ii)] Geometric weights $h(\Delta)=\gamma^\Delta$ satisfy this for every sub-exponential law $\varphi$.
\item[(iii)] Hyperbolic weights $h(\Delta)=(1+\rho'\Delta^q)^{-1}$ have unbounded error at returning receivers whenever
$\varphi(\Delta)/\Delta^{2q}\to\infty$.
\end{enumerate}
\end{proposition}

\emph{Proof.} \emph{Upper bound.} By convexity of the squared norm,
\begin{equation*}
\mathcal E(w)\le\sum_jw_j\,v(\tau_j)=\sigma^2+\delta^2\sum_jw_j\varphi(\tau_j).
\end{equation*}
With $w_j=h(\Delta_j)/Z$ and $Z\ge h(0)=1$,
\begin{equation*}
\sum_jw_j\varphi(\tau_j)\le(m+1)\sup_\Delta h(\Delta)\varphi(\tau_{\min}+\Delta),
\end{equation*}
which does not depend on $\tau_i$. Suppose $h=\gamma^\Delta$ and $\varphi$ is sub-exponential, meaning
$\varphi(a+\Delta)e^{-\alpha\Delta}\to0$ for every $\alpha>0$. Then the supremum is finite.

\emph{Lower bound.} As in Theorem~\ref{thm:stale}, $\mathcal E(w)\ge w_{ii}^2v(\tau_i)$. Since $h\le1$,
$w_{ii}=h(\Delta_i)/Z\ge h(\Delta_i)/(m+1)$. Fix $\tau_{\min}$ and let $\tau_i\to\infty$. Then $\varphi(\tau_i)\ge\varphi(\Delta_i)$.
For $h(\Delta)=(1+\rho'\Delta^q)^{-1}$, the bound is at least a constant times
\begin{equation*}
\frac{\varphi(\Delta_i)}{(1+\rho'\Delta_i^q)^2},
\end{equation*}
which diverges under the stated condition. \hfill$\square$

The proposition is a robustness guarantee. It separates geometric from hyperbolic weights under drift that grows faster than
$\Delta^{2q}$. For the two canonical laws, both families keep the error bounded. The empirical rate is analysed in
\cref{sec:rate}.

\paragraph{Identifiability (used by PWG).} The difference of two candidates is
\begin{equation*}
x_j-x_k=\big(c_{t-\tau_j}-c_{t-\tau_k}\big)+\big(\varepsilon_j-\varepsilon_k\big).
\end{equation*}
Under RW, the first term is a sum of $|\tau_j-\tau_k|$ uncorrelated increments. It is uncorrelated with the second term, whose
second moment is $2\sigma^2$. Hence
\begin{equation}\label{eq:distlaw}
\E\|x_j-x_k\|^2=2\sigma^2+\delta^2|\tau_j-\tau_k| .
\end{equation}
A regression of squared distances on age differences with intercept $a$ and slope $b$ therefore identifies $\rho=2b/a$. Under
\cref{ass:dynamic}, the analogue is
\begin{equation*}
\E\|x_j-x_k\|^2=2\sigma^2+\delta^2\|A_j-A_k\|^2 .
\end{equation*}
This equals $2\sigma^2+\delta^2|\ell_j-\ell_k|$ only when the rows of $A$ are $0/1$ and nested. Regressing on lag differences
(PWG-L) is therefore a heuristic. \hfill$\square$

\paragraph{Proof of Proposition~\ref{prop:lag}.}
Under \cref{ass:dynamic}, an available client's persistent error after aggregation and training is $\sum_jw_{ij}p_j^t$. An
unavailable client's is $p_i^t-\xi_t$. By induction from $A^0=0$, $p_i^t=-\sum_{r<t}A^t_{ir}\xi_r$ with the recursion
\eqref{eq:Arec}. Every entry of $A$ lies in $[0,1]$. New columns are $0/1$, and rows of $\widetilde W^t$ are convex combinations.

\emph{(i)} Summing \eqref{eq:Arec} over columns gives \eqref{eq:lagrec}. The increments are uncorrelated with second moment
$\delta^2$, so
\begin{equation*}
\E\|p_i^t\|^2=\delta^2\sum_r(A^t_{ir})^2\le\delta^2\sum_rA^t_{ir}=\delta^2\ell_i^t,
\end{equation*}
because $A\in[0,1]$. Equality holds if and only if the row has entries in $\{0,1\}$.

\emph{(ii), first claim.} Suppose client $i$ has been unavailable for the last $\tau_i^t$ rounds. Its row received an entry $1$ in
each of those columns, so $\ell_i^t\ge\tau_i^t$. After it returns, aggregates and trains, $\tau_i=0$, while
$\ell_i=\sum_jw_{ij}\ell_j$. This is positive unless every candidate has lag $0$.

\emph{(ii), second claim.} We construct the configuration in three steps. First, let neighbours $a$ and $b$ have lag $0$, be offline in the same $D$
rounds, and return together. Second, in the return round, let all of $b$'s other neighbours be offline and all of $a$'s other
neighbours be available with lag $0$. This needs a neighbour of $a$ that is not adjacent to $b$, and common neighbours are offline.
Then $b$ averages only with $a$, which missed the same increments, and keeps a row of $D$ ones. Client $a$ gives total weight $v$
to the two stale candidates (itself and $b$, both of age $D$), and gets a row of $v$'s on the same $D$ columns. Both train. Third,
in the next round, let all their other neighbours be offline. So $a$ aggregates over $\{a,b\}$ and $b$ over $\{b,a\}$.
Both clients are now fresh, have the same ages and availability histories, and share the link $a$--$b$. A rule whose weights
depend only on the candidates' roles, ages and availability histories therefore assigns the same weights $(w,1-w)$ to
(receiver, neighbour) in both aggregations. The persistent errors are
\begin{equation*}
\delta^2D\,\big(v+(1-w)(1-v)\big)^2\ \text{for receiver }a,\qquad
\delta^2D\,\big(v+w(1-v)\big)^2\ \text{for receiver }b .
\end{equation*}
The larger of the two is at least $\delta^2D\big(\frac{1+v}2\big)^2$. Keeping only the less-lagged $a$ gives $\delta^2v^2D$ in
both. The excess is therefore at least
\begin{equation*}
\delta^2D\Big[\Big(\frac{1+v}2\Big)^2-v^2\Big]=\frac{\delta^2D(1-v)(1+3v)}{4}\ge\frac{\delta^2(1-v)D}{4}.
\end{equation*}

\emph{(iii)} If $\mathbf 1^\top\widetilde W^t=\mathbf 1^\top$, then $\mathbf 1^\top\ell^{t+1}=\mathbf 1^\top\ell^t+\mathbf 1^\top\mathbf 1[a^t=0]$.

\emph{(iv)} The aggregate's error is $\sum_jw_j(p_j+\varepsilon_j)$. The dispersions are independent of each other and of the
persistent errors. By convexity and (i),
\begin{equation*}
\E\Big\|\sum_jw_j(p_j+\varepsilon_j)\Big\|^2=\sigma^2\sum_jw_j^2+\E\Big\|\sum_jw_jp_j\Big\|^2
\le\sigma^2\sum_jw_j^2+\delta^2\sum_jw_j\ell_j .
\end{equation*}
With \sfg{} weights on lags, as in the proof of Theorem~\ref{thm:sfg},
\begin{equation*}
\sum_jw_j\ell_j=\ell_{\min}+\sum_jw_j(\ell_j-\ell_{\min})\le\ell_{\min}+m\max_{x\ge0}x\gamma^x .
\end{equation*}

\emph{(v)} The split is the one in (iv). By \cref{ass:dynamic}, the next round's dispersion is regenerated independently of the
past, while the persistent error of the aggregate is propagated by the recursion for $p$. Since $p_j=-\sum_rA_{jr}\xi_r$ with
uncorrelated increments,
\begin{equation*}
\E\Big\|\sum_jw_jp_j\Big\|^2=\delta^2\sum_r\Big(\sum_jw_jA_{jr}\Big)^2=\delta^2w^\top Gw,\qquad G=AA^\top .
\end{equation*}
As $A\in[0,1]$, we have $0\le G_{jk}=\sum_rA_{jr}A_{kr}\le\sum_rA_{jr}=\ell_j$, and likewise $G_{jk}\le\ell_k$. Hence
$G_{jk}\le\min(\ell_j,\ell_k)$. Under nesting, the rows are $0/1$ with supports ordered by inclusion, and
$G_{jk}=|\mathrm{supp}A_j\cap\mathrm{supp}A_k|=\min(\ell_j,\ell_k)$. \hfill$\square$

\paragraph{Proof of Proposition~\ref{prop:counts}.}
Take a receiver $i$ with a neighbour $j$, and let $j$ have $m\ge1$ further neighbours $k_1,\dots,k_m$ that are not neighbours
of $i$. Both histories agree everywhere except for client $j$ in rounds $t_0-D,\dots,t_0-1$. In the first history, $j$ is
available throughout. In the second, $j$ is unavailable in rounds $t_0-D,\dots,t_0-1$ and returns at round $t_0$. In both
histories, the clients $k_1,\dots,k_m$ are available throughout and have lag 0, and $i$ is unavailable at round $t_0$. At round
$t_0$, client $j$ aggregates with $k_1,\dots,k_m$, gives its own model weight $v$, and then trains. At round $t_0+1$, $i$ returns
and all its neighbours other than $j$ are unavailable, so $S_i^{t_0+1}=\{i,j\}$.

\emph{Counts and ages agree.} Under maximum inheritance, $j$'s count after merging with $k_1,\dots,k_m$ at round $t_0$ equals the
largest count among them, which is the count of a client that never missed a round. After training, $j$'s count therefore equals
its count in the first history. Its age at round $t_0+1$ is 0 in both histories, and $i$'s count and age are the same in both.

\emph{Lags differ.} In the second history, $j$'s row of $A$ has ones on the $D$ columns of its outage before round $t_0$. The rows
of $k_1,\dots,k_m$ are zero. After aggregation, $j$'s row is $v$ on those $D$ columns, and training leaves it unchanged. Hence
$\ell_j=vD$ and $G_{jj}=\sum_r A_{jr}^2=v^2D$. In the first history, $j$'s row is zero.

\emph{Error.} A rule that depends only on roles, ages and max-inherited counts sees identical inputs at round $t_0+1$ and assigns
the same weights in both histories. In the second history, $G\ge0$ entrywise and $w\ge0$, so
\begin{equation*}
\E\Big\|\sum_kw_kp_k\Big\|^2=\delta^2\,w^\top Gw\ \ge\ \delta^2w_j^2G_{jj}=\delta^2w_j^2v^2D .
\end{equation*}
Limited merging keeps $j$ for every margin $L\ge0$, because $j$'s count is the largest. \sfg{} weights on counts give $j$ the
weight $h(0)=1$. \sfgcut{} sees $\ell_j-\ell_{\min}\ge vD-\ell_i$ and drops $j$ once this exceeds $L$. In general, its persistent
error is bounded by \cref{prop:horizon}(iii). \hfill$\square$

\paragraph{Proof of Proposition~\ref{prop:horizon}.}
Fix every other aggregation, including the receiver's later ones. The receiver is available at round $t$. Its row of $A^{t+1}$
is therefore $a=w^\top A_S$, with a zero in column $t$.

For $s>t$, the columns of $A^s$ up to $t$ are those of $\Phi^sA^{t+1}$. The later columns are created by later outages and mixed
by the fixed $\Phi$, so they do not depend on $w$. The increments are uncorrelated with second moment $\delta^2$, so
\begin{equation*}
\E\|p_i^s\|^2=\delta^2\|\psi_sa+b_s\|^2+\text{const},\qquad b_s=\sum_{k\ne i}\Phi^s_{ik}A_k^{t+1}.
\end{equation*}
Expanding, with $G=A_SA_S^\top$,
\begin{equation*}
\|\psi_sa+b_s\|^2=\psi_s^2\,w^\top Gw+2\psi_s\,w^\top A_Sb_s+\|b_s\|^2 .
\end{equation*}
Summing over $s=t+1,\dots,T$ gives $H=\sum_s\psi_s^2\ge\psi_{t+1}^2=1$ and $g=A_SB$. Adding $\sigma^2\|w\|^2$ and using
$\rho=\delta^2/\sigma^2$ gives \eqref{eq:Jform}. Finally, $B\ge0$ because $\Phi\ge0$ and $A\ge0$.

\emph{(i)} $I+H\rho G$ is positive definite. The minimiser over $\mathbf 1^\top w=1$ is therefore the GLS solution
\begin{equation*}
w=\frac{(I+H\rho G)^{-1}\mathbf 1}{\mathbf 1^\top(I+H\rho G)^{-1}\mathbf 1}.
\end{equation*}
Under nesting, $G=\Lambda$. This is the covariance of \cref{prop:blue} under RW drift with $(\rho,\tau)$ replaced by
$(H\rho,\ell)$, so parts (i) and (ii) of \cref{prop:blue} apply verbatim.

For the limit, write $\Lambda=\ell_{\min}\mathbf 1\mathbf 1^\top+\Lambda'$ with $\Lambda'_{jk}=\min(\Delta_j,\Delta_k)$. On
$\mathbf 1^\top w=1$,
\begin{equation*}
w^\top(I+H\rho\Lambda)w=\sum_jw_j^2+H\rho\ell_{\min}+H\rho\,w^\top\Lambda'w,\qquad
w^\top\Lambda'w=\int_0^\infty\Big(\sum_{j:\Delta_j>s}w_j\Big)^2ds\ge0 .
\end{equation*}
As $H\to\infty$, the minimisers of this quadratic penalty problem converge to the minimiser of $\sum_jw_j^2$ over the affine set
$\{\mathbf 1^\top w=1,\ w^\top\Lambda'w=0\}$. The second condition holds if and only if the weights on each group of equal
positive $\Delta$ sum to zero. Minimising $\sum_jw_j^2$ then sets every positive-$\Delta$ weight to zero and spreads the unit
mass uniformly over $\{j:\Delta_j=0\}$.

\emph{(ii)} The objective $f(w)=w^\top(I+H\rho\Lambda)w+2\rho g^\top w$ is strictly convex. Its minimiser over the simplex is
unique and satisfies the Karush--Kuhn--Tucker (KKT) conditions
\begin{equation*}
2w_j+s_j=\nu+\mu_j,\qquad s_j:=2H\rho(\Lambda w)_j+2\rho g_j,\qquad \mu_j\ge0,\quad \mu_jw_j=0 .
\end{equation*}
Hence $w_j=\max\{0,(\nu-s_j)/2\}$.

Under nesting, the supports of the rows are ordered by inclusion. If $\ell_k\le\ell_j$, then
$\Lambda_{kl}=\min(\ell_k,\ell_l)\le\Lambda_{jl}$ for every $l$, so $(\Lambda w)_k\le(\Lambda w)_j$ for $w\ge0$. Also
\begin{equation*}
g_j-g_k=\sum_{r\in\mathrm{supp}A_j\setminus\mathrm{supp}A_k}B_r\ge0 .
\end{equation*}
Hence $s_k\le s_j$ and $w_k\ge w_j$. The weights are therefore non-increasing in lag. The support $\{j:s_j<\nu\}$ contains every
candidate whose lag does not exceed the largest kept lag. It is a cut, and it contains $j_0$.

For a kept $j$, subtracting the conditions of $j_0$ gives
\begin{equation*}
w_j=w_{j_0}-H\rho\big[(\Lambda w)_j-(\Lambda w)_{j_0}\big]-\rho(g_j-g_{j_0})\le w_{j_0}-\rho(g_j-g_{j_0}),
\end{equation*}
so $\rho(g_j-g_{j_0})<w_{j_0}\le1$. With nested $0/1$ rows and $B_r=\beta$ on their supports, $g_j=\beta\ell_j$.

\emph{(iii)} For $w\ge0$, $G\le\Lambda$ entrywise (\cref{prop:lag}(v)) gives $w^\top Gw\le w^\top\Lambda w$. Moreover,
$\min(\ell_j,\ell_k)\le\frac12(\ell_j+\ell_k)$ gives $w^\top\Lambda w\le\sum_jw_j\ell_j$. This is at most $\ell_{\min}+L$ for
weights supported on $\Delta_j\le L$. \hfill$\square$

\paragraph{Carry-over through local training.} \Cref{ass:dynamic} lets training regenerate the dispersion entirely and keep
the persistent error entirely. We now relax both. Suppose a fraction $\theta_\varepsilon^2$ of the aggregate's dispersion
energy and a fraction $\theta_\delta^2$ of its persistent-error energy survive each round of local training, independently of
the rest. For $x\in\{\varepsilon,\delta\}$, define
\begin{equation*}
D_x^r=\mathrm{diag}\big(\theta_x\text{ for clients that train at round }r,\ 1\text{ otherwise}\big),\qquad
\Phi_x^s=D_x^{s-1}\widetilde W^{s-1}\cdots D_x^{t+1}\widetilde W^{t+1}D_x^t .
\end{equation*}
The aggregate's dispersion reaches the receiver's model at round $s$ with weight $(\Phi_\varepsilon^s)_{ii}$. Dispersions of
different aggregations are independent. The computation above then gives
\begin{equation}\label{eq:Jcarry}
J(w)=\sigma^2\Big[H_\varepsilon\|w\|^2+\rho\big(H_\delta\,w^\top\widetilde Gw+2\tilde g^\top w\big)\Big]+\text{const},
\qquad H_x=\sum_{s=t+1}^T(\Phi_x^s)_{ii}^2 .
\end{equation}
Here $\widetilde G$ and $\tilde g$ are defined as before, from missed fractions that also decay by $\theta_\delta$ per training
round.

\Cref{ass:dynamic} is the case $\theta_\varepsilon=0$, $\theta_\delta=1$. Then $H_\varepsilon=0$, since the aggregate's
dispersion never reaches a trained model, and the convention of \cref{prop:horizon} counts it once. Dividing by $H_\varepsilon$,
the effective rate is $\rho H_\delta/H_\varepsilon$, and the import slope is scaled by $1/H_\varepsilon$. With $\theta_\delta=1$,
as we measure, $\widetilde G=G$ and $\tilde g=g$. Parts (i)--(iii) of \cref{prop:horizon} then hold with $H\rho$ replaced by
$\rho H_\delta/H_\varepsilon$, and $\rho$ in the import term replaced by $\rho/H_\varepsilon$.

This is a first-order approximation, for two reasons. Carry-over correlates the dispersions of candidates that averaged the same
models. The receiver also later imports part of its aggregate's dispersion back from neighbours. The derivation ignores both.
Unlike \cref{prop:horizon}, it counts dispersion on trained models from round $t+1$. So $H_\varepsilon$ can be below 1, and the
two conventions do not coincide as $\theta_\varepsilon\to0$. The carry-over extension and Theorem~\ref{thm:stale} on lags were not
checked numerically.

\section{A secondary effect: staleness in the fine-tuned geometry}
\label{app:ftgeom}

Age laundering (\cref{sec:lags}) is not the only reason why a rate fitted to raw distances is too small. Local fine-tuning also
changes the appropriate distance-based calibration. The receiver runs $K$ local steps from the aggregate, so the relevant error
is the part that survives this fine-tuning.

\begin{proposition}[Fine-tuned geometry]\label{prop:ft}
Model receiver $i$'s local fine-tuning as the affine map
\begin{equation*}
\mathrm{FT}_i(x)=x_i^\star+M(x-x_i^\star),\qquad M=(I-\alpha\nabla^2f_i)^K .
\end{equation*}
This is exact for $K$ gradient steps of size $\alpha$ on a quadratic local loss with Hessian
$\nabla^2f_i=\sum_k\lambda_kq_kq_k^\top$ and $0\le\alpha\lambda_k\le2$. Assume RW drift whose increments $\xi_r$ share a
covariance $\Sigma_\xi$, and dispersions that share a covariance $\Sigma_\varepsilon$.
\begin{enumerate}[nosep,leftmargin=*]
\item[(i)] The post-fine-tuning error $\mathrm{FT}_i(\bar x)-\mathrm{FT}_i(c_t)$ of an aggregate $\bar x$ has the covariance
\eqref{eq:cov}, with $(\sigma^2,\delta^2)$ replaced by
\begin{equation*}
\sigma_M^2=\mathrm{tr}(M\Sigma_\varepsilon M),\qquad \delta_M^2=\mathrm{tr}(M\Sigma_\xi M).
\end{equation*}
The optimal weights are therefore those of \cref{prop:blue} with rate $\rho_M=\delta_M^2/\sigma_M^2$.
\item[(ii)] $\rho_M$ is identified by post-fine-tuning distances:
\begin{equation*}
\E\|\mathrm{FT}_i(x_j)-\mathrm{FT}_i(x_k)\|^2=2\sigma_M^2+\delta_M^2|\tau_j-\tau_k| .
\end{equation*}
\item[(iii)] Let $\pi_k$ and $\vartheta_k$ be the fractions of dispersion and drift energy along $q_k$, and let
$\mu_k=(1-\alpha\lambda_k)^2$. Then
\begin{equation*}
\frac{\rho_M}{\rho}=\frac{\sum_k\vartheta_k\mu_k^K}{\sum_k\pi_k\mu_k^K}.
\end{equation*}
Suppose drift is concentrated on flatter directions than dispersion, in the sense that $\vartheta_k/\pi_k$ is non-decreasing in
$\mu_k$. Then $\rho_M\ge\rho$, and $\rho_M/\rho$ is non-decreasing in $K$.
\end{enumerate}
\end{proposition}

\textbf{Why the condition is plausible.} The condition in (iii) fits personalised learning. Fresh dispersion is \emph{created}
by local training. Local training moves each client fastest along the high-curvature directions of its own loss, and these are
also the directions that local fine-tuning restores. Collaborative progress, in contrast, accumulates along slow, flat directions
that a few local steps cannot recover.

\textbf{A limitation.} The shared-covariance assumption has a specific weakness. Each client's dispersion may lie along its
\emph{own} high-curvature directions. Then the dispersions do not share one covariance, and the receiver's dispersion is partly
the personal signal that personalisation should keep. The proposition is therefore a heuristic for the direction of the effect,
not a calibration formula.

\textbf{Two predictions.} The proposition makes two testable predictions. \textbf{(P1)}: distances measured after the receiver's fine-tuning grow faster with age than raw distances, so
$\rho_M>\rho$. \textbf{(P2)}: the ratio $\rho_M/\rho$, and with it the best discount, grows with $K$.

\textbf{Test.} We ran 36 DFedAvg runs: 3 datasets, $K\in\{2,5,10,20\}$ and 3 seeds, under bursty $u{=}0.8$. Every 10 rounds, up to
16 fresh receivers ran their $K$ local steps from their own model and from every live neighbour's model, with a common minibatch
sequence (Appendix~\ref{app:details}).

\emph{P1 holds} in all twelve dataset--$K$ settings (\cref{fig:ftgeom}). The rate fitted to post-fine-tuning distances exceeds the
raw one by a factor of 1.1--1.8. In each setting, 66--94\% of stale pairs have a larger distance ratio after fine-tuning than
before. Here a stale pair has neighbour age at least 1, and there are 311--345 pairs per setting. The share is 70--83\% on
CIFAR-10, 70--94\% on CIFAR-100 and 66--77\% on FEMNIST across $K$. The effect accounts for part of the observed rate. However,
$\rho_M$ stays between 0.24 and 0.64. This is below the $\rho\approx1$--10 at which distance-calibrated weights match \sfg{}
(\cref{tab:gamma}).

\emph{P2 does not hold.} The ratio $\rho_M/\rho$ is largest at $K=5$ on CIFAR-10 and CIFAR-100, and smallest there on FEMNIST. It
falls to 1.1--1.2 at $K=20$ on both CIFAR datasets. The best $\gamma$ does not move with $K$ (\cref{fig:rate}d).

The measurements support the predicted direction of the effect, but not its monotonic dependence on local training. Age
laundering is the larger effect.

\begin{figure}[h]
\centering
\includegraphics[width=0.62\textwidth]{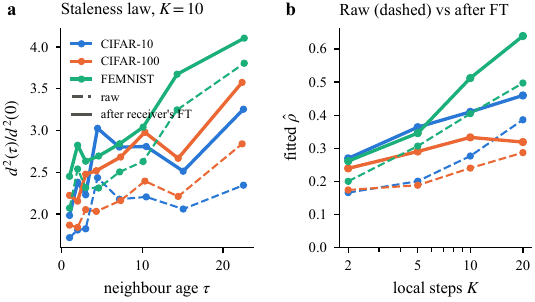}
\caption{\textbf{Fine-tuned geometry} (DFedAvg, bursty $u{=}0.8$, 3 seeds). (a) Squared distance between a fresh receiver's
model and a neighbour's, relative to fresh neighbours, against the neighbour's age. Dashed: before the receiver's $K=10$ local
steps from both models. Solid: after. (b) Rate fitted to the raw and post-fine-tuning distance laws for each $K$.}
\label{fig:ftgeom}
\end{figure}

\paragraph{Proof of Proposition~\ref{prop:ft}.}
\emph{(i)} $\mathrm{FT}_i$ is affine with linear part $M$, and $\sum_jw_j=1$. Hence
\begin{equation*}
\mathrm{FT}_i(\bar x)-\mathrm{FT}_i(c_t)=M(\bar x-c_t)=\sum_jw_jMe_j .
\end{equation*}
We repeat the covariance computation of \cref{eq:cov} with $M\varepsilon_j$ and $M\xi_r$ in place of $\varepsilon_j$ and $\xi_r$.
Since $M$ is symmetric, this gives
\begin{equation*}
\E\langle Me_j,Me_k\rangle=\sigma_M^2\mathbf 1[j=k]+\delta_M^2\min(\tau_j,\tau_k).
\end{equation*}
\Cref{prop:blue} then applies verbatim.

\emph{(ii)} $\mathrm{FT}_i(x_j)-\mathrm{FT}_i(x_k)=M(x_j-x_k)$. Apply the identifiability argument of Appendix~\ref{app:proofs} to
$M\varepsilon$ and $M\xi$.

\emph{(iii)} The map is $M=\sum_k(1-\alpha\lambda_k)^Kq_kq_k^\top$. Hence
\begin{equation*}
\mathrm{tr}(M\Sigma_\varepsilon M)=\sum_k\mu_k^K\,q_k^\top\Sigma_\varepsilon q_k=\sigma^2\sum_k\pi_k\mu_k^K,
\qquad \pi_k=\frac{q_k^\top\Sigma_\varepsilon q_k}{\mathrm{tr}\,\Sigma_\varepsilon},
\end{equation*}
and likewise $\delta_M^2=\delta^2\sum_k\vartheta_k\mu_k^K$. Define
\begin{equation*}
f(K):=\frac{\rho_M}{\rho}=\frac{\sum_k\vartheta_k\mu_k^K}{\sum_k\pi_k\mu_k^K},\qquad f(0)=1 .
\end{equation*}
Then $f(K+1)\ge f(K)$ if and only if $\sum_{k,l}\vartheta_k\pi_l\mu_k^K\mu_l^K(\mu_k-\mu_l)\ge0$. Symmetrising over $(k,l)$, the
left side equals
\begin{equation*}
\frac12\sum_{k,l}\mu_k^K\mu_l^K(\mu_k-\mu_l)\big(\vartheta_k\pi_l-\vartheta_l\pi_k\big).
\end{equation*}
If $\vartheta_k/\pi_k$ is non-decreasing in $\mu_k$, then $\mu_k>\mu_l$ implies $\vartheta_k\pi_l\ge\vartheta_l\pi_k$. Every term is
then non-negative. Thus $f$ is non-decreasing, and $f(K)\ge f(0)=1$. \hfill$\square$

\section{Experimental details}
\label{app:details}

\paragraph{Simulator.} All $N$ client models are stored as one flat tensor and trained in parallel with
\texttt{torch.func.vmap}. Gossip is a dense row-stochastic matrix product. For a given seed, every method shares the data
partition, the graph, the initial model, the outage trace and the minibatch random stream. All comparisons are therefore paired.

Seeds 1--5 are used for evaluation. Seed 0 was used only for pilots. These pilots fixed $\gamma=0.5$, the DFedAvgM learning rate,
the Ditto coefficient, the ResNet learning rate and the margin $L$ of \sfgcut{}. All other hyperparameters are shared defaults.
The learning rate is re-tuned per method in \cref{tab:lr}. Each run uses one NVIDIA A100, with up to four runs sharing a GPU. A
300-round CNN run takes 2--5 minutes.

\paragraph{Data.} \emph{CIFAR-10/100.} We apply Dirichlet label skew with $\alpha=0.3$ over the 50k training images, and every
client holds at least 20 images. Each client's test set has 300 images drawn from the CIFAR test split with the client's own
label distribution. Training uses random horizontal flips. \emph{FEMNIST.} We use the LEAF writer split \citep{caldas2018leaf}.
Per seed, we draw 100 writers with at least 250 samples, each evaluated on its own LEAF test split, with no augmentation.

On every dataset, 15\% of each client's training data is held out as a validation split. It is used only by diagnostics, never
for training or selection.

\paragraph{Models and training.} The CNN has two $5\times5$ convolutions (32 and 64 channels) with max-pooling, a 512-unit
hidden layer and a linear head. It has 0.88M parameters on CIFAR-10. ResNet-18 uses GroupNorm with 2 groups, so that it has no
running statistics, and has 11.2M parameters. Local training is SGD with batch 32, learning rate 0.05 and weight decay
$5\cdot10^{-4}$. Each round runs $K=10$ steps, for 300 rounds (CNN). The graph is a 6-nearest-neighbour geometric graph on uniform
random positions in the unit square, with $N=100$ clients. It is symmetrised and connected if necessary.

\paragraph{ResNet-18 protocol.} Trained from scratch on a few hundred images per client, ResNet-18-GN reaches only about 55\% on
CIFAR-10 without outages. To evaluate a stronger backbone, we start all clients from a common backbone pretrained centrally on
\emph{public} data from the other CIFAR dataset.
For the CIFAR-10 federation, we pretrain on CIFAR-100, reaching 69.6\% CIFAR-100 test accuracy. For the CIFAR-100 federation,
we pretrain on CIFAR-10, reaching 91.9\% test accuracy.
Pretraining uses SGD with Nesterov momentum, a one-cycle schedule with peak learning rate 0.1, batch 256, crop-and-flip
augmentation and 40 epochs. The classifier head is re-initialised. Federated training then runs 150 rounds with learning rate
0.03. This rate was chosen on seed 0 from $\{0.01,0.03,0.05\}$ under bursty outages.

\paragraph{Outage processes.} Synthetic traces are resampled until their realised unavailability is within $\pm0.005$ of the
target $u$. The exception is heterogeneous rates at $u=0.8$, where clipping makes the target unreachable.
\emph{i.i.d.} Each client is off independently with probability $u$ in every round.
\emph{Bursty.} A two-state Gilbert--Elliott chain per client, with stationary unavailability $u$ and mean outage length 10
rounds.
\emph{Heterogeneous.} Per-client rates are drawn from $\mathrm{Beta}(2u,2(1-u))$, rescaled to mean $u$ and clipped at 0.98.
Each client then follows a Gilbert--Elliott chain with mean outage length 5. Clipping limits the realised rate to 0.77 for target
0.8.
\emph{Regional} (anatomy only). Regions are $k$-means clusters of the node positions (spatial), or a random permutation of
the same labels (scattered). Each region follows one Gilbert--Elliott chain. Clients additionally fail independently with
probability 0.05. The regional rate is set so that the total matches $u$.

\paragraph{FLASH traces.}
\label{app:flash}
We use the public state traces released with FLASH \citep{yang2021characterizing}, licensed BSD-2-Clause. They are per-device
event logs of 1,000 smartphones, recording screen on/off and lock, charging, and WiFi/cellular state. We forward-fill each
device's state per minute, align it to its first local midnight, and cut 300 rounds of 10 minutes (50 hours). A round counts as
available if the rule holds for at least half of its minutes. \emph{Strict} is the standard ``screen off, charging and on WiFi''
rule. \emph{Relaxed} drops charging.

Devices with at least 50 hours of trace and at least 5\% availability are eligible.
The strict rule leaves 326 devices, with availability 0.180, mean outage 45.5 rounds ($\approx$ 7.6\,h) and longest outage 283
rounds. The relaxed rule leaves 512 devices, with availability 0.384 and mean outage 16.6 rounds.
Each seed draws 100 eligible devices without replacement and assigns them to clients. Among all 1,000 devices, 29.5\% never
satisfy the strict rule.

\paragraph{STUNner trace.}
STUNner \citep{szabo2019stunner} logs the network and battery state of Android phones running a measurement app. The data are
public for non-profit use. We use the three months from November 2015 to January 2016. During this period, the app took a
measurement at least every 30 minutes and on every charger or network change.

A device counts as available when it is plugged in (AC, USB or wireless) and on WiFi. This is the STUNner analogue of
FLASH-strict and is close to the definition used in gossip learning \citep{hegedus2019gossip}. States are forward-filled from the
last measurement. A state older than 45 minutes counts as unavailable, because the phone was off or not reporting.

Each device contributes its first 50-hour window that starts at a local midnight, has measurement coverage of at least 60\% and
availability of at least 5\%. The window is discretised into 300 rounds of 10 minutes, as for FLASH. This yields 729 devices with
availability 0.177, mean outage 8.6\,h (median 4.7\,h, longest 47\,h) and mean coverage 0.89. \Cref{fig:traces} compares the three
real regimes.

\begin{figure}[h]
\centering
\includegraphics[width=\textwidth]{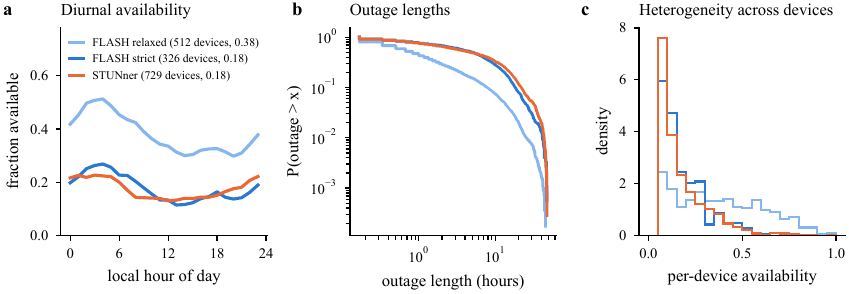}
\caption{\textbf{Real availability traces} (10-minute rounds, 50 hours from local midnight). (a) Availability by hour of day.
The legend gives the number of devices and the mean availability. (b) Distribution of outage lengths. (c) Per-device
availability.}
\label{fig:traces}
\end{figure}

\paragraph{Baselines.} All baselines use the same local training and differ only in aggregation or local work. Weights are
written unnormalised, $w_{ij}\propto h_{ij}$ over the candidate set.

\emph{Stale-information baselines.}
\emph{Model memory.} Each receiver keeps the last model received over every edge. It averages that model in, undiscounted,
when the neighbour is offline. It is the model-level analogue of MIFA \citep{gu2021mifa}, which reuses stale \emph{updates} at a
server. It sends no extra messages.
\emph{Update amplification} (FedAWE-style, \citealp{xiang2024fedawe}). A returning client scales its local update by
$1+\min(\tau_i,c)$, with $c=1$.
\emph{Catch-up.} A returning client runs $K(1+\min(\tau_i,3))$ local steps. On average over the main grid, this is
$1.71\times$ the local computation of DFedAvg.

\emph{Age-of-information baselines.}
\emph{AoI.} $h_{ii}=1$ and $h_{ij}=\gamma^{\tau_j}$ on absolute ages.
\emph{AoI + IPW.} AoI weights divided by an exponentially averaged estimate $\hat q_{ij}$ of the probability that link
$(i,j)$ is live. The estimate uses factor $0.9$ and floor $0.05$, and is updated whenever the receiver is available.
\emph{DFL-AA} \citep{hewa2026dflaa}. As AoI + IPW, but with the age of information measured relative to the youngest live
neighbour. The receiver keeps $h_{ii}=1$, and
\begin{equation*}
h_{ij}=\frac{\gamma^{\,\tau_j-\min_{k\in\mathcal N_i^t}\tau_k}}{\hat q_{ij}} .
\end{equation*}

\emph{Gossip-learning baselines.}
\emph{Gossip-learning model age} \citep{hegedus2019gossip}. Every model carries an update count $n_j$, initially 1. An
aggregate inherits the \emph{largest} count among the candidates it merges, and each round of local training adds 1. Every
candidate, the receiver included, has weight $h_{ij}=n_j$.
\emph{Limited merging} \citep{danner2023limited}. The same inherited counts. A candidate whose count is more than $L$ below
the largest in the candidate set is discarded, the receiver's own model included. We use $L\in\{2,5,10,20,50\}$.
\emph{\sfg{} weights on counts.} $h_{ij}=\gamma^{n_{\max}-n_j}$ on the same max-inherited counts, the receiver included,
optionally with IPW.
\emph{\sfgcut{}} (the cut on lags, i.e.\ limited merging on effective lags). The hard cut applied to mean-inherited counts
$n_j=1+t-\ell_j$. Candidates whose lag exceeds the smallest by more than $L\in\{1,2,10\}$ are dropped, the receiver included.
The rest are weighted by $n_j$.
\emph{Local counts.} Variants using per-client counters without inheritance, with $L\in\{2,10\}$ (\cref{tab:extra}).

\emph{Distance- and theory-based baselines.}
\emph{Relative-variance combiner} \citep{zhao2012diffusion}. It uses no ages. Receiver $i$ tracks, for every candidate $j$
including itself,
\begin{equation*}
v_{ij}\leftarrow(1-\nu)v_{ij}+\nu\|x_j-\bar x_i^{-}\|^2,\qquad h_{ij}=1/v_{ij},\qquad \nu=0.2,
\end{equation*}
where $\bar x_i^{-}$ is its own previous aggregate.
\emph{PWG} (precision-weighted gossip). The weights of \cref{prop:blue}(ii) on relative ages, applied to every candidate:
\begin{equation*}
h_{ij}=\frac{1}{1+\hat\rho\,\Delta_j},\qquad \Delta_j=\tau_j-\tau_{\min}.
\end{equation*}
The rate $\hat\rho$ is estimated online per client by exponentially forgotten least squares of squared distances on age
differences. The forgetting factor is 0.98, with prior $\rho_0=1$ until five observations. \emph{PWG (absolute ages)} applies
$1/(1+\hat\rho\,\tau_j)$ to absolute ages.
\emph{PWG-NP.} A non-parametric variant. It estimates the distance ratio $r(\tau)$ per age bin and uses $h=1/(2r(\tau)-1)$.
\emph{GLS.} The exact minimum-error weights $\Sigma^{-1}\mathbf 1$ of \cref{prop:blue}, with $\Sigma$ from \cref{eq:cov}
under RW drift and the online $\hat\rho$, clipped at zero.
\emph{Metropolis--Hastings.} $w_{ij}=1/(1+\max(d_i,d_j))$ on the live graph.

\emph{Controls.}
\emph{Self-weight control.} Neighbours weight 1, and the receiver a constant $c\in\{0.25,0.5\}$, regardless of age.
\emph{Self-only discount.} Neighbours weight 1. A returning receiver ($\tau_i\ge1$) with at least one live neighbour weights
its own model by $0.5^{\tau_i}$.
\emph{Reset on return.} The same with weight 0. A returning client discards its own model and averages its live neighbours.
\emph{\sfg{} with base self-weight $c$.} \sfg{} weights, with the receiver's weight further multiplied by $c$.
\emph{Neighbour-only on lags or counts.} $h_{ii}=1$ and $h_{ij}=\gamma^{s_j-\min_{k\in\mathcal N_i^t}s_k}$ over live
neighbours. Here $s$ is the effective lag or the deficit of the max-inherited count. This is DFL-AA's relative AoI without IPW,
applied to inherited staleness.
\emph{Copy a random candidate.} Every available receiver replaces its model by one candidate drawn uniformly from itself and
its live neighbours, from a separate random stream. It has the sparsity of $\gamma=0$ without staleness information.
\emph{Model memory + \sfg{}/\sfgl{}.} Remembered models of offline neighbours join the candidate set. Their staleness is
their age or lag at reception plus the rounds since. All candidates get \sfg{} weights.

\paragraph{Our rules.} \sfgcut{} is defined above, next to limited merging.
\emph{\sfg{}.} \Cref{eq:sfg} with ages and $\gamma=0.5$.
\emph{\sfgl{}.} The same weights with ages replaced by effective lags. Every client keeps a real-valued lag $\ell_i$,
initially 0, and sends it with its model. After aggregating, an available client sets $\ell_i\leftarrow\sum_jw_{ij}\ell_j$ with the
weights it just used. Training leaves $\ell_i$ unchanged, and each round without training adds 1. The weights depend on lag
differences and are invariant to the common level.
\emph{PWG-L.} PWG with lags in place of ages, both in the weights (relative lags) and in the online regression.
\emph{+ IPW.} Neighbour weights divided by the link-reception estimate $\hat q_{ij}$ of AoI + IPW. The receiver's own model
has $\hat q=1$.
\emph{+ catch-up.} As above.
\emph{Freshest only} ($\gamma=0$, with $0^0=1$). \sfg{} averages uniformly the candidates of minimal age. On real-valued
lags, ties are rare after the first outage, so \sfgl{} with $\gamma=0$ in effect copies the least-lagged candidate. With
\emph{rounded ties}, the weights use lags rounded to integers, so the rule averages all candidates whose rounded lag is minimal.
\emph{Pinned receiver.} \sfgl{}'s weights and reference (the least-lagged candidate, the receiver included), but with the
receiver's own weight fixed at 1.

\paragraph{Horizon-based and oracle rules.}
\emph{Horizon rates.} PWG-L and the GLS weights of \cref{prop:blue} on lags, with a fixed rate $H\rho\in\{3,10,30\}$ in place
of the online estimate.
\emph{Oracle GLS.} The simulator tracks the matrix $A$ of missed-increment fractions (\cref{sec:lags}). It applies the
minimiser of \cref{prop:horizon}(i), $(I+H\rho\,G)^{-1}\mathbf 1$ normalised over the candidates, with $G=AA^\top$, clipped at
zero and renormalised. Clipping the affine GLS solution generally differs from minimising over the simplex. The rates
$H\rho\in\{3,10\}$ are taken directly from the PWG-L sweep. This oracle uses global information to test the theory. It is not a
deployable rule.
\emph{Derived rules.} These minimise the objective of \cref{prop:horizon} with uniform imports, $g_j=\eta\ell_j/\rho$, over
the simplex. \emph{Sparsemax on lags} drops the quadratic term, so $w=\mathrm{sparsemax}(-\eta\Delta)$. It has a closed form, and
candidates with $\Delta_j\ge1/\eta$ get no weight. The \emph{QP on lags} minimises
\begin{equation*}
\|w\|^2+a\,w^\top\Lambda'w+2\eta\Delta^\top w,\qquad \Lambda'_{jk}=\min(\Delta_j,\Delta_k),\quad a=H\rho .
\end{equation*}
The \emph{oracle QP} uses the exact $G=AA^\top$ in place of $\Lambda'$ and is not deployable. QPs are solved by accelerated
projected gradient (300 steps, exact simplex projection).

\paragraph{Cut variants.} All use $L=1$.
The \emph{placebo cut} keeps the same number of candidates as \sfgcut{}, but drawn at random, with uniform weights. The
\emph{cut on ages} keeps the candidates within one round of the youngest, with uniform weights. \emph{Receiver kept} is
\sfgcut{} with a receiver that is never dropped. The \emph{+ IPW} and \emph{+ catch-up} variants are as above.

\paragraph{Choice of the margin.} The margin $L$ of \sfgcut{} was chosen on the pilot seed 0 from $\{0.5,1,2,3\}$
(\cref{tab:pilot}). This pilot was run after the test-seed results for $L\in\{1,2\}$ were known. It documents a choice made with
those results available. $L=1$ is best on average ($+3.98$ against $+3.89$\,pp for $L=2$), driven by i.i.d.\ outages. Without
them, $L=2$ is slightly better ($+3.87$ against $+3.82$\,pp), as on the test seeds.

\paragraph{Anatomy conditions.} \emph{Structure} uses 300 runs: 5 structures, $u\in\{0.6,0.8\}$, DFedAvg and \sfg{}, 3 datasets
and 5 seeds. The structures are independent bursty chains (mean outage 10), and regional or scattered outages with 5 or 20 groups,
as defined above. \emph{What an outage disables} uses 240 runs, with the same factors and 4 modes.
\emph{Full outage.} The client stops both training and communication.
\emph{Cannot communicate.} The client keeps training on its own data but neither sends nor aggregates, so its age stays 0.
\emph{Cannot train, still shares.} The client keeps sending its frozen model, whose age grows. It does not adopt the
aggregate it would have formed.
A fourth mode communicates and adopts the aggregate but cannot train. It is reported in \cref{tab:anatomy}, not in the
figure.

\paragraph{STUNner, six days.} To test multi-day structure, we cut a second STUNner window of 300 rounds of 30 minutes (6.25
days), with the same availability rule and eligibility criteria. It yields 542 devices with availability 0.158, mean outage
14.6\,h (median 8\,h, longest 143\,h) and mean coverage 0.85. All methods run with the default hyperparameters.

\paragraph{Learning-rate tuning.} For DFedAvg, AoI + IPW, \sfg{} and \sfgl{}, we additionally ran learning rates
$\{0.02,0.1,0.2\}$. These runs cover four processes (bursty $u{=}0.8$, heterogeneous $u{=}0.6$, FLASH strict, STUNner), all three
datasets and seeds 1--3. We compare each method at its own best learning rate per dataset (\cref{tab:lr}). The learning rate is
selected on the same runs it is evaluated on, which favours every method equally. DFedAvg's best rate is 0.1 on two of the three
datasets, and \sfgl{} keeps its gain there.

\paragraph{Statistics.} All comparisons are paired over seeds 1--5 with shared randomness. The outage processes of a (dataset,
seed) pair share the partition, graph and initialisation. The unit of inference is therefore the (dataset, seed) pair. We average
each paired difference over the processes of a unit, which gives 15 units.
\emph{Intervals.} Stratified $t$-intervals over the 15 units, with the variance pooled within datasets (12 degrees of
freedom).
\emph{Primary test.} An exact two-sided sign-flip permutation test over the 15 units ($2^{15}$ sign patterns). It is
Holm-corrected over the family of 70 comparisons listed in the accompanying file \texttt{comparisons.csv}.
\emph{Secondary test.} Wilcoxon signed-rank tests on the 24 cell means, Holm-corrected over the same family.
The file lists every comparison with its interval, wins and both tests. Where the interval and the corrected unit test disagree,
we treat the difference as unresolved.

In nine comparisons, the interval excludes zero but the corrected test does not.
These are copying the least-lagged candidate versus \sfgl{} ($+0.22$\,pp), limited merging versus \sfg{} weights on counts
($+0.04$\,pp), \sfgcut{} versus limited merging at $L=2$ ($+0.12$\,pp) and at $L=1$ ($+0.13$\,pp), least-lagged with rounded ties
versus \sfgl{} ($+0.23$\,pp), GLS on lags versus \sfgl{} ($+0.25$\,pp), \sfgcut{} versus the oracle QP of \cref{tab:horizon}
($+0.10$\,pp), and \sfgcut{} + IPW versus \sfgl{} + IPW ($+0.18$\,pp) and versus \sfg{} on counts + IPW ($+0.12$\,pp).
In 17 further comparisons, the corrected unit test is significant but the corrected cell-level test is not. These differences are
consistent across seeds but not across processes. They include \sfgl{} versus \sfg{} (18/24 cells), \sfgcut{} versus \sfgl{}
(17/24), \sfgcut{} with $L=1$ versus $L=2$ (14/24), and the oracle versus \sfgl{} (14/24).

\emph{Design sensitivity.} Every unit averages the eight processes, so i.i.d.\ outages weigh $1/8$ in every contrast. The file
therefore also reports three further quantities: the number of processes with a positive mean, the range of the mean over the
eight leave-one-process-out designs, and the mean without i.i.d.\ outages with its uncorrected permutation $p$-value. Without
i.i.d.\ outages, \sfgcut{} leads \sfgl{} by $+0.14$\,pp ($p=0.002$) and the oracle by $+0.17$\,pp ($p=0.008$). Three pairs do not
differ: \sfgcut{} and limited merging at $L=1$ ($+0.03$, $p=0.55$), the oracle and \sfgl{} ($-0.03$, $p=0.44$), and \sfgcut{} at
$L=1$ and $L=2$ ($-0.04$, $p=0.33$).

The no-outage control gives the noise floor of paired win counts. There, \sfg{} and DFedAvg are identical up to GPU
nondeterminism, and \sfg{} wins 2 of 15 runs ($-0.04\pm0.15$\,pp).

\paragraph{Backbones.}
\emph{Full model} (default). The whole model is gossiped and fine-tuned.
\emph{Personal head} (DFedPer, \citealp{arivazhagan2019fedper}). The classifier head is never averaged.
\emph{DFedAvgM} \citep{sun2023dfedavgm}. Local heavy-ball momentum 0.9, with learning rate 0.02 chosen on seed 0 from
$\{0.005,0.01,0.02\}$.
\emph{Ditto} \citep{li2021ditto}. Each client also trains a personal model with a proximal pull towards the
post-aggregation shared model, and is evaluated with the personal model. The pull strength $\lambda=0.2$ was chosen on seed 0
from $\{0.05,0.2,1\}$.

\paragraph{Diagnostics.} The diagnostics use a separate random stream and do not affect training.
\emph{Raw staleness law.} Every 5 rounds, for every fresh receiver, we record the squared distance to each live neighbour,
divided by the mean fresh--fresh distance, binned by the neighbour's age.
\emph{Fine-tuned geometry.} Every 10 rounds, up to 16 fresh receivers with at least one fresh and one stale live neighbour
run their $K$ local steps from their own model and from every live neighbour's model, with a common minibatch sequence. We record
distances before and after this fine-tuning, normalised by the fresh-neighbour mean, and validation losses.
\emph{Lag law.} Every 5 rounds, over all live pairs, we record the squared distance, normalised by the mean over pairs with
(near-)zero difference. It is binned by the age difference and by the rounded effective-lag difference. Every run tracks
effective lags with its own weights, whether or not the rule uses them.
\emph{Returning clients.} At every round, for every client that returns from an outage, we record several quantities on
its validation split, binned by its outage length. These are the validation loss of its own stale model, of its post-aggregation
model, of the model after that round's local training, and of its model five rounds later. We also record the loss of each live
neighbour's model and the weight the rule assigns to the client's own model.

\section{Ages versus lags and the discount rate, in detail}
\label{app:rate}

This appendix expands \cref{sec:rate}.

\textbf{The puzzle.} \Cref{prop:blue} also prescribes a \emph{rate}, through weights $1/(1+\rho\Delta)$. The rate $\rho$ is
identifiable from pairwise distances. We implement this as precision-weighted gossip (PWG), which estimates $\rho$ online from
ages ($\hat\rho\approx0.2$--$0.7$). The theory-derived rule turns out too mild. It gains $+1.48$\,pp over DFedAvg, against
\sfg{}'s $+2.44$\,pp (\cref{tab:extra}). Yet PWG with a \emph{fixed} large $\rho\in\{1,10\}$ matches \sfg{} on CIFAR-10
(\cref{tab:gamma}).

\textbf{Lags explain half of it.} Measured against effective lag rather than age, the distance law of DFedAvg trajectories loses
its spurious jump at age 1 (\cref{fig:rate}a). Its fitted rate is 1.4--3.1$\times$ larger in every rule--dataset combination we
logged (\cref{tab:laglaw}). This is age laundering made visible.

Feeding lags instead of ages to the \emph{same} online estimator (PWG-L) raises $\hat\rho$ to 1.0--1.5 under bursty $u{=}0.8$,
heterogeneous $u{\approx}0.77$ and FLASH strict (\cref{fig:rate}b). It lifts the gain from $+1.48$ to $+2.44$\,pp, in 102 of 120
paired runs. This closes half of the gap to \sfgl{} ($+3.49$\,pp, \cref{fig:rate}c). Lags also improve \sfg{} itself, by
$+1.05$\,pp, where the candidates' relative lags and ages differ most. The difference reaches 2.8--4.2 rounds under bursty
$u{=}0.8$, FLASH strict and STUNner (\cref{tab:stalest}).

\textbf{The horizon explains the rest.} Even on lags, the one-step weights $1/(1+\hat\rho\Delta)$ are too mild.
\Cref{prop:horizon} says why. The receiver carries its persistent error for $H\ge3$ rounds, and for 8--23 rounds under long or
frequent outages. The dispersion it absorbs survives for a shorter time. With the measured carry-over, the effective rate
\eqref{eq:effrate} is 0.8--9 (\cref{tab:hmeasured,tab:carry}). Raising PWG-L's rate to $H\rho=3$, 10 and 30, as the horizon
prescribes, lifts its gain from $+2.44$\,pp to $+3.35$, $+3.56$ and $+3.59$\,pp (\cref{tab:horizon}).

\textbf{The theory identifies the shape, not the exact rule.} We ran a pre-registered test on held-out seeds
(Appendix~\ref{app:prereg}). A quadratic program (QP) on lags, with per-process parameters measured on the traces, does no better
than constant parameters outside i.i.d.\ outages ($+0.05$\,pp, $[-0.05,0.16]$). It trails \sfgcut{} with $L=2$ by 0.27\,pp
($[0.14,0.40]$). An additional comparison, outside the pre-registration, shows a $+0.76$\,pp gain over one-step weights (online
PWG-L). The measured horizon thus supports steeper weights, but it does not single out the best rule.

\textbf{A deployable rule matches the oracle.} An oracle QP that knows the true persistent-error Gram matrix $G$ gains
$+3.87$\,pp at $H\rho=10$. This is within 0.1\,pp of \sfgcut{}, an unresolved difference. The oracle shows the same advantage of
steep rules under i.i.d.\ outages (\cref{tab:perprocess}). That a rule using only local information matches an oracle with
global information is, in our view, the strongest evidence that the three principles capture what matters.

Minimisers on scalar lags trail \sfgcut{} by 0.2--0.3\,pp. For example, sparsemax at $\eta=0.5$ trails by 0.21\,pp
($p_{\text{Holm}}=0.004$). The reason lies in heterogeneous rates at $u{\approx}0.77$. There, steep weights lose ($+0.2$ to
$+1.7$\,pp), while \sfgcut{} and the oracles, which average near-ties, hold up ($+1.1$ to $+1.8$\,pp).

\begin{figure*}[t]
\centering
\includegraphics[width=\textwidth]{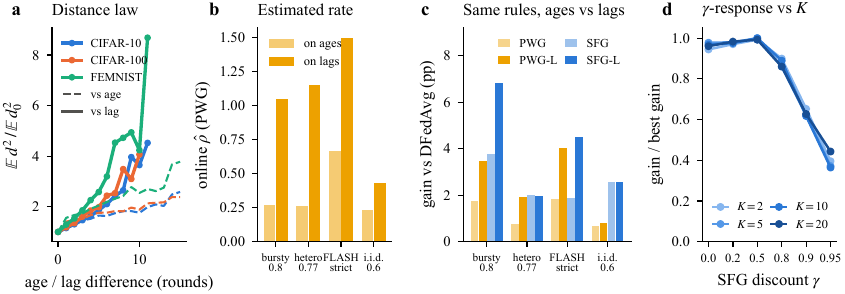}
\caption{\textbf{Ages versus lags.} (a) Squared distance between live models, relative to the zero-difference baseline, against
their age difference (dashed) and their effective-lag difference (solid). DFedAvg trajectories under bursty $u{=}0.8$,
heterogeneous $u{\approx}0.77$, FLASH strict and i.i.d.\ $u{=}0.6$, 3 seeds. (b) The online rate estimate of PWG is
1.9--4.3$\times$ larger on lags. (c) Gain over DFedAvg of the same rules on ages (light) and on lags (dark), 3 datasets
$\times$ 5 seeds. (d) \sfg{}'s response to $\gamma$ has the same shape for every number of local steps $K$ (bursty $u{=}0.8$,
gain normalised by its best value and averaged over datasets).}
\label{fig:rate}
\end{figure*}

\section{Additional results}
\label{app:results}

Throughout this appendix, ``the cut'' and ``cut on lags'' denote \sfgcut{} (\cref{sec:method}). ``Cut$_L$'' denotes its margin.

\paragraph{Anatomy of outages.}
\label{app:anatomy}
\Cref{fig:anatomy,tab:anatomy} expand \cref{sec:anatomy}. The conditions are described in Appendix~\ref{app:details}. Four
findings stand out.

\emph{The rate matters, the structure does not.} At matched unavailability, the structure of outages changes DFedAvg's accuracy
drop by at most 0.78\,pp per dataset at $u=0.6$, and by at most 0.52\,pp at $u=0.8$. Moving between the two rates changes it by
3.8--4.4\,pp.

\emph{Structure matters for the remedy.} \sfg{}'s gain shrinks as outages become more correlated. It falls from $+3.4$/$+3.8$\,pp
under independent outages to $+0.8$/$+0.7$\,pp when five whole regions fail together. Five scattered groups already halve it at
$u=0.8$. The cause is therefore not spatial proximity but correlated \emph{return}. Neighbours that fail and return together are
equally stale, so a returning client has no fresher candidate to shift weight to.

\emph{Under plain averaging, stale sharing is worse than silence.} A client that cannot communicate but keeps training costs about
half as much as a full outage. It gains nothing from \sfg{}, because its model is never stale. A client that cannot train but
keeps sharing its frozen model costs DFedAvg more than a full outage. Under \sfg{}, the same client costs \emph{less} than a full
outage (1.7 against 4.6\,pp at $u=0.8$). Receivers discount its model by its age. A model frozen one or two rounds ago therefore
still contributes, while an old one is ignored. If a client that cannot train also adopts the aggregate it forms, the stale models
it averages spread further. The cost under DFedAvg is then the largest of all at $u=0.8$.

\emph{Why model memory does not help.} Model memory uses older copies, last received while both clients were online. Under
bursty $u=0.8$, the receiver's remembered copy of an offline neighbour is on average 32.5 rounds old. The neighbour's current
frozen model is only 9.7 rounds old. Only 6.5\% of remembered copies are within two rounds, against 19.4\% of frozen models
(\cref{tab:memory}, traces only, 3 seeds). Discounted memory changes \sfgl{}'s accuracy by $-0.61$\,pp, with interval
$[-0.74,-0.49]$ (\texttt{comparisons.csv}).

\begin{figure}[h]
\centering
\includegraphics[width=0.62\textwidth]{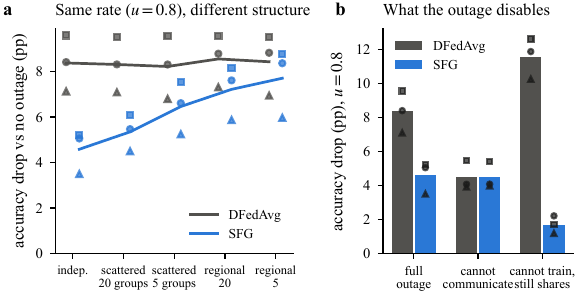}
\caption{\textbf{Anatomy of outages} ($u=0.8$). Markers: CIFAR-10 $\circ$, CIFAR-100 $\square$, FEMNIST $\triangle$.
(a) At a matched rate, outage structure does not change DFedAvg's loss. Correlated outages, however, leave \sfg{} little to work
with. (b) Sharing a model one cannot train hurts more than going silent. \sfg{} neutralises this cost.}
\label{fig:anatomy}
\end{figure}

\begin{table}[h]
\centering
\caption{\textbf{Anatomy of outages.} Accuracy drop relative to DFedAvg without outages (pp, mean over 3 datasets $\times$ 5
seeds), and \sfg{}'s gain over DFedAvg. Within a single dataset, DFedAvg's drop varies across the five structures by at most
0.78\,pp ($u=0.6$) and 0.52\,pp ($u=0.8$).}
\label{tab:anatomy}
\small
\adjustbox{max width=\linewidth}{%
\begin{tabular}{lcccccc}
\toprule
 & \multicolumn{3}{c}{$u=0.6$} & \multicolumn{3}{c}{$u=0.8$} \\
\cmidrule(lr){2-4}\cmidrule(lr){5-7}
 & DFedAvg & \sfg{} & gain & DFedAvg & \sfg{} & gain \\
\midrule
\multicolumn{7}{l}{\emph{Outage structure at matched rate}} \\
Independent (bursty) & 4.47 & 1.11 & +3.36 & 8.38 & 4.59 & +3.79 \\
Scattered, 20 groups & 4.52 & 1.28 & +3.24 & 8.31 & 5.35 & +2.96 \\
Scattered, 5 groups & 4.23 & 1.93 & +2.30 & 8.22 & 6.46 & +1.76 \\
Regional, 20 regions & 4.47 & 2.74 & +1.73 & 8.55 & 7.22 & +1.34 \\
Regional, 5 regions & 4.00 & 3.24 & +0.76 & 8.43 & 7.71 & +0.72 \\
\midrule
\multicolumn{7}{l}{\emph{What the outage disables (bursty chains)}} \\
Full outage & 4.50 & 1.05 & +3.45 & 8.37 & 4.61 & +3.76 \\
Cannot communicate (still trains) & 2.89 & 2.98 & -0.09 & 4.50 & 4.50 & -0.00 \\
Cannot train (shares, adopts aggregate) & 7.70 & 5.82 & +1.88 & 13.29 & 9.03 & +4.26 \\
Cannot train (still shares stale model) & 9.28 & 1.09 & +8.20 & 11.59 & 1.71 & +9.88 \\
\bottomrule
\end{tabular}
}
\end{table}

\begin{table}[h]
\centering
\caption{\textbf{Why model memory does not help \sfgl{}.} For an available receiver and an offline neighbour, we compare two
models. The frozen model is the neighbour's current model, which the receiver would see if the neighbour kept sharing. The
remembered copy was last received while both were online. Traces and graphs only, 3 seeds.}
\label{tab:memory}
\small
\adjustbox{max width=\linewidth}{%
\begin{tabular}{lcccc}
\toprule
 & \multicolumn{2}{c}{mean age (rounds)} & \multicolumn{2}{c}{within 2 rounds (\%)} \\
\cmidrule(lr){2-3}\cmidrule(lr){4-5}
Outage process & frozen model & remembered copy & frozen model & remembered copy \\
\midrule
i.i.d.\ $u{=}0.6$ & 2.5 & 7.1 & 64.3 & 16.5 \\
bursty $u{=}0.6$ & 9.6 & 18.1 & 19.8 & 9.1 \\
bursty $u{=}0.8$ & 9.7 & 32.5 & 19.4 & 6.5 \\
hetero $u{=}0.6$ & 4.9 & 11.2 & 36.5 & 15.1 \\
hetero $u{\approx}0.77$ & 4.9 & 17.4 & 37.0 & 10.4 \\
FLASH relaxed & 36.9 & 50.2 & 10.8 & 5.6 \\
FLASH strict & 49.7 & 89.7 & 4.3 & 2.2 \\
STUNner & 61.9 & 106.1 & 3.9 & 2.0 \\
STUNner 6-day & 43.0 & 86.6 & 7.2 & 3.3 \\
\bottomrule
\end{tabular}
}
\end{table}

\paragraph{Controls.}
\label{app:controls}
\Cref{tab:controls} separates the gain from staleness-aware weighting from the effect of lowering the receiver's self-weight.
Lowering every client's self-weight to a constant, independent of age, gains 0.2--0.7\,pp. It does so equally with and
without outages.
Metropolis--Hastings weights are slightly worse than uniform.
\sfg{} gains $+1.9$ to $+3.8$\,pp under outages. Without outages, it is identical to DFedAvg up to GPU nondeterminism
($-0.04\pm0.15$\,pp over 15 runs). This is also the noise floor of our win counts.
The two effects are orthogonal. On top of the constant self-weight $c=0.25$, \sfg{} adds another $+2.1$ to $+3.5$\,pp under
outages. It beats that control in 15 of 15 paired runs in every outage column.

\begin{table}[h]
\centering
\caption{\textbf{Controls.} Gain over DFedAvg (pp), pooled over CIFAR-10, CIFAR-100 and FEMNIST, with wins/runs against DFedAvg.
The last row is paired against the constant-self-weight control instead. Age-independent self-weights help a little, and equally
without outages. \sfg{} helps only when there are stale models to discount, and its gain adds to a tuned self-weight. Bold: best value in each column, last row excluded.}
\label{tab:controls}
\small
\adjustbox{max width=\linewidth}{%
\begin{tabular}{lcccc}
\toprule
Rule & no outage & i.i.d.\ $u{=}0.6$ & bursty $u{=}0.8$ & FLASH strict \\
\midrule
Metropolis--Hastings weights & -0.39 \tiny{(2/15)} & -0.29 \tiny{(3/15)} & -0.08 \tiny{(3/15)} & -0.21 \tiny{(2/15)} \\
constant self-weight $c{=}0.5$ & +0.35 \tiny{(15/15)} & +0.47 \tiny{(13/15)} & +0.29 \tiny{(14/15)} & +0.21 \tiny{(12/15)} \\
constant self-weight $c{=}0.25$ & +0.65 \tiny{(15/15)} & +0.72 \tiny{(13/15)} & +0.52 \tiny{(15/15)} & +0.26 \tiny{(11/15)} \\
AoI + IPW & $\pm$0 by def. & +1.78 \tiny{(15/15)} & +2.46 \tiny{(15/15)} & +1.25 \tiny{(15/15)} \\
\textbf{SFG} & -0.04 \tiny{(2/15)} & +2.59 \tiny{(15/15)} & +3.79 \tiny{(15/15)} & +1.89 \tiny{(15/15)} \\
\textbf{SFG} with base self-weight $c{=}0.25$ & +0.64 \tiny{(15/15)} & \textbf{+3.20} \tiny{(15/15)} & \textbf{+4.03} \tiny{(15/15)} & \textbf{+2.33} \tiny{(15/15)} \\
Copy a random candidate (no staleness information) & \textbf{+0.86} \tiny{(10/15)} & -0.42 \tiny{(5/15)} & -0.79 \tiny{(1/15)} & -2.15 \tiny{(0/15)} \\
\midrule
\textbf{SFG} ($c{=}0.25$) $-$ constant self-weight $c{=}0.25$ & -0.01 \tiny{(3/15)} & +2.48 \tiny{(15/15)} & +3.51 \tiny{(15/15)} & +2.07 \tiny{(15/15)} \\
\bottomrule
\end{tabular}
}
\end{table}

\paragraph{Both halves.} \Cref{tab:decomp} gives the numbers behind \cref{fig:mech}c.

\begin{table}[h]
\centering
\caption{\textbf{Both halves of the discount.} Bursty $u{=}0.8$, heterogeneous $u{\approx}0.77$, and FLASH relaxed and strict,
3 datasets $\times$ 5 seeds. Wins/runs against DFedAvg are in brackets. The last column gives paired wins of \sfg{} over the rule. Bold: best gain.}
\label{tab:decomp}
\small
\adjustbox{max width=\linewidth}{%
\begin{tabular}{lcc}
\toprule
Discounted & Gain vs DFedAvg (pp) & \sfg{} wins \\
\midrule
Neighbours only (AoI) & +0.99 \tiny{(53/60)} & 58/60 \\
Receiver only ($0.5^{\tau_i}$) & +1.05 \tiny{(60/60)} & 60/60 \\
Receiver only, reset on return & +1.05 \tiny{(60/60)} & 60/60 \\
Both, relative (\sfg{}) & \textbf{+2.48} \tiny{(60/60)} & -- \\
\bottomrule
\end{tabular}
}
\end{table}

\paragraph{How often is the receiver the stalest candidate, and where do lags differ from ages?}
\label{app:stalest}
\Cref{tab:stalest} measures the premise of the paper directly from the outage traces and graphs, with no training. At every
aggregation of an available client with at least one live neighbour, we record whether the client is returning
($\tau_i\ge1$). If so, we record whether its own model is strictly older than every live neighbour's.

\emph{The premise holds on real traces.} On the three real traces of the main grid, a returning client is the strictly stalest
candidate in 85--93\% of returns, and in 83\% on the six-day window. Under i.i.d.\ outages, most returns come after a single round
and neighbours are often just as stale. The premise then holds less often (37\%).

\emph{Lags differ from ages where \sfgl{} helps.} Lags depend only on the trace, the graph and the weights. The same pass therefore
also tracks them along a DFedAvg trajectory. We define the lag--age discrepancy as the mean over candidates of
\begin{equation*}
\big|(\ell_j-\ell_{\min})-(\tau_j-\tau_{\min})\big| .
\end{equation*}
It measures how far the relative staleness that \sfgl{} sees differs from the one \sfg{} sees. It is largest on the processes
where many neighbours are stale at once ($u\approx0.8$ with homogeneous availability). These are exactly the processes where
\sfgl{} improves on \sfg{}. Relative lags are also more compressed than relative ages. A neighbour that trained after averaging
with stale models is young but lagged, and a returning client's neighbours often carry lag of their own.

\begin{table}[h]
\centering
\caption{\textbf{The stalest model in the room, and where lags differ from ages.} kNN graph, $N=100$, 300 rounds, mean over 3
seeds. Traces and graphs only, with lags along DFedAvg trajectories. Spread: mean over aggregations of the largest relative age or
lag among the candidates. Last column: paired gain of \sfgl{} over \sfg{} (3 datasets $\times$ 5 seeds).}
\label{tab:stalest}
\small
\setlength{\tabcolsep}{4pt}
\adjustbox{max width=\linewidth}{%
\begin{tabular}{lcccccc}
\toprule
Outage process & returning & strictly & no live & lag--age & spread of & \sfg-L $-$ \sfg{} \\
 & receivers (\%) & stalest (\%) & neighbour (\%) & discrepancy & ages / lags & (pp) \\
\midrule
i.i.d.\ $u{=}0.6$ & 59.9 & 36.7 & 3.2 & 0.8 & 3.2 / 1.6 & $-$0.01 \\
bursty $u{=}0.6$ & 14.8 & 79.7 & 2.9 & 1.3 & 4.9 / 3.1 & +0.43 \\
bursty $u{=}0.8$ & 39.6 & 68.1 & 21.6 & 2.8 & 8.3 / 2.7 & +3.03 \\
hetero $u{=}0.6$ & 26.6 & 63.9 & 2.5 & 1.2 & 4.2 / 3.5 & $-$0.37 \\
hetero $u{\approx}0.77$ & 51.0 & 53.0 & 10.9 & 1.7 & 5.4 / 3.1 & $-$0.03 \\
FLASH relaxed & 8.6 & 84.8 & 2.9 & 1.7 & 4.5 / 4.4 & +0.27 \\
FLASH strict & 7.9 & 90.3 & 22.4 & 3.2 & 8.4 / 4.1 & +2.64 \\
STUNner & 7.0 & 93.0 & 26.6 & 3.6 & 9.2 / 4.0 & +2.42 \\
STUNner 6-day & 17.1 & 82.8 & 27.8 & 4.2 & 10.9 / 4.5 & +2.96 \\
\bottomrule
\end{tabular}
}
\end{table}

\begin{table}[h]
\centering
\caption{\textbf{Distance law against ages and against lags.} The rate $\rho$ is fitted to $\E d^2/\E d^2_0=1+\tfrac{\rho}{2}x$ over
all live pairs. Here $x$ is the age difference or the (rounded) effective-lag difference, with $1\le x\le20$. Trajectories of the
given rule, four outage processes, 3 seeds. The rate is 1.4--3.1$\times$ larger against lags.}
\label{tab:laglaw}
\small
\adjustbox{max width=\linewidth}{%
\begin{tabular}{lcccccc}
\toprule
 & \multicolumn{2}{c}{CIFAR-10} & \multicolumn{2}{c}{CIFAR-100} & \multicolumn{2}{c}{FEMNIST} \\
\cmidrule(lr){2-3}\cmidrule(lr){4-5}\cmidrule(lr){6-7}
Trajectory & ages & lags & ages & lags & ages & lags \\
\midrule
DFedAvg & 0.18 & 0.45 & 0.20 & 0.48 & 0.32 & 0.84 \\
AoI & 0.26 & 0.59 & 0.26 & 0.69 & 0.41 & 1.25 \\
PWG & 0.24 & 0.45 & 0.23 & 0.49 & 0.42 & 0.90 \\
SFG & 0.48 & 0.67 & 0.56 & 0.83 & 0.76 & 1.24 \\
\bottomrule
\end{tabular}
}
\end{table}

\paragraph{Returning clients.}
\label{app:returns}
\Cref{tab:returns} expands \cref{fig:mech}b. It reports three quantities for the returning client's model: the validation loss
right after aggregation, the validation accuracy after the local training of the return round, and the accuracy five rounds later.
We use bursty $u{=}0.8$, heterogeneous $u{\approx}0.77$ and FLASH strict, with 3 datasets $\times$ 3 seeds. Bins mix different
clients, because rarely available clients have longer outages. Comparisons are therefore meaningful within a column.

\begin{table}[h]
\centering
\caption{\textbf{Returning clients.} Loss after aggregation, before local training. Accuracy (\%) after the return round's local
training, and five rounds later. Bold: best value in each column, lowest loss and highest accuracy.}
\label{tab:returns}
\scriptsize
\setlength{\tabcolsep}{3pt}
\adjustbox{max width=\linewidth}{%
\begin{tabular}{lcccccccccccc}
\toprule
 & \multicolumn{3}{c}{outage 1} & \multicolumn{3}{c}{outage 2--4} & \multicolumn{3}{c}{outage 5--16} & \multicolumn{3}{c}{outage 17--60} \\
\cmidrule(lr){2-4}\cmidrule(lr){5-7}\cmidrule(lr){8-10}\cmidrule(lr){11-13}
Rule & loss agg. & acc. & acc.+5 & loss agg. & acc. & acc.+5 & loss agg. & acc. & acc.+5 & loss agg. & acc. & acc.+5 \\
\midrule
DFedAvg & 2.13 & 48.5 & 48.8 & 2.14 & 48.3 & 48.7 & 2.14 & 48.5 & 49.1 & 2.23 & 47.5 & 49.4 \\
AoI & \textbf{2.10} & 50.0 & 50.4 & \textbf{2.10} & 49.8 & 50.2 & \textbf{2.09} & 50.0 & 50.7 & \textbf{2.18} & 48.9 & 50.8 \\
AoI + IPW & 2.14 & 50.5 & 50.9 & 2.16 & 50.3 & 50.7 & 2.18 & 50.5 & 51.2 & 2.36 & 49.0 & 51.2 \\
SFG & \textbf{2.10} & 51.9 & 52.4 & 2.35 & 51.6 & 52.2 & 2.67 & 51.8 & 52.6 & 2.87 & 51.0 & 52.8 \\
SFG-L & 2.13 & \textbf{53.2} & \textbf{53.6} & 2.28 & \textbf{53.1} & \textbf{53.6} & 2.48 & \textbf{53.7} & \textbf{54.4} & 2.73 & \textbf{53.8} & \textbf{55.5} \\
\bottomrule
\end{tabular}
}
\end{table}

\begin{figure}[h]
\centering
\includegraphics[width=0.6\textwidth]{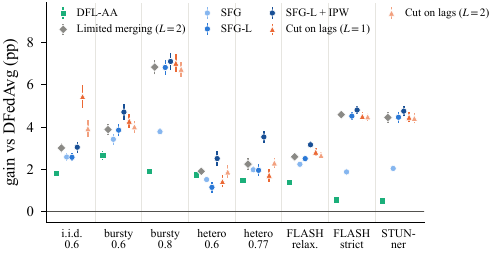}
\caption{\textbf{Gain over DFedAvg by outage process} (3 datasets $\times$ 5 seeds, bars: s.e.\ of the paired gains). Lags help
where many neighbours are stale at once (bursty $u{=}0.8$, FLASH strict, STUNner). Limited merging with $L=2$ on max-inherited
counts (no IPW) tracks \sfgl{} everywhere, slightly ahead under i.i.d.\ and heterogeneous rates. The link correction (IPW) adds most
under heterogeneous rates. \sfgcut{} with $L=1$ leads under i.i.d.\ outages and trails under heterogeneous rates. $L=2$ is the most
even.}
\label{fig:main}
\end{figure}

\begin{table}[h]
\centering
\caption{\textbf{Per-cell results.} Mean personalised accuracy (\%) $\pm$ s.d.\ over 5 seeds. Bold: best zero-overhead rule among
the columns shown, with ties at the displayed precision bolded jointly. Catch-up uses 1.71$\times$ local computation. \sfg{} on
counts + IPW is as good as \sfgl{} + IPW in \cref{tab:main} and is not shown. Limited: limited merging with inherited counts and
$L{=}2$ ($L=1$ is in \cref{tab:inherit}). SFG$^\dagger$: \sfg{} weights on the same counts. SFG-L$_0$: \sfgl{} with $\gamma=0$
(least-lagged candidate only).}
\label{tab:cells}
\scriptsize
\setlength{\tabcolsep}{2.6pt}
\adjustbox{max width=\linewidth}{%
\begin{tabular}{llccccccccccccc}
\toprule
Data & Outage & DFedAvg & AoI+IPW & DFL-AA & Limited & SFG$^\dagger$ & SFG & SFG+IPW & SFG-L$_0$ & SFG-L & SFG-L+IPW & Cut$_1$ & Cut$_2$ & SFG-L+CU \\
\midrule
\multirow{8}{*}{\rotatebox{90}{CIFAR-10}} & i.i.d.\ $u{=}0.6$ & 69.0\tiny{$\pm$0.5} & 70.8\tiny{$\pm$0.5} & 70.9\tiny{$\pm$0.6} & 71.9\tiny{$\pm$0.4} & 72.3\tiny{$\pm$0.5} & 71.6\tiny{$\pm$0.6} & 72.0\tiny{$\pm$0.6} & 73.9\tiny{$\pm$0.7} & 71.5\tiny{$\pm$0.7} & 71.9\tiny{$\pm$0.6} & \textbf{74.0}\tiny{$\pm$0.6} & 72.3\tiny{$\pm$0.5} & 73.9\tiny{$\pm$0.7} \\
 & bursty $u{=}0.6$ & 68.9\tiny{$\pm$0.4} & 71.9\tiny{$\pm$0.4} & 71.7\tiny{$\pm$0.4} & 73.0\tiny{$\pm$0.6} & 72.8\tiny{$\pm$0.6} & 72.6\tiny{$\pm$0.5} & 73.2\tiny{$\pm$0.7} & 73.3\tiny{$\pm$0.6} & 73.0\tiny{$\pm$0.6} & \textbf{73.7}\tiny{$\pm$0.5} & 73.3\tiny{$\pm$0.5} & 73.1\tiny{$\pm$0.5} & 74.1\tiny{$\pm$0.9} \\
 & bursty $u{=}0.8$ & 65.3\tiny{$\pm$1.0} & 67.6\tiny{$\pm$0.6} & 67.0\tiny{$\pm$0.8} & 71.6\tiny{$\pm$0.4} & 71.6\tiny{$\pm$0.6} & 68.6\tiny{$\pm$0.4} & 68.8\tiny{$\pm$0.6} & \textbf{71.8}\tiny{$\pm$0.6} & 71.6\tiny{$\pm$0.6} & \textbf{71.8}\tiny{$\pm$0.6} & \textbf{71.8}\tiny{$\pm$0.6} & 71.6\tiny{$\pm$0.6} & 73.0\tiny{$\pm$0.7} \\
 & hetero $u{=}0.6$ & 69.0\tiny{$\pm$0.8} & 70.4\tiny{$\pm$0.8} & 70.5\tiny{$\pm$0.9} & 70.9\tiny{$\pm$1.1} & 70.6\tiny{$\pm$1.0} & 70.4\tiny{$\pm$0.9} & \textbf{71.5}\tiny{$\pm$0.8} & 70.1\tiny{$\pm$1.1} & 70.2\tiny{$\pm$1.0} & \textbf{71.5}\tiny{$\pm$0.9} & 70.4\tiny{$\pm$0.7} & 70.7\tiny{$\pm$0.9} & 72.0\tiny{$\pm$1.0} \\
 & hetero $u{\approx}0.77$ & 66.2\tiny{$\pm$1.0} & 67.6\tiny{$\pm$0.9} & 67.5\tiny{$\pm$0.8} & 68.6\tiny{$\pm$1.3} & 68.4\tiny{$\pm$1.4} & 68.4\tiny{$\pm$0.9} & 69.1\tiny{$\pm$1.0} & 66.9\tiny{$\pm$1.8} & 68.3\tiny{$\pm$1.2} & \textbf{69.6}\tiny{$\pm$0.9} & 68.0\tiny{$\pm$1.5} & 68.7\tiny{$\pm$1.5} & 70.6\tiny{$\pm$1.6} \\
 & FLASH relaxed & 68.1\tiny{$\pm$0.6} & 69.8\tiny{$\pm$0.9} & 69.6\tiny{$\pm$0.8} & 71.1\tiny{$\pm$0.8} & 71.1\tiny{$\pm$0.8} & 70.7\tiny{$\pm$0.8} & 71.2\tiny{$\pm$0.9} & 71.2\tiny{$\pm$0.8} & 71.2\tiny{$\pm$0.8} & \textbf{71.8}\tiny{$\pm$0.9} & 71.3\tiny{$\pm$0.7} & 71.3\tiny{$\pm$1.0} & 71.7\tiny{$\pm$1.1} \\
 & FLASH strict & 64.1\tiny{$\pm$0.9} & 65.4\tiny{$\pm$1.1} & 64.6\tiny{$\pm$0.7} & 68.8\tiny{$\pm$0.8} & 68.9\tiny{$\pm$0.9} & 66.4\tiny{$\pm$1.0} & 66.5\tiny{$\pm$0.8} & 68.8\tiny{$\pm$1.0} & 68.7\tiny{$\pm$0.8} & \textbf{69.1}\tiny{$\pm$0.7} & 68.7\tiny{$\pm$0.9} & 68.8\tiny{$\pm$0.9} & 69.3\tiny{$\pm$0.7} \\
 & STUNner & 64.3\tiny{$\pm$0.8} & 65.4\tiny{$\pm$0.4} & 64.9\tiny{$\pm$0.8} & 68.6\tiny{$\pm$0.6} & 68.6\tiny{$\pm$0.7} & 66.4\tiny{$\pm$0.4} & 66.3\tiny{$\pm$0.5} & 68.6\tiny{$\pm$0.7} & 68.6\tiny{$\pm$0.6} & \textbf{68.9}\tiny{$\pm$0.4} & 68.5\tiny{$\pm$0.7} & 68.5\tiny{$\pm$0.7} & 68.8\tiny{$\pm$0.9} \\
 & \textit{no outage} & \multicolumn{13}{l}{\textit{DFedAvg 73.7. Every staleness-based rule coincides with DFedAvg.}} \\
\midrule
\multirow{8}{*}{\rotatebox{90}{CIFAR-100}} & i.i.d.\ $u{=}0.6$ & 28.0\tiny{$\pm$0.6} & 30.2\tiny{$\pm$0.6} & 30.5\tiny{$\pm$0.7} & 32.2\tiny{$\pm$0.7} & 32.9\tiny{$\pm$0.8} & 31.4\tiny{$\pm$0.5} & 32.2\tiny{$\pm$0.6} & \textbf{36.6}\tiny{$\pm$0.4} & 31.4\tiny{$\pm$0.8} & 32.3\tiny{$\pm$0.9} & 36.0\tiny{$\pm$0.8} & 34.0\tiny{$\pm$1.0} & 33.0\tiny{$\pm$1.0} \\
 & bursty $u{=}0.6$ & 27.5\tiny{$\pm$0.4} & 31.2\tiny{$\pm$0.5} & 30.9\tiny{$\pm$0.4} & 32.4\tiny{$\pm$0.6} & 32.3\tiny{$\pm$0.7} & 31.8\tiny{$\pm$0.6} & 32.9\tiny{$\pm$0.9} & 33.0\tiny{$\pm$0.6} & 32.5\tiny{$\pm$0.7} & \textbf{33.8}\tiny{$\pm$0.5} & 33.1\tiny{$\pm$0.7} & 32.6\tiny{$\pm$0.7} & 33.2\tiny{$\pm$0.6} \\
 & bursty $u{=}0.8$ & 23.5\tiny{$\pm$0.4} & 26.2\tiny{$\pm$0.5} & 25.6\tiny{$\pm$0.4} & 31.9\tiny{$\pm$0.6} & 31.9\tiny{$\pm$0.6} & 27.8\tiny{$\pm$0.7} & 28.2\tiny{$\pm$0.8} & 32.2\tiny{$\pm$0.5} & 31.8\tiny{$\pm$0.8} & \textbf{32.5}\tiny{$\pm$0.8} & 32.4\tiny{$\pm$0.5} & 31.8\tiny{$\pm$0.6} & 33.1\tiny{$\pm$0.8} \\
 & hetero $u{=}0.6$ & 27.1\tiny{$\pm$0.8} & 29.5\tiny{$\pm$0.6} & 29.6\tiny{$\pm$0.7} & 29.5\tiny{$\pm$0.7} & 29.0\tiny{$\pm$0.8} & 29.0\tiny{$\pm$0.7} & \textbf{30.9}\tiny{$\pm$0.7} & 28.7\tiny{$\pm$1.0} & 29.0\tiny{$\pm$1.1} & \textbf{30.9}\tiny{$\pm$0.8} & 29.2\tiny{$\pm$1.3} & 29.8\tiny{$\pm$1.1} & 31.0\tiny{$\pm$0.8} \\
 & hetero $u{\approx}0.77$ & 24.3\tiny{$\pm$0.6} & 25.7\tiny{$\pm$0.9} & 26.3\tiny{$\pm$0.6} & 27.3\tiny{$\pm$1.5} & 26.8\tiny{$\pm$1.9} & 26.7\tiny{$\pm$0.9} & 27.8\tiny{$\pm$0.9} & 25.5\tiny{$\pm$2.3} & 26.7\tiny{$\pm$1.9} & \textbf{28.9}\tiny{$\pm$0.9} & 26.7\tiny{$\pm$1.9} & 27.1\tiny{$\pm$1.4} & 29.0\tiny{$\pm$1.3} \\
 & FLASH relaxed & 26.2\tiny{$\pm$0.3} & 28.0\tiny{$\pm$0.3} & 27.7\tiny{$\pm$0.2} & 28.7\tiny{$\pm$0.6} & 28.6\tiny{$\pm$0.6} & 28.2\tiny{$\pm$0.4} & 29.0\tiny{$\pm$0.3} & 28.9\tiny{$\pm$0.6} & 28.5\tiny{$\pm$0.4} & \textbf{29.5}\tiny{$\pm$0.4} & 29.1\tiny{$\pm$0.8} & 28.8\tiny{$\pm$0.3} & 29.3\tiny{$\pm$0.4} \\
 & FLASH strict & 22.8\tiny{$\pm$0.7} & 24.2\tiny{$\pm$0.6} & 23.6\tiny{$\pm$0.7} & 27.0\tiny{$\pm$0.6} & 26.9\tiny{$\pm$0.7} & 24.7\tiny{$\pm$0.7} & 25.0\tiny{$\pm$0.9} & 26.8\tiny{$\pm$0.6} & 27.0\tiny{$\pm$0.6} & \textbf{27.3}\tiny{$\pm$0.6} & 26.9\tiny{$\pm$0.5} & 27.0\tiny{$\pm$0.6} & 27.2\tiny{$\pm$0.9} \\
 & STUNner & 22.5\tiny{$\pm$0.4} & 23.5\tiny{$\pm$0.5} & 22.8\tiny{$\pm$0.8} & 26.2\tiny{$\pm$0.5} & 26.3\tiny{$\pm$0.6} & 24.2\tiny{$\pm$0.7} & 24.3\tiny{$\pm$0.8} & 26.3\tiny{$\pm$0.7} & 26.2\tiny{$\pm$0.3} & \textbf{26.6}\tiny{$\pm$0.5} & 26.3\tiny{$\pm$0.6} & 26.3\tiny{$\pm$0.5} & 26.5\tiny{$\pm$0.4} \\
 & \textit{no outage} & \multicolumn{13}{l}{\textit{DFedAvg 33.1. Every staleness-based rule coincides with DFedAvg.}} \\
\midrule
\multirow{8}{*}{\rotatebox{90}{FEMNIST}} & i.i.d.\ $u{=}0.6$ & 83.7\tiny{$\pm$0.7} & 85.1\tiny{$\pm$1.0} & 84.9\tiny{$\pm$0.9} & 85.7\tiny{$\pm$0.9} & 86.2\tiny{$\pm$0.9} & 85.5\tiny{$\pm$0.9} & 85.7\tiny{$\pm$0.9} & \textbf{87.1}\tiny{$\pm$1.1} & 85.6\tiny{$\pm$0.8} & 85.8\tiny{$\pm$0.7} & \textbf{87.1}\tiny{$\pm$1.1} & 86.2\tiny{$\pm$0.8} & 86.8\tiny{$\pm$0.7} \\
 & bursty $u{=}0.6$ & 83.6\tiny{$\pm$0.9} & 85.5\tiny{$\pm$0.9} & 85.3\tiny{$\pm$1.0} & 86.2\tiny{$\pm$0.9} & 86.2\tiny{$\pm$0.8} & 85.8\tiny{$\pm$1.0} & 86.2\tiny{$\pm$0.9} & 86.4\tiny{$\pm$0.8} & 86.0\tiny{$\pm$1.0} & \textbf{86.6}\tiny{$\pm$0.8} & 86.4\tiny{$\pm$0.9} & 86.2\tiny{$\pm$0.8} & 86.8\tiny{$\pm$0.7} \\
 & bursty $u{=}0.8$ & 79.5\tiny{$\pm$1.1} & 81.9\tiny{$\pm$1.0} & 81.4\tiny{$\pm$1.1} & \textbf{85.4}\tiny{$\pm$0.8} & 85.3\tiny{$\pm$0.7} & 83.2\tiny{$\pm$0.9} & 83.1\tiny{$\pm$0.9} & 85.1\tiny{$\pm$1.0} & \textbf{85.4}\tiny{$\pm$0.7} & 85.3\tiny{$\pm$0.7} & 85.2\tiny{$\pm$0.8} & 85.2\tiny{$\pm$0.8} & 86.4\tiny{$\pm$0.4} \\
 & hetero $u{=}0.6$ & 83.1\tiny{$\pm$0.4} & 84.4\tiny{$\pm$0.6} & 84.4\tiny{$\pm$0.4} & 84.6\tiny{$\pm$0.4} & 84.1\tiny{$\pm$0.6} & 84.3\tiny{$\pm$0.5} & \textbf{84.9}\tiny{$\pm$0.5} & 83.6\tiny{$\pm$0.4} & 83.5\tiny{$\pm$0.3} & 84.3\tiny{$\pm$0.3} & 83.9\tiny{$\pm$0.4} & 84.4\tiny{$\pm$0.1} & 85.2\tiny{$\pm$0.8} \\
 & hetero $u{\approx}0.77$ & 80.4\tiny{$\pm$0.9} & 81.5\tiny{$\pm$0.9} & 81.5\tiny{$\pm$0.7} & 81.7\tiny{$\pm$1.1} & 81.5\tiny{$\pm$1.0} & 81.8\tiny{$\pm$0.8} & 82.8\tiny{$\pm$0.7} & 79.1\tiny{$\pm$1.1} & 81.8\tiny{$\pm$1.1} & \textbf{83.0}\tiny{$\pm$1.0} & 81.5\tiny{$\pm$1.3} & 82.1\tiny{$\pm$1.2} & 83.7\tiny{$\pm$1.1} \\
 & FLASH relaxed & 82.3\tiny{$\pm$0.7} & 83.6\tiny{$\pm$0.8} & 83.6\tiny{$\pm$0.8} & 84.5\tiny{$\pm$0.5} & 84.4\tiny{$\pm$0.4} & 84.4\tiny{$\pm$0.6} & 84.7\tiny{$\pm$0.5} & 84.5\tiny{$\pm$0.5} & 84.4\tiny{$\pm$0.4} & \textbf{84.9}\tiny{$\pm$0.4} & 84.6\tiny{$\pm$0.5} & 84.4\tiny{$\pm$0.6} & 85.1\tiny{$\pm$0.5} \\
 & FLASH strict & 77.6\tiny{$\pm$0.6} & 78.7\tiny{$\pm$0.4} & 78.0\tiny{$\pm$0.9} & \textbf{82.5}\tiny{$\pm$0.9} & 82.3\tiny{$\pm$1.0} & 79.1\tiny{$\pm$0.6} & 79.4\tiny{$\pm$0.4} & 82.4\tiny{$\pm$0.9} & 82.4\tiny{$\pm$0.9} & \textbf{82.5}\tiny{$\pm$0.8} & \textbf{82.5}\tiny{$\pm$0.9} & 82.2\tiny{$\pm$0.9} & 82.6\tiny{$\pm$1.0} \\
 & STUNner & 76.3\tiny{$\pm$0.4} & 77.7\tiny{$\pm$0.2} & 76.9\tiny{$\pm$1.0} & 81.7\tiny{$\pm$0.9} & 81.8\tiny{$\pm$0.8} & 78.7\tiny{$\pm$0.4} & 78.7\tiny{$\pm$0.3} & 81.6\tiny{$\pm$1.1} & 81.8\tiny{$\pm$0.9} & \textbf{82.0}\tiny{$\pm$0.9} & 81.7\tiny{$\pm$0.9} & 81.5\tiny{$\pm$0.9} & 82.5\tiny{$\pm$0.7} \\
 & \textit{no outage} & \multicolumn{13}{l}{\textit{DFedAvg 86.7. Every staleness-based rule coincides with DFedAvg.}} \\
\bottomrule
\end{tabular}
}
\end{table}

\begin{table}[h]
\centering
\caption{\textbf{Backbones} (bursty $u{=}0.8$ and FLASH strict, 3 datasets $\times$ 5 seeds). Gain over the same backbone with
uniform weights, in pp (wins/runs). Bold: best rule in each row.}
\label{tab:backbones}
\small
\adjustbox{max width=\linewidth}{%
\begin{tabular}{lcccccc}
\toprule
Backbone & DFedAvg (\%) & AoI + IPW & Limited, $L{=}2$ & SFG & SFG-L & \sfgcut{} \\
\midrule
Full model (DFedAvg) & 55.5 & +1.85 \tiny{(30/30)} & +5.71 \tiny{(30/30)} & +2.84 \tiny{(30/30)} & +5.67 \tiny{(30/30)} & \textbf{+5.77} \tiny{(30/30)} \\
Local momentum (DFedAvgM) & 55.0 & +1.74 \tiny{(29/30)} & +3.76 \tiny{(26/30)} & +2.11 \tiny{(27/30)} & \textbf{+4.22} \tiny{(28/30)} & +3.07 \tiny{(25/30)} \\
Ditto personal models & 56.3 & +1.52 \tiny{(30/30)} & +5.19 \tiny{(30/30)} & +2.53 \tiny{(30/30)} & +5.17 \tiny{(30/30)} & \textbf{+5.29} \tiny{(30/30)} \\
Personal head (DFedPer) & 53.4 & +1.00 \tiny{(30/30)} & +2.61 \tiny{(30/30)} & +1.56 \tiny{(30/30)} & +2.67 \tiny{(30/30)} & \textbf{+2.70} \tiny{(30/30)} \\
\bottomrule
\end{tabular}
}
\end{table}

\begin{table}[h]
\centering
\caption{\textbf{ResNet-18-GN} from a pretrained backbone (150 rounds, learning rate 0.03, 3 seeds). DFedAvg accuracy (\%) and
gain over DFedAvg (pp, wins/runs). $^\dagger$On gossip-learning counts inherited by the maximum: limited merging with $L=2$, and
\sfg{} weights. CU: catch-up local steps ($1.5\times$ computation on average). Bold: best rule without extra computation in each row.}
\label{tab:resnet}
\small
\setlength{\tabcolsep}{3.5pt}
\adjustbox{max width=\linewidth}{%
\begin{tabular}{llcccccccc}
\toprule
Data & Outage & DFedAvg & AoI+IPW & Limited$^\dagger$ & SFG$^\dagger$ & SFG & SFG-L & Cut & SFG-L+CU \\
\midrule
CIFAR-10 & bursty $u{=}0.8$ & 83.9 & +0.86 \tiny{(3/3)} & +2.09 \tiny{(3/3)} & +2.32 \tiny{(3/3)} & +1.45 \tiny{(3/3)} & +2.28 \tiny{(3/3)} & \textbf{+2.34} \tiny{(3/3)} & +3.59 \tiny{(3/3)} \\
 & FLASH strict & 79.0 & +0.92 \tiny{(3/3)} & \textbf{+1.83} \tiny{(3/3)} & +1.72 \tiny{(3/3)} & +1.03 \tiny{(3/3)} & +1.44 \tiny{(3/3)} & +1.62 \tiny{(3/3)} & +2.24 \tiny{(3/3)} \\
 & STUNner & 67.8 & $-$0.06 \tiny{(1/3)} & +0.62 \tiny{(2/3)} & +0.55 \tiny{(3/3)} & +0.50 \tiny{(2/3)} & \textbf{+0.71} \tiny{(3/3)} & +0.63 \tiny{(2/3)} & +1.13 \tiny{(3/3)} \\
\midrule
CIFAR-100 & bursty $u{=}0.8$ & 30.9 & +3.97 \tiny{(3/3)} & +11.01 \tiny{(3/3)} & +10.97 \tiny{(3/3)} & +6.22 \tiny{(3/3)} & +10.74 \tiny{(3/3)} & \textbf{+11.08} \tiny{(3/3)} & +17.72 \tiny{(3/3)} \\
 & FLASH strict & 31.3 & +2.00 \tiny{(3/3)} & +4.94 \tiny{(3/3)} & +5.10 \tiny{(3/3)} & +3.26 \tiny{(3/3)} & +5.13 \tiny{(3/3)} & \textbf{+5.17} \tiny{(3/3)} & +6.35 \tiny{(3/3)} \\
 & STUNner & 25.5 & +0.74 \tiny{(2/3)} & \textbf{+2.54} \tiny{(3/3)} & +2.21 \tiny{(3/3)} & +1.39 \tiny{(3/3)} & +2.30 \tiny{(3/3)} & +2.34 \tiny{(3/3)} & +4.49 \tiny{(3/3)} \\
\bottomrule
\end{tabular}
}
\end{table}

\paragraph{Discount factor.}
\label{app:gamma}
The discount $\gamma=0.5$ was fixed on a pilot seed. We swept $\gamma$ under bursty $u{=}0.8$ and FLASH strict, on CIFAR-10 and
FEMNIST (\cref{fig:rate}d, \cref{tab:gamma,tab:gammal}).
Both rules are flat for $\gamma\in[0,0.5]$.
\sfg{} also tolerates $\gamma=0.8$. Under bursty outages, $\gamma\ge0.9$ is too mild for \sfg{}, and $\gamma=0.8$ is too
mild for \sfgl{}. FLASH strict tolerates milder discounts.
Neighbour-only AoI stays below $+1.6$\,pp for every $\gamma$. Its deficit is therefore not a tuning problem.
Across all eight processes, $\gamma=0$ is slightly better on average ($+0.22$\,pp) but fragile (Appendix~\ref{app:rate}). We recommend
$\gamma\in[0.2,0.5]$.

\begin{table}[h]
\centering
\caption{\textbf{Discount factor and calibrated weights} (CIFAR-10, gain over DFedAvg in pp, mean $\pm$ s.e.\ over 5 seeds).
\sfg{} is flat for $\gamma\in[0.2,0.8]$ here, and $\gamma=0$ appears in \cref{fig:rate}d. Neighbour-only AoI stays below $+1.6$\,pp
for every $\gamma$. PWG on absolute ages is too mild with its online estimate ($\hat\rho\approx0.2$--$0.7$) or with a small fixed
$\rho$. It matches \sfg{} with a large fixed $\rho$. Bold: best value in each column.}
\label{tab:gamma}
\small
\adjustbox{max width=\linewidth}{%
\begin{tabular}{lcc}
\toprule
Rule & bursty $u{=}0.8$ & FLASH strict \\
\midrule
SFG $\gamma{=}0.2$ & +3.19$\pm$0.30 & +2.17$\pm$0.16 \\
SFG $\gamma{=}0.5$ & \textbf{+3.31}$\pm$0.31 & +2.28$\pm$0.17 \\
SFG $\gamma{=}0.8$ & +3.06$\pm$0.16 & \textbf{+2.35}$\pm$0.15 \\
SFG $\gamma{=}0.95$ & +1.35$\pm$0.11 & +2.15$\pm$0.23 \\
\midrule
AoI $\gamma{=}0.2$ & +1.50$\pm$0.17 & +1.39$\pm$0.14 \\
AoI $\gamma{=}0.5$ & +1.53$\pm$0.21 & +1.37$\pm$0.14 \\
AoI $\gamma{=}0.8$ & +1.42$\pm$0.16 & +1.32$\pm$0.11 \\
AoI $\gamma{=}0.95$ & +0.79$\pm$0.11 & +1.33$\pm$0.20 \\
\midrule
PWG, online $\hat\rho$ & +1.49$\pm$0.10 & +2.06$\pm$0.17 \\
PWG, fixed $\rho{=}0.1$ & +1.40$\pm$0.10 & +1.97$\pm$0.17 \\
PWG, fixed $\rho{=}1$ & +2.94$\pm$0.25 & +2.30$\pm$0.19 \\
PWG, fixed $\rho{=}10$ & +3.15$\pm$0.30 & +2.24$\pm$0.17 \\
PWG, quadratic law & +1.21$\pm$0.11 & +1.91$\pm$0.25 \\
PWG, exact GLS weights & +1.74$\pm$0.09 & +2.10$\pm$0.17 \\
\bottomrule
\end{tabular}
}
\end{table}

\begin{table}[h]
\centering
\caption{\textbf{Settings} (CIFAR-10, gain over DFedAvg in pp, 5 seeds). $^a$Bursty $u{=}0.6$ with the given mean outage length.
$^b$Bursty $u{=}0.8$. Gains are small on a ring, where a returning client rarely has a live neighbour. They are also small under a
pathological two-class split, where collaboration itself loses to local training. Bold: best rule in each row.}
\label{tab:settings}
\small
\adjustbox{max width=\linewidth}{%
\begin{tabular}{llcccccc}
\toprule
Factor & Setting & DFedAvg (\%) & Local & AoI & PWG & \textbf{SFG} & \textbf{SFG-L} \\
\midrule
Mean outage length$^a$ & 2 rounds & 69.2 & -- & +0.39 & +0.39 & \textbf{+1.75} & +1.54 \\
 & 5 rounds & 69.3 & -- & +1.81 & +1.26 & +3.69 & \textbf{+4.21} \\
 & 20 rounds & 68.7 & -- & +1.96 & +2.54 & +3.38 & \textbf{+3.77} \\
 & 40 rounds & 68.8 & -- & +1.76 & +2.53 & +2.99 & \textbf{+3.26} \\
\midrule
Topology$^b$ & Erd\H{o}s--R\'enyi & 65.4 & -- & +1.36 & +1.64 & +3.53 & \textbf{+7.25} \\
 & random 6-regular & 65.4 & -- & +1.09 & +1.28 & +3.25 & \textbf{+6.91} \\
 & ring & 64.4 & -- & $-$0.12 & +0.45 & +0.44 & \textbf{+1.59} \\
\midrule
Clients$^b$ & $N{=}50$ & 65.6 & -- & +1.94 & +1.57 & +3.87 & \textbf{+6.78} \\
 & $N{=}300$ & 62.8 & -- & +0.70 & +1.40 & +2.76 & \textbf{+5.06} \\
\midrule
Data skew$^b$ & Dirichlet $\alpha{=}0.1$ & 77.0 & +0.28 & +0.88 & +0.89 & +1.95 & \textbf{+3.49} \\
 & Dirichlet $\alpha{=}1.0$ & 54.4 & $-$4.94 & +1.31 & +1.92 & +4.46 & \textbf{+8.57} \\
 & 2 classes per client & 80.5 & \textbf{+5.49} & +1.59 & +0.60 & +0.72 & +2.10 \\
\bottomrule
\end{tabular}
}
\end{table}

\begin{table}[h]
\centering
\caption{\textbf{Unreliable links.} Links fail independently every round, on top of bursty outages ($u{=}0.8$, 3 datasets
$\times$ 5 seeds, pp against DFedAvg, wins/runs). Link failures do not change ages or lags. Bold: best rule in each row.}
\label{tab:links}
\small
\adjustbox{max width=\linewidth}{%
\begin{tabular}{lcccccc}
\toprule
Setting (bursty $u{=}0.8$) & DFedAvg (\%) & AoI + IPW & Limited, $L{=}2$ & SFG & SFG-L & \sfgcut{} \\
\midrule
30\% of links fail & 56.0 & +1.74 \tiny{(15/15)} & +6.11 \tiny{(15/15)} & +2.64 \tiny{(15/15)} & +6.11 \tiny{(15/15)} & \textbf{+6.17} \tiny{(15/15)} \\
60\% of links fail & 55.6 & +1.01 \tiny{(15/15)} & \textbf{+4.64} \tiny{(15/15)} & +1.49 \tiny{(15/15)} & +4.62 \tiny{(15/15)} & +4.61 \tiny{(15/15)} \\
\bottomrule
\end{tabular}
}
\end{table}

\begin{table}[h]
\centering
\caption{\textbf{Discount factor of \sfgl{}} (gain over DFedAvg in pp, 5 seeds). The rule is flat for $\gamma\in[0,0.5]$. Under
bursty outages, lags call for a firmer discount than ages. Bold: best $\gamma$ in each row.}
\label{tab:gammal}
\small
\adjustbox{max width=\linewidth}{%
\begin{tabular}{lccccc}
\toprule
\sfg-L, $\gamma=$ & 0.0 & 0.2 & 0.5 & 0.8 & 0.95 \\
\midrule
CIFAR-10, bursty & +6.55 & \textbf{+6.65} & +6.33 & +3.24 & +0.72 \\
CIFAR-10, FLASH strict & \textbf{+4.69} & +4.66 & +4.63 & +4.64 & +3.28 \\
FEMNIST, bursty & +5.56 & +5.59 & \textbf{+5.84} & +3.44 & +0.81 \\
FEMNIST, FLASH strict & +4.71 & \textbf{+4.93} & +4.71 & +4.36 & +3.08 \\
\bottomrule
\end{tabular}
}
\end{table}

\paragraph{Further baselines.}
\Cref{tab:extra} compares additional variants: local counts without inheritance, PWG on absolute ages, the non-parametric PWG-NP,
and stronger update amplification. Counts that are not inherited drift apart under heterogeneous availability. A hard threshold
on them then cuts frequently available clients off from everyone else. Inheriting counts, as gossip learning prescribes, removes
this failure (\cref{tab:main}).

\begin{table}[h]
\centering
\caption{\textbf{Further baselines} on the main grid. Gain over DFedAvg, 95\% stratified $t$-interval, and wins over paired runs.
$^a$Seven processes (no STUNner), 105 runs. Bold: best gain without extra computation.}
\label{tab:extra}
\small
\adjustbox{max width=\linewidth}{%
\begin{tabular}{lccc}
\toprule
Rule & Gain vs DFedAvg (pp) & 95\% CI & Wins \\
\midrule
Reset own model on return & +1.00 & [0.94, 1.06] & 120/120 \\
Limited merging (counts), $L{=}2$ & +3.70 & [3.57, 3.82] & 120/120 \\
\sfgcut{}, $L{=}2$ & \textbf{+3.81} & [3.66, 3.96] & 119/120 \\
Least-lagged only ($\gamma{=}0$, lags) & +3.70 & [3.51, 3.90] & 112/120 \\
PWG-L, horizon rate $H\rho{=}10$ & +3.56 & [3.41, 3.71] & 112/120 \\
Oracle horizon-GLS, $H\rho{=}10$ & +3.75 & [3.61, 3.89] & 118/120 \\
Model memory + \sfg-L & +2.87 & [2.72, 3.03] & 112/120 \\
PWG (calibrated, relative ages) & +1.48 & [1.43, 1.54] & 120/120 \\
PWG, absolute ages & +1.46 & [1.41, 1.51] & 120/120 \\
PWG-NP (non-parametric, ages) & +1.13 & [1.09, 1.17] & 118/120 \\
PWG-L (calibrated, lags) & +2.44 & [2.33, 2.54] & 117/120 \\
Freshest only (\sfg{}, $\gamma{=}0$, ages) & +2.41 & [2.32, 2.51] & 115/120 \\
Model memory + \sfg{} (ages) & +2.13 & [2.06, 2.20] & 120/120 \\
Limited merging on lags (mean-inherited), $L{=}10$ & +1.97 & [1.86, 2.09] & 118/120 \\
Model-age weighting, local counts & +0.37 & [0.31, 0.43] & 91/120 \\
Limited merging, local counts, $L{=}2$ & $-$8.52 & [$-$8.97, $-$8.07] & 0/120 \\
Limited merging, local counts, $L{=}10$ & $-$6.40 & [$-$7.25, $-$5.55] & 6/120 \\
Update amplification, $c{=}3$ ($\times4$)$^a$ & $-$12.62 & [$-$14.37, $-$10.87] & 1/105 \\
\sfg{} + catch-up & +3.88 & [3.78, 3.97] & 120/120 \\
\bottomrule
\end{tabular}
}
\end{table}

\paragraph{How staleness is inherited, and how it is used.}
\label{app:inherit}
\Cref{tab:inherit} crosses two ways of inheriting staleness through merging with two ways of using it. The inheritance rules are
the maximum update count of gossip learning and the weighted mean of effective lag. The uses are the hard cut of limited merging,
with count-proportional weights among the candidates within a margin $L$ of the least stale, and the geometric discount of
\sfg{}.
\emph{The margin matters most.} Both cuts lose 1.1--1.8\,pp as $L$ grows from 2 to 10.
\emph{At a loose margin}, $L=10$, the geometric discount is better for both inheritance rules. The maximum is then slightly
better than the mean, by $+0.17$\,pp with the discount and $+0.58$\,pp with the cut.
\emph{At a tight margin, the order reverses}, though not significantly after correction. \sfgcut{} gains $+3.97$ and
$+3.81$\,pp at $L=1$ and 2, against $+3.84$ and $+3.70$\,pp for the cut on counts.
Copying the least-lagged candidate ($\gamma=0$ on lags, $+3.70$\,pp) is the $L\to0$ limit of \sfgcut{}.

\begin{table}[h]
\centering
\caption{\textbf{Inheritance $\times$ use} on the main grid (gain over DFedAvg in pp, 120 paired runs each, -- not run). Bold: best inheritance rule in each column.}
\label{tab:inherit}
\small
\adjustbox{max width=\linewidth}{%
\begin{tabular}{l ccccc c}
\toprule
 & \multicolumn{5}{c}{hard cut at margin $L$} & geometric \\
\cmidrule(lr){2-6}
Staleness inherited by & $L{=}1$ & $L{=}2$ & $L{=}5$ & $L{=}10$ & $L{=}50$ & $\gamma{=}0.5$ \\
\midrule
none (ages) & +2.51 & -- & -- & -- & -- & +2.44 \\
maximum (gossip-learning counts) & +3.84 & +3.70 & +3.15 & \textbf{+2.55} & +1.35 & \textbf{+3.66} \\
weighted mean (effective lags) & \textbf{+3.97} & \textbf{+3.81} & -- & +1.97 & -- & +3.49 \\
\bottomrule
\end{tabular}
}
\end{table}

\paragraph{The horizon, tested.}
\label{app:horizon}
\Cref{tab:horizon} compares rules on effective lags that differ only in steepness.
\emph{Steeper is better on average.} Raising the rate of the calibrated weights from the online one-step estimate to the
horizon rates $H\rho\in\{3,10,30\}$ of \cref{prop:horizon} lifts the average gain by up to 1.2\,pp. The exact GLS weights at
$H\rho=10$ gain more still.
\emph{The steepest rules are fragile.} They lose most under heterogeneous rates, where a frequently available client keeps
only itself.
\emph{Cuts are robust.} Cuts at a margin of one or two rounds keep near-ties and retain the average gain. $L=2$ is the most
robust, and $L=1$ the best on average.
The minimisers of the derived objective behave as \cref{prop:horizon}(ii) predicts. With $\eta\ge1$, they keep only exact ties on
integer lags and coincide with the nested limit ($+3.70$\,pp, fragile under heterogeneous rates). At the self-consistent slopes
$\eta\in[0.25,0.5]$ (\cref{tab:hmeasured}), they gain $+3.66$ to $+3.77$\,pp on scalar lags and $+3.87$\,pp with the exact $G$.
There, solving the simplex QP instead of clipping GLS adds 0.1\,pp. The coherence term changes nothing on average ($-0.00$\,pp,
$[-0.06,0.06]$).

All these minimisers stay 0.1--0.3\,pp below \sfgcut{}. For example, \sfgcut{} ($L=1$) beats sparsemax at $\eta=0.5$ by
$+0.21$\,pp ($p_{\text{Holm}}=0.004$). The main reason is that their weights fall within the margin, which makes them steeper under
heterogeneous rates.

\begin{table}[h]
\centering
\caption{\textbf{The horizon, tested.} Rules on effective lags, with gain over DFedAvg in pp, wins over 120 paired runs, the worst
of the 24 cells, and the mean gain under heterogeneous rates at $u{\approx}0.77$. The last eight rows minimise the objective of
\cref{prop:horizon}. They are sparsemax and QP on scalar lags with uniform imports of slope $\eta=\rho\beta$, and the oracle QP on
the exact $G$ with the same scalar import term ($\eta=0$: none), not the exact $g=A_SB$ (Appendix~\ref{app:details}). Bold: best value in each column.}
\label{tab:horizon}
\small
\adjustbox{max width=\linewidth}{%
\begin{tabular}{lcccc}
\toprule
Weights on effective lags & gain (pp) & wins & worst cell & hetero $u{\approx}0.77$ \\
\midrule
PWG-L, online $\hat\rho$ ($H=1$) & +2.44 & 117/120 & $-$0.07 & +1.92 \\
PWG-L, $H\rho=3$ & +3.35 & 113/120 & $-$0.44 & +1.59 \\
PWG-L, $H\rho=10$ & +3.56 & 112/120 & $-$0.80 & +1.01 \\
PWG-L, $H\rho=30$ & +3.59 & 107/120 & $-$0.99 & +0.50 \\
GLS on lags (nested kernel), $H\rho=10$ & +3.74 & 107/120 & $-$0.53 & +0.80 \\
oracle GLS (exact $G$), $H\rho=3$ & +3.71 & \textbf{120}/120 & +0.59 & +1.83 \\
oracle GLS (exact $G$), $H\rho=10$ & +3.75 & 118/120 & +0.49 & +1.59 \\
least-lagged only ($H\to\infty$, $\gamma=0$) & +3.70 & 112/120 & $-$1.32 & +0.19 \\
least-lagged, ties on rounded lags & +3.71 & 113/120 & $-$1.06 & +0.23 \\
\sfgcut{}, $L{=}2$ & +3.81 & 119/120 & \textbf{+1.27} & \textbf{+2.32} \\
\sfgcut{}, $L{=}1$ & \textbf{+3.97} & 119/120 & +0.76 & +1.73 \\
\sfg-L ($\gamma=0.5$) & +3.49 & 118/120 & +0.38 & +1.96 \\
sparsemax on lags, $\eta{=}2$ & +3.70 & 112/120 & $-$1.23 & +0.16 \\
sparsemax on lags, $\eta{=}1$ & +3.71 & 112/120 & $-$1.27 & +0.15 \\
sparsemax on lags, $\eta{=}0.5$ & +3.76 & 112/120 & $-$0.02 & +0.89 \\
sparsemax on lags, $\eta{=}0.25$ & +3.66 & 117/120 & +0.07 & +1.70 \\
QP on lags, $H\rho{=}3$, $\eta{=}0.5$ & +3.77 & 111/120 & $-$0.79 & +0.54 \\
QP on lags, $H\rho{=}10$, $\eta{=}0.5$ & +3.72 & 110/120 & $-$1.08 & +0.30 \\
oracle QP (exact $G$), $H\rho{=}10$, $\eta{=}0$ & +3.87 & 115/120 & +0.43 & +1.48 \\
oracle QP (exact $G$), $H\rho{=}10$, $\eta{=}0.5$ & +3.87 & 115/120 & +0.05 & +1.11 \\
\bottomrule
\end{tabular}
}
\end{table}

\paragraph{Measuring the horizon.}
\label{app:hderive}
Lags, the missed-fraction matrix $A$ and the weights of every rule on lags depend only on the outage trace, the graph and the
rule. They never depend on the models. We therefore replay each rule's weight process on the traces of every process, with 3
seeds, $N=100$, $T=300$ and the 6-nearest-neighbour graph. For every aggregation before round 250, we compute the receiver's
self-return $H$ and import vector $B$ of \cref{prop:horizon} under the realised future weights (\texttt{horizon\_derive.py}).

\Cref{tab:hmeasured} reports four quantities.
These are the self-return $H$, the effective rate $H\hat\rho$ with the online rate of PWG-L on lags, and the import slope
pooled over aggregations,
\begin{equation*}
\beta=\frac{\sum_j(g_j-g_{j_0})}{\sum_j\Delta_j}.
\end{equation*}
The fourth is the self-consistent cut margin. Running the cut with margin $L$ gives a slope $\beta(L)$. The best response to it
cuts at $1/(\hat\rho\beta(L))$. We report where the two agree, interpolated over $L\in\{0.5,1,2,3,5\}$.
These quantities hold given the rule that runs, the training length ($T=300$), and the convention that the aggregate's dispersion
counts once. The next paragraph measures how much dispersion survives local training, which lowers the effective multiplier.
There are three findings.

\emph{(1) One-step weights are too mild.} $H\ge3$ on every process and for every rule. Along \sfgl{} trajectories, $H$ is 8--23
under bursty $u{=}0.8$, heterogeneous $u{\approx}0.77$, FLASH strict and STUNner. The receiver keeps its aggregate through its
next outages. $H$ converges within about 50 rounds. One-step weights ($H=1$) are therefore too mild. PWG-L with its online rate
trails every fixed $H\rho\in\{3,10,30\}$ on six of the eight processes, all but the two heterogeneous ones.

\emph{(2) $H\hat\rho$ does not predict the best rate per process.} It lies where the calibrated rates are within 0.1\,pp of the
best on five processes (heterogeneous $u{=}0.6$, FLASH relaxed, bursty $u{=}0.8$, FLASH strict, STUNner). A constant $H\rho=30$,
however, does so on six. Much steeper weights win under i.i.d.\ and bursty $u{=}0.6$ outages, where $H\hat\rho\approx2$ against a
best rate of 30. The online rate wins under heterogeneous rates at $u{\approx}0.77$. The rank correlation between $H\hat\rho$ and
the best rate is weak (Spearman about 0.2, $p>0.5$).

\emph{(3) The import slope depends on the neighbours' rule and on the run length.} It is positive against freely averaging
neighbours, at $\beta=3.5$--12 along DFedAvg trajectories, which gives a best-response margin below one round. It is smaller as
the network's own rule steepens (0.3--2.3 along \sfgl{}). It is negative on FLASH strict and STUNner, and, along \sfgcut{} with
$L=1$, also under bursty $u{=}0.8$. There, imports concentrate on increments that the least-lagged candidate also missed, so
nesting fails. The slope grows with the remaining run length. Along DFedAvg under i.i.d.\ outages, it is 0.62, 1.95 and 3.62 when
the sum runs over 10, 50 and all remaining rounds (seed 1).

The self-consistent margin is 1.9--3.2 rounds on five processes. Under bursty $u{=}0.8$, it lies between 3 and 5, where
$\beta(L)$ changes sign. On the two long-outage traces, there is none. At $\hat\rho$, the exact minimiser of the objective keeps
1.9--3.9 candidates on average along \sfgl{} trajectories. Its support coincides with that of \sfgcut{} at $L=1$ in 47--92\% of
aggregations, and in 25--80\% along \sfgcut{}'s own trajectories. Without the coherence term, it keeps more candidates (2.2--4.0).

The derivation thus explains why the discount must be steep. The cut structure needs nesting and positive imports. These fail on
the long-outage traces, so the margin is empirical.

\begin{table}[h]
\centering
\caption{\textbf{The horizon, measured} on the traces (3 seeds, \cref{prop:horizon}). $\hat\rho$: online rate of PWG-L on lags.
$H$: mean self-return of the receiver along each rule's trajectory (cut: $L=1$). Best $H\rho$: the best rate of PWG-L among the
online estimate ($H=1$) and $\{3,10,30\}$ in the runs of \cref{tab:horizon}. $\beta$: pooled import slope. Self-consistent margin:
see text. The bracket 3--5 is where $\beta(L)$ changes sign.}
\label{tab:hmeasured}
\scriptsize
\setlength{\tabcolsep}{3pt}
\adjustbox{max width=\linewidth}{%
\begin{tabular}{lccccccccc}
\toprule
 & & \multicolumn{3}{c}{self-return $H$} & & best $H\rho$ & \multicolumn{2}{c}{import slope $\beta$} & cut margin \\
\cmidrule(lr){3-5}\cmidrule(lr){8-9}
Process & $\hat\rho$ & DFedAvg & \sfg-L & cut & $H\hat\rho$ (\sfg-L) & (PWG-L) & DFedAvg & \sfg-L & self-consistent \\
\midrule
i.i.d.\ $u{=}0.6$ & 0.43 & 3.2 & 3.5 & 5.1 & 1.5 & 30 & 3.6 & 2.27 & 2.4 \\
bursty $u{=}0.6$ & 0.54 & 3.4 & 4.5 & 4.6 & 2.4 & 30 & 4.1 & 0.33 & 3.2 \\
bursty $u{=}0.8$ & 1.05 & 8.8 & 14.0 & 16.6 & 14.7 & 30 & 8.7 & 0.49 & 3--5 \\
hetero $u{=}0.6$ & 0.61 & 3.1 & 3.9 & 4.3 & 2.4 & 3 & 3.5 & 0.53 & 2.6 \\
hetero $u{\approx}0.77$ & 1.15 & 5.2 & 7.8 & 9.1 & 8.9 & online & 5.7 & 0.94 & 1.9 \\
FLASH relaxed & 1.00 & 3.6 & 4.8 & 4.9 & 4.8 & 3 & 4.4 & 0.29 & 2.6 \\
FLASH strict & 1.50 & 12.4 & 19.2 & 19.6 & 28.9 & 30 & 12.1 & $-$0.38 & none \\
STUNner & 1.65 & 14.3 & 22.6 & 22.7 & 37.3 & 30 & 12.2 & $-$1.44 & none \\
\bottomrule
\end{tabular}
}
\end{table}

\paragraph{Carry-over through local training.}
\label{app:carry}
We measure carry-over with the fine-tuned-geometry diagnostic on DFedAvg trajectories. We use four processes (i.i.d.\ $u{=}0.6$,
bursty $u{=}0.8$, heterogeneous $u{\approx}0.77$, STUNner) with 3 datasets $\times$ 2 seeds. For each fresh receiver and each
live neighbour, we record the absolute squared distance between their models before and after the receiver's $K$ local steps
from both. We also record the pair's lag difference $\Delta\ell$. We use rounds $\ge50$, with 52--925 pairs per run. Under the
staleness model, the two distances are
\begin{equation*}
\text{before: } 2\sigma^2+\delta^2|\Delta\ell|,\qquad
\text{after: } 2\sigma^2\theta_\varepsilon^2+\delta^2\theta_\delta^2|\Delta\ell| .
\end{equation*}
Per-run least-squares fits therefore give $\theta_\varepsilon^2$ and $\theta_\delta^2$.

\emph{Findings.} Local training keeps most of the dispersion ($\theta_\varepsilon^2$: median 0.80, range 0.62--0.93 over the 24
runs). It keeps nearly all persistent error ($\theta_\delta^2$: median 0.99). This supports the persistence assumption in
\cref{ass:dynamic}, while showing substantial dispersion carry-over.

\emph{Effective rate.} We set $\theta_\varepsilon^2=0.805$ and $\theta_\delta=1$, and evaluate the horizon terms of the
carry-over extension (Appendix~\ref{app:proofs}) along \sfgl{} trajectories. They give $H/H_\varepsilon=1.5$--4 and an effective
rate $\hat\rho\lambda_GH/H_\varepsilon$ of 0.8--9 (\cref{tab:carry}). Here
\begin{equation*}
\lambda_G=\frac{\sum_j\|A_j-A_{j_0}\|^2}{\sum_j\Delta_j}\in[0.9,1.4]
\end{equation*}
converts persistent differences to lag differences. Across the measured range $\theta_\varepsilon^2\in[0.62,0.93]$, $H/H_\varepsilon$
is 2.0--6.6 at the lower end and 1.2--2.3 at the upper end. The effective rate is 1.2--15 and 0.7--5.2, respectively.

\emph{Caveats.} The intercept ratio overstates $\theta_\varepsilon^2$ when persistent differences are not proportional to the lag
difference (non-nested rows). This understates the effective rate. The extension also keeps only the receiver's own path of its
aggregate. It treats the candidates' dispersions as uncorrelated, although carry-over correlates them. It omits the dispersion
the receiver later imports from neighbours that averaged the same candidates. $H/H_\varepsilon$ is therefore a first-order
approximation.

\emph{Link to the fine-tuned geometry.} The ratio $\theta_\delta^2/\theta_\varepsilon^2\approx1.23$ matches the rate ratio
$\rho_M/\rho=1.1$--1.8 of the fine-tuned geometry (Appendix~\ref{app:ftgeom}). The carry-over is that effect measured in absolute
units.

\emph{Iteration.} Iterating the QP rule's parameters along its own trajectory does not converge exactly. The parameter $a$
alternates within about $\pm0.4$ under bursty $u{=}0.8$ and heterogeneous $u{\approx}0.77$. The coherence slope $\eta$ oscillates
more and is small (0--0.21 after damping).

\begin{table}[h]
\centering
\caption{\textbf{The horizon with carry-over} (trace seeds 1--2, $\theta_\varepsilon^2=0.805$, $\theta_\delta=1$, $\hat\rho$ from
simulator seeds 1--2). Self-consistent QP: parameters of the QP on scalar lags iterated along its own trajectory, as the mean of
the last two of four iterations. This damping was chosen after the oscillation was seen, before launch. Dispersion once: the
earlier convention ($H_\varepsilon=1$) along \sfgl{} trajectories. Both are evaluated in \cref{tab:c26b}.}
\label{tab:carry}
\scriptsize
\setlength{\tabcolsep}{3pt}
\adjustbox{max width=\linewidth}{%
\begin{tabular}{lccccccccc}
\toprule
 & & \multicolumn{4}{c}{along \sfg-L} & \multicolumn{2}{c}{self-consistent QP} & \multicolumn{2}{c}{dispersion once} \\
\cmidrule(lr){3-6}\cmidrule(lr){7-8}\cmidrule(lr){9-10}
Process & $\hat\rho$ & $H$ & $H_\varepsilon$ & $\lambda_G$ & $\hat\rho\lambda_GH/H_\varepsilon$ & $a$ & $\eta$ & $a$ & $\eta$ \\
\midrule
i.i.d.\ $u{=}0.6$ & 0.43 & 3.5 & 2.4 & 1.32 & 0.84 & 0.86 & 0.21 & 2.0 & 0.98 \\
bursty $u{=}0.6$ & 0.54 & 4.3 & 2.4 & 0.95 & 0.89 & 0.86 & 0.06 & 2.2 & 0.17 \\
bursty $u{=}0.8$ & 0.96 & 14.0 & 6.5 & 1.41 & 2.90 & 3.10 & 0.07 & 18.9 & 0.63 \\
hetero $u{=}0.6$ & 0.65 & 4.0 & 2.3 & 0.90 & 1.02 & 1.02 & 0.08 & 2.3 & 0.36 \\
hetero $u{\approx}0.77$ & 1.37 & 7.7 & 3.8 & 1.01 & 2.81 & 3.10 & 0.11 & 10.6 & 1.32 \\
FLASH relaxed & 1.00 & 4.6 & 2.2 & 0.95 & 1.96 & 1.97 & 0.06 & 4.4 & 0.30 \\
FLASH strict & 1.48 & 19.7 & 6.0 & 1.14 & 5.55 & 5.47 & 0.00 & 33.3 & 0.00 \\
STUNner & 1.85 & 22.5 & 5.6 & 1.21 & 8.95 & 8.92 & 0.00 & 50.4 & 0.00 \\
\bottomrule
\end{tabular}
}
\end{table}

\paragraph{A pre-registered predictive test.}
\label{app:prereg}
Before running the predictive test, we recorded the rule, its parameters, the evaluation seeds and five hypotheses in the
accompanying \texttt{c26b\_prereg.md}.

\emph{Design.} The rule is the QP on scalar lags with the self-consistent per-process parameters of \cref{tab:carry}. These were
measured on trace seeds 1--2, with $\hat\rho$ from simulator seeds 1--2. The rule is evaluated on the held-out seeds 3--5, with 72
runs per rule. The primary metric is the mean paired difference over the seven processes other than i.i.d., over 9
dataset--seed units. The test was to succeed if two conditions held. The measured rule had to beat the constant QP ($a=10$,
$\eta=0.5$, H1). It also had to come within 0.1\,pp of the best of the cuts and \sfgl{} (H2). The registered criteria are point
estimates.

\emph{Outcome. The test fails on H2.}
\emph{H1.} Against the constant QP, the difference is $+0.05$\,pp. H1 is met as registered, but the interval $[-0.05,0.16]$
includes zero (6/9 units).
\emph{H2.} Against the best heuristic, \sfgcut{} with $L=2$, the difference is $-0.27$\,pp ($[-0.40,-0.14]$, 1/9). H2 is
rejected.
\emph{H3.} The measured rule is within 0.1\,pp of the best comparator on 2 of the 7 processes, where 5 were required.
\emph{H4.} Parameters under the count-once convention do 0.15\,pp worse ($[0.01,0.30]$, 8/9, uncorrected $p=0.03$). H4 is
met. These parameters are also steeper, however, so this equally says that milder parameters do better outside i.i.d.\ outages.
\emph{H5.} The coherence term changes nothing there ($+0.01$\,pp, $[-0.08,0.10]$). H5 is met as registered, but the effect is
negligible. Under i.i.d.\ outages, in contrast, the term adds $+1.93$\,pp ($[1.46,2.40]$, 9/9). This is the steepness that pays
there.

\emph{Exploratory comparisons.} The measured rule trails \sfgl{} by 0.11\,pp ($[-0.24,0.02]$, 3/9). It beats one-step weights
(online PWG-L) by 0.76\,pp ($[0.58,0.94]$, 9/9). Horizon-scaled weights thus beat one-step weights. Measured per-process scales,
however, do no better than constants and do not select the rule.

\emph{Disclosure.} The comparators' results on seeds 3--5 were known at registration. No multiplicity correction was registered or
applied. The damping rule was added after the oscillation was seen, before launch. The accompanying registration is
self-attested.

\begin{table}[h]
\centering
\caption{\textbf{Pre-registered test} on held-out seeds 3--5. Gain over \sfgl{} in pp by outage process, 3 datasets $\times$ 3
seeds. Measured: per-process parameters of \cref{tab:carry}. All QPs are on scalar effective lags. Bold: best rule in each column.}
\label{tab:c26b}
\scriptsize
\setlength{\tabcolsep}{3pt}
\adjustbox{max width=\linewidth}{%
\begin{tabular}{lcccccccccc}
\toprule
Gain over \sfg-L (pp), seeds 3--5 & i.i.d. & b.\,0.6 & b.\,0.8 & h.\,0.6 & h.\,0.77 & F.\,rel. & F.\,str. & STUN. & all & w/o i.i.d. \\
\midrule
QP, measured $(a,\eta)$ & +2.18 & +0.09 & +0.13 & $-$0.13 & $-$0.68 & +0.19 & $-$0.17 & $-$0.24 & +0.17 & $-$0.11 \\
QP, measured $a$, $\eta=0$ & +0.25 & +0.04 & $-$0.14 & $-$0.07 & $-$0.42 & +0.13 & $-$0.12 & $-$0.29 & $-$0.08 & $-$0.12 \\
QP, dispersion counted once & +2.98 & +0.31 & +0.15 & $-$0.32 & $-$1.73 & +0.16 & $-$0.18 & $-$0.24 & +0.14 & $-$0.26 \\
QP, constant ($a{=}10$, $\eta{=}0.5$) & \textbf{+3.11} & \textbf{+0.41} & \textbf{+0.22} & $-$0.21 & $-$1.56 & +0.11 & $-$0.05 & $-$0.08 & +0.24 & $-$0.17 \\
cut, $L{=}1$ & +3.02 & +0.40 & +0.19 & +0.27 & $-$0.10 & \textbf{+0.24} & \textbf{$-$0.02} & \textbf{$-$0.02} & \textbf{+0.50} & +0.14 \\
cut, $L{=}2$ & +1.34 & +0.17 & $-$0.07 & \textbf{+0.67} & \textbf{+0.43} & +0.06 & $-$0.05 & $-$0.11 & +0.31 & \textbf{+0.16} \\
\bottomrule
\end{tabular}
}
\end{table}

\paragraph{Where steepness pays.}
\label{app:perprocess}
\Cref{tab:perprocess} gives each steep rule's gain over \sfgl{} by outage process.

\emph{Outside i.i.d.\ outages}, all rules with the three properties, including the oracle, are on average within about 0.3\,pp of
\sfgl{}. Per process, they are within 0.8\,pp, except the steepest rules under heterogeneous rates. The placebo cut loses
0.5--7\,pp.

\emph{Under i.i.d.\ outages}, most aggregations involve a client returning after one or two rounds. There, the steeper rules gain
1--3\,pp more, and so does the oracle of \cref{prop:horizon}. However, four rules exceed outage-free DFedAvg there: \sfgcut{} with
$L=1$, copying the least-lagged candidate, the oracle QP and the oracle GLS at $H\rho=10$. They exceed it by 1.23, 1.39, 1.51 and
0.78\,pp. No rule can do this under \cref{ass:dynamic}, where any outage adds persistent error. Part of this gain therefore lies
outside our staleness model.

The excess is concentrated on CIFAR-100, where \sfgcut{} is $+2.95$\,pp above outage-free DFedAvg, against $+0.27$ on CIFAR-10 and
$+0.48$ on FEMNIST. On CIFAR-100, sparser mixing helps even without outages. Copying a random candidate gains $+2.97$\,pp over
outage-free DFedAvg on CIFAR-100 and loses 0.2\,pp on the other two datasets. \Cref{tab:pilot} gives the seed-0 pilot for the
margin of \sfgcut{}.

\begin{table}[h]
\centering
\caption{\textbf{Gain over \sfgl{} by outage process} (pp, 3 datasets $\times$ 5 seeds). $\pm$: s.e.\ over the 15 dataset--seed
units. b.: bursty. h.: heterogeneous. F.: FLASH. Bold: best rule in each column.}
\label{tab:perprocess}
\scriptsize
\setlength{\tabcolsep}{2pt}
\adjustbox{max width=\linewidth}{%
\begin{tabular}{lcccccccccc}
\toprule
Gain over \sfg-L (pp) & i.i.d. & b.\,0.6 & b.\,0.8 & h.\,0.6 & h.\,0.77 & F.\,rel. & F.\,str. & STUN. & all & w/o i.i.d. \\
\midrule
\sfgcut{}, $L{=}1$ & +2.88\,{\tiny$\pm$0.20} & \textbf{+0.43}\,{\tiny$\pm$0.05} & +0.21\,{\tiny$\pm$0.08} & +0.28\,{\tiny$\pm$0.14} & $-$0.23\,{\tiny$\pm$0.11} & \textbf{+0.31}\,{\tiny$\pm$0.09} & $-$0.01\,{\tiny$\pm$0.06} & +0.01\,{\tiny$\pm$0.08} & \textbf{+0.48}\,{\tiny$\pm$0.04} & +0.14\,{\tiny$\pm$0.03} \\
\sfgcut{}, $L{=}2$ & +1.36\,{\tiny$\pm$0.07} & +0.16\,{\tiny$\pm$0.05} & $-$0.08\,{\tiny$\pm$0.07} & +0.73\,{\tiny$\pm$0.13} & \textbf{+0.35}\,{\tiny$\pm$0.12} & +0.15\,{\tiny$\pm$0.08} & $-$0.03\,{\tiny$\pm$0.06} & $-$0.05\,{\tiny$\pm$0.06} & +0.33\,{\tiny$\pm$0.03} & \textbf{+0.18}\,{\tiny$\pm$0.04} \\
placebo cut (same size, random) & $-$2.03\,{\tiny$\pm$0.09} & $-$2.89\,{\tiny$\pm$0.10} & $-$6.91\,{\tiny$\pm$0.17} & $-$0.45\,{\tiny$\pm$0.25} & $-$2.19\,{\tiny$\pm$0.29} & $-$1.95\,{\tiny$\pm$0.13} & $-$5.74\,{\tiny$\pm$0.27} & $-$5.69\,{\tiny$\pm$0.30} & $-$3.48\,{\tiny$\pm$0.09} & $-$3.69\,{\tiny$\pm$0.11} \\
least-lagged only ($\gamma=0$) & +3.03\,{\tiny$\pm$0.22} & +0.38\,{\tiny$\pm$0.04} & +0.11\,{\tiny$\pm$0.11} & $-$0.12\,{\tiny$\pm$0.16} & $-$1.77\,{\tiny$\pm$0.26} & +0.16\,{\tiny$\pm$0.06} & $-$0.03\,{\tiny$\pm$0.07} & $-$0.02\,{\tiny$\pm$0.08} & +0.22\,{\tiny$\pm$0.06} & $-$0.18\,{\tiny$\pm$0.06} \\
limited merging, counts, $L{=}2$ & +0.44\,{\tiny$\pm$0.08} & +0.03\,{\tiny$\pm$0.05} & +0.02\,{\tiny$\pm$0.07} & \textbf{+0.75}\,{\tiny$\pm$0.18} & +0.29\,{\tiny$\pm$0.10} & +0.09\,{\tiny$\pm$0.07} & \textbf{+0.07}\,{\tiny$\pm$0.05} & $-$0.01\,{\tiny$\pm$0.05} & +0.21\,{\tiny$\pm$0.03} & \textbf{+0.18}\,{\tiny$\pm$0.03} \\
limited merging, counts, $L{=}1$ & +2.06\,{\tiny$\pm$0.14} & +0.23\,{\tiny$\pm$0.05} & +0.13\,{\tiny$\pm$0.08} & +0.49\,{\tiny$\pm$0.18} & $-$0.11\,{\tiny$\pm$0.13} & +0.08\,{\tiny$\pm$0.05} & +0.03\,{\tiny$\pm$0.07} & $-$0.04\,{\tiny$\pm$0.07} & +0.36\,{\tiny$\pm$0.04} & +0.12\,{\tiny$\pm$0.04} \\
\sfg{} on counts & +0.98\,{\tiny$\pm$0.06} & $-$0.03\,{\tiny$\pm$0.05} & +0.03\,{\tiny$\pm$0.08} & +0.35\,{\tiny$\pm$0.18} & $-$0.02\,{\tiny$\pm$0.10} & $-$0.01\,{\tiny$\pm$0.08} & $-$0.00\,{\tiny$\pm$0.05} & \textbf{+0.08}\,{\tiny$\pm$0.06} & +0.17\,{\tiny$\pm$0.02} & +0.06\,{\tiny$\pm$0.03} \\
PWG-L, $H\rho=10$ & +2.33\,{\tiny$\pm$0.11} & +0.21\,{\tiny$\pm$0.10} & +0.06\,{\tiny$\pm$0.09} & $-$0.66\,{\tiny$\pm$0.12} & $-$0.95\,{\tiny$\pm$0.10} & $-$0.04\,{\tiny$\pm$0.07} & $-$0.16\,{\tiny$\pm$0.06} & $-$0.22\,{\tiny$\pm$0.08} & +0.07\,{\tiny$\pm$0.04} & $-$0.25\,{\tiny$\pm$0.04} \\
oracle GLS, $H\rho=3$ & +2.01\,{\tiny$\pm$0.10} & +0.19\,{\tiny$\pm$0.03} & $-$0.44\,{\tiny$\pm$0.12} & +0.09\,{\tiny$\pm$0.09} & $-$0.14\,{\tiny$\pm$0.13} & +0.23\,{\tiny$\pm$0.08} & $-$0.10\,{\tiny$\pm$0.09} & $-$0.05\,{\tiny$\pm$0.10} & +0.22\,{\tiny$\pm$0.04} & $-$0.03\,{\tiny$\pm$0.04} \\
oracle GLS, $H\rho=10$ & +2.43\,{\tiny$\pm$0.10} & +0.38\,{\tiny$\pm$0.04} & $-$0.37\,{\tiny$\pm$0.11} & $-$0.04\,{\tiny$\pm$0.11} & $-$0.38\,{\tiny$\pm$0.16} & +0.27\,{\tiny$\pm$0.08} & $-$0.15\,{\tiny$\pm$0.08} & $-$0.03\,{\tiny$\pm$0.07} & +0.26\,{\tiny$\pm$0.04} & $-$0.05\,{\tiny$\pm$0.04} \\
oracle QP, $H\rho=10$ & \textbf{+3.16}\,{\tiny$\pm$0.12} & +0.34\,{\tiny$\pm$0.04} & +0.04\,{\tiny$\pm$0.10} & $-$0.08\,{\tiny$\pm$0.17} & $-$0.48\,{\tiny$\pm$0.14} & +0.23\,{\tiny$\pm$0.05} & $-$0.09\,{\tiny$\pm$0.07} & $-$0.04\,{\tiny$\pm$0.08} & +0.38\,{\tiny$\pm$0.04} & $-$0.01\,{\tiny$\pm$0.04} \\
sparsemax on lags, $\eta=0.5$ & +2.85\,{\tiny$\pm$0.15} & +0.35\,{\tiny$\pm$0.06} & \textbf{+0.27}\,{\tiny$\pm$0.09} & $-$0.34\,{\tiny$\pm$0.08} & $-$1.07\,{\tiny$\pm$0.16} & +0.13\,{\tiny$\pm$0.08} & +0.02\,{\tiny$\pm$0.06} & $-$0.01\,{\tiny$\pm$0.07} & +0.28\,{\tiny$\pm$0.04} & $-$0.09\,{\tiny$\pm$0.04} \\
\bottomrule
\end{tabular}
}
\end{table}

\begin{table}[h]
\centering
\caption{\textbf{Pilot for the margin of \sfgcut{}} (seed 0, 3 datasets, one run per cell, gain over DFedAvg in pp). The pilot was
run after the test-seed results for $L\in\{1,2\}$ were known. Bold: best rule in each column.}
\label{tab:pilot}
\scriptsize
\setlength{\tabcolsep}{3pt}
\adjustbox{max width=\linewidth}{%
\begin{tabular}{lcccccccccc}
\toprule
Gain over DFedAvg (pp) & i.i.d. & b.\,0.6 & b.\,0.8 & h.\,0.6 & h.\,0.77 & F.\,rel. & F.\,str. & STUN. & all & w/o i.i.d. \\
\midrule
\sfgcut{}, $L{=}0.5$ & \textbf{+5.54} & +3.98 & +7.06 & +0.18 & +0.78 & +3.27 & +4.03 & \textbf{+5.05} & +3.74 & +3.48 \\
\sfgcut{}, $L{=}1$ & +5.07 & +4.00 & \textbf{+7.09} & +1.26 & +1.96 & \textbf{+3.32} & \textbf{+4.11} & +5.02 & \textbf{+3.98} & +3.82 \\
\sfgcut{}, $L{=}2$ & +4.04 & \textbf{+4.40} & +6.61 & +1.92 & +2.13 & +3.06 & +4.09 & +4.86 & +3.89 & \textbf{+3.87} \\
\sfgcut{}, $L{=}3$ & +2.10 & +4.03 & +6.04 & \textbf{+2.09} & \textbf{+2.30} & +2.75 & +4.08 & +4.67 & +3.51 & +3.71 \\
\sfg-L ($\gamma=0.5$) & +2.53 & +3.88 & +6.77 & +1.27 & +1.94 & +3.04 & +4.08 & +4.76 & +3.53 & +3.68 \\
\bottomrule
\end{tabular}
}
\end{table}

\paragraph{Controls for the i.i.d.\ effect.}
\label{app:iid}
\Cref{tab:iidrate,tab:iidlr} test whether the steep rules' excess over outage-free DFedAvg under i.i.d.\ outages is a staleness
effect.

\emph{Across outage rates} (\cref{tab:iidrate}). \sfgcut{} with $L=1$ is above the outage-free reference at $u=0.3$ ($+0.48$\,pp,
11/15 units) and at $u=0.6$ ($+1.23$\,pp). It is far below it at $u=0.8$ ($-3.43$\,pp), where every rule loses 3.4--8.5\,pp. Its
lead over \sfgl{} is largest at $u=0.6$, at $+2.88$\,pp against $+0.83$ at $u=0.3$ and $+0.87$ at $u=0.8$.

\emph{Not a learning-rate artefact} (\cref{tab:iidlr}). We put every rule and the outage-free reference at their best learning rate
per unit among $\{0.02,0.05,0.1\}$. \sfgcut{} still exceeds outage-free DFedAvg by $+1.20$\,pp (12/15 units), with $+0.27$,
$+2.80$ and $+0.53$\,pp on CIFAR-10, CIFAR-100 and FEMNIST. \sfgl{} stays 1.50\,pp below it.

\emph{A mixing effect.} Two further controls identify a mixing effect (\cref{tab:iidmix}).
\emph{Sparse mixing alone gains as much.} Under moderate i.i.d.\ outages, the steep rules average fewer models per round.
The effective number $1/\sum_jw_j^2$ is 1.8 for \sfgcut{}, against 3.5 for DFedAvg and about 8 candidates without outages. Sparse
mixing alone, without any outages and without regard to staleness, gains as much. Averaging two random candidates beats
outage-free DFedAvg by $+1.36$\,pp (14/15 units), with $+1.26$, $+2.51$ and $+0.33$\,pp on the three datasets. The excess of
\sfgcut{} is $+1.23$\,pp.
\emph{A better-personalised reference removes the excess.} With Ditto personal models, which make the reference itself
better personalised, no rule exceeds outage-free Ditto under i.i.d.\ outages. \sfgcut{} is $1.31$\,pp below it (2/15 units), while
it still leads \sfgl{} by 2.3\,pp and DFedAvg by 5.0\,pp. Ditto halves the benefit of sparse mixing without removing it. Two random
candidates with Ditto still beat outage-free Ditto by $+0.64$\,pp (15/15).
These controls support a personalisation effect of sparser mixing as the explanation for the excess over outage-free DFedAvg,
outside our staleness model. The steep rules' lead over other rules under i.i.d.\ outages persists with Ditto, at 2.30\,pp over
\sfgl{} against 2.88\,pp without. It is explained neither by the model nor by the mixing control.

\emph{Gap fractions depend on the reference.} Staleness-free sparse mixing beats outage-free DFedAvg. Relative to outage-free
random-2 mixing, \sfgcut{} closes 46--53\% of the gap, against 56--60\% relative to DFedAvg. For comparison, \sfgl{} closes
38--49\%, limited merging at $L=1$ 43--52\%, \sfgcut{} + IPW 52--56\%, and AoI + IPW 20--24\%.

\begin{table}[h]
\centering
\caption{\textbf{i.i.d.\ outage rate.} Accuracy minus outage-free DFedAvg (pp, 3 datasets $\times$ 5 seeds). Units above the
reference are in parentheses. $u=0.6$ is the main grid. Bold: best rule in each column.}
\label{tab:iidrate}
\small
\adjustbox{max width=\linewidth}{%
\begin{tabular}{lccc}
\toprule
Accuracy minus outage-free DFedAvg (pp) & i.i.d.\ $u{=}0.3$ & $u{=}0.6$ & $u{=}0.8$ \\
\midrule
DFedAvg & $-$1.28 \tiny{(1/15)} & $-$4.22 \tiny{(0/15)} & $-$8.21 \tiny{(0/15)} \\
\sfg-L & $-$0.35 \tiny{(4/15)} & $-$1.65 \tiny{(0/15)} & $-$4.31 \tiny{(0/15)} \\
\sfgcut{}, $L{=}2$ & $-$0.84 \tiny{(2/15)} & $-$0.29 \tiny{(5/15)} & $-$4.03 \tiny{(0/15)} \\
\sfgcut{}, $L{=}1$ & \textbf{+0.48} \tiny{(11/15)} & \textbf{+1.23} \tiny{(11/15)} & \textbf{$-$3.43} \tiny{(0/15)} \\
oracle GLS, $H\rho{=}3$ & +0.45 \tiny{(11/15)} & +0.37 \tiny{(8/15)} & $-$4.30 \tiny{(0/15)} \\
placebo cut & $-$1.00 \tiny{(0/15)} & $-$3.68 \tiny{(0/15)} & $-$8.46 \tiny{(0/15)} \\
\bottomrule
\end{tabular}
}
\end{table}

\begin{table}[h]
\centering
\caption{\textbf{Learning rate under i.i.d.\ outages} ($u=0.6$). Left: accuracy minus outage-free DFedAvg at the same learning
rate. Right: each rule at its best learning rate per unit, against the outage-free reference at its own best (pp, 3 datasets
$\times$ 5 seeds). Units above the reference are in parentheses. Outage-free DFedAvg diverged on one CIFAR-10 seed at learning
rate 0.1. Bold: best rule in each column.}
\label{tab:iidlr}
\small
\setlength{\tabcolsep}{3pt}
\adjustbox{max width=\linewidth}{%
\begin{tabular}{lcccc}
\toprule
i.i.d.\ $u{=}0.6$: accuracy minus outage-free DFedAvg (pp) & lr 0.02 & lr 0.05 & lr 0.1 & best lr vs best lr \\
\midrule
DFedAvg & $-$5.36 \tiny{(0/15)} & $-$4.22 \tiny{(0/15)} & +0.57 \tiny{(1/15)} & $-$3.99 \tiny{(0/15)} \\
\sfg-L & $-$2.62 \tiny{(0/15)} & $-$1.65 \tiny{(0/15)} & +3.05 \tiny{(2/15)} & $-$1.50 \tiny{(0/15)} \\
\sfgcut{}, $L{=}2$ & $-$1.68 \tiny{(1/15)} & $-$0.29 \tiny{(5/15)} & +4.33 \tiny{(10/15)} & $-$0.04 \tiny{(7/15)} \\
\sfgcut{}, $L{=}1$ & \textbf{$-$0.00} \tiny{(8/15)} & \textbf{+1.23} \tiny{(11/15)} & \textbf{+5.36} \tiny{(11/15)} & \textbf{+1.20} \tiny{(12/15)} \\
\midrule
outage-free DFedAvg, accuracy (\%) & 63.5 & 64.5 & 60.0 & \\
\bottomrule
\end{tabular}
}
\end{table}

\begin{table}[h]
\centering
\caption{\textbf{Mixing controls for the i.i.d.\ effect} (3 datasets $\times$ 5 seeds, pp). Units above the reference are in
parentheses. Top: staleness-free random mixing without outages, the last of these rows with Ditto. Bottom: i.i.d.\ $u{=}0.6$ with
and without Ditto personal models ($\lambda=0.2$), each against the outage-free reference with the same backbone.}
\label{tab:iidmix}
\small
\setlength{\tabcolsep}{3pt}
\adjustbox{max width=\linewidth}{%
\begin{tabular}{llcccc}
\toprule
Rule & reference & CIFAR-10 & CIFAR-100 & FEMNIST & mean (units $>0$) \\
\midrule
copy one random candidate (no outages) & outage-free DFedAvg & $-$0.17 & +2.97 & $-$0.23 & +0.86 (10/15) \\
average 2 random candidates (no outages) & outage-free DFedAvg & +1.26 & +2.51 & +0.33 & +1.36 (14/15) \\
average 3 random candidates (no outages) & outage-free DFedAvg & +0.79 & +1.33 & +0.12 & +0.75 (13/15) \\
average 2 random candidates + Ditto (no outages) & outage-free DFedAvg + Ditto & +0.44 & +1.02 & +0.47 & +0.64 (15/15) \\
DFedAvg + Ditto, i.i.d.\ $u{=}0.6$ & outage-free DFedAvg + Ditto & $-$6.41 & $-$9.77 & $-$2.83 & $-$6.34 (0/15) \\
DFedAvg, i.i.d.\ $u{=}0.6$ & outage-free DFedAvg & $-$4.66 & $-$5.08 & $-$2.93 & $-$4.22 (0/15) \\
\sfg-L + Ditto, i.i.d.\ $u{=}0.6$ & outage-free DFedAvg + Ditto & $-$3.73 & $-$5.60 & $-$1.48 & $-$3.61 (0/15) \\
\sfg-L, i.i.d.\ $u{=}0.6$ & outage-free DFedAvg & $-$2.16 & $-$1.66 & $-$1.12 & $-$1.65 (0/15) \\
cut, $L{=}1$ + Ditto, i.i.d.\ $u{=}0.6$ & outage-free DFedAvg + Ditto & $-$1.66 & $-$2.09 & $-$0.17 & $-$1.31 (2/15) \\
cut, $L{=}1$, i.i.d.\ $u{=}0.6$ & outage-free DFedAvg & +0.27 & +2.95 & +0.48 & +1.23 (11/15) \\
cut, $L{=}2$ + Ditto, i.i.d.\ $u{=}0.6$ & outage-free DFedAvg + Ditto & $-$3.40 & $-$4.36 & $-$0.88 & $-$2.88 (1/15) \\
cut, $L{=}2$, i.i.d.\ $u{=}0.6$ & outage-free DFedAvg & $-$1.36 & +0.93 & $-$0.43 & $-$0.29 (5/15) \\
\bottomrule
\end{tabular}
}
\end{table}

\paragraph{The realised objective.}
\label{app:jh}
To evaluate the realised objective, 20 rules log two terms for every aggregation of an available receiver with a live
neighbour: the dispersion $\sum_jw_j^2$, and the exact persistent error $w^\top Gw$ from the tracked matrix $A$. We use seeds 1--2,
3 datasets and all eight processes. \Cref{tab:jh} correlates the per-process mean of each term, relative to DFedAvg, with the
accuracy gain.

\emph{The persistent error separates the staleness-aware rules from the rest.} Over all 20 rules, its Spearman correlation with
the gain is 0.75--0.89 on six processes, and 0.31--0.40 under heterogeneous rates. Dispersion alone does not separate them.

\emph{It does not order the top group.} Among the nine staleness-aware rules without IPW, the correlation is positive only under
i.i.d.\ and short bursty outages (0.82 and 0.60). It is near zero on two processes and negative on four ($-0.92$ to $-0.48$). The
rule with the least persistent error never has the largest gain.

\emph{Why the objective cannot identify $H$.} The persistent differences between rules (median $44\,\delta^2$) dwarf the
dispersion differences (median $0.09\,\sigma^2$). With $\hat\rho\approx0.4$--1.7, the ranking is the same for every
$H\in[1,100]$. The realised objective therefore cannot identify $H$, which is why we measure it on the traces
(\cref{tab:hmeasured}). The largest gains come from rules with IPW, whose link correction lies outside the model.

\begin{table}[h]
\centering
\caption{\textbf{The realised objective separates rules but does not order the top group.} Spearman correlation across rules
between the accuracy gain over DFedAvg and minus the realised persistent error or dispersion, both relative to DFedAvg. All 20
rules, and the nine staleness-aware rules without IPW.}
\label{tab:jh}
\small
\adjustbox{max width=\linewidth}{%
\begin{tabular}{lccccc}
\toprule
 & \multicolumn{2}{c}{all 20 rules} & top 9 & & \\
\cmidrule(lr){2-3}\cmidrule(lr){4-4}
Process & persistent & dispersion & persistent & least persistent error & largest gain \\
\midrule
i.i.d.\ $u{=}0.6$ & +0.89 & $-$0.38 & +0.82 & least-lagged & cut + IPW \\
bursty $u{=}0.6$ & +0.86 & $-$0.22 & +0.60 & least-lagged & \sfg-L + IPW \\
bursty $u{=}0.8$ & +0.79 & $-$0.29 & +0.03 & \sfg{} counts & cut + IPW \\
hetero $u{=}0.6$ & +0.40 & +0.05 & $-$0.70 & least-lagged & \sfg-L + IPW \\
hetero $u{\approx}0.77$ & +0.31 & +0.18 & $-$0.92 & least-lagged & \sfg-L + IPW \\
FLASH relaxed & +0.77 & $-$0.10 & +0.05 & least-lagged & cut + IPW \\
FLASH strict & +0.75 & $-$0.26 & $-$0.57 & oracle & cut + IPW \\
STUNner & +0.79 & $-$0.20 & $-$0.48 & oracle & \sfg-L + IPW \\
\bottomrule
\end{tabular}
}
\end{table}

\paragraph{Self-inclusion on every staleness measure.}
\label{app:selfincl}
\Cref{tab:selfincl} applies the same geometric discount ($\gamma=0.5$) in two ways. Once to neighbours only, relative to the
least-stale live neighbour, with the receiver's weight fixed at 1. Once to every candidate, the receiver included.
Including the receiver adds 1.5--2.5\,pp on ages, lags and counts alike.
Pinning the receiver's weight to 1 under \sfgl{}'s own reference, which changes nothing else, costs 1.76\,pp.
Inheriting staleness helps self-inclusive rules far more ($+1.0$ to $+1.2$\,pp over ages) than neighbour-only ones ($+0.2$
to $+0.5$\,pp). This is expected when the receiver's own inherited lag is what goes unpriced.

\begin{table}[h]
\centering
\caption{\textbf{Self-inclusion $\times$ staleness measure} on the main grid (gain over DFedAvg in pp). The last column gives the
paired difference with stratified $t$-interval and wins over 120 runs. Restricted: the discount applied to neighbours only (first
three rows), the receiver's weight pinned at 1, or the receiver never dropped by the cut. The self-inclusive column is the
unrestricted rule. Bold: the better variant in each row.}
\label{tab:selfincl}
\small
\adjustbox{max width=\linewidth}{%
\begin{tabular}{lccc}
\toprule
Staleness measure & restricted & self-inclusive (\sfg{}) & difference [95\% CI], wins \\
\midrule
ages & +0.94 & \textbf{+2.44} & +1.50 [1.42, 1.59], 118/120 \\
effective lags (mean-inherited) & +1.48 & \textbf{+3.49} & +2.01 [1.91, 2.11], 117/120 \\
update counts (max-inherited) & +1.14 & \textbf{+3.66} & +2.52 [2.38, 2.66], 119/120 \\
effective lags, receiver pinned at 1 (same reference) & +1.73 & \textbf{+3.49} & +1.76 [1.65, 1.86], 118/120 \\
\sfgcut{} ($L{=}1$), receiver never dropped & +1.90 & \textbf{+3.97} & +2.07 [1.93, 2.20], 119/120 \\
\bottomrule
\end{tabular}
}
\end{table}

\paragraph{Learning rate.}
\label{app:lr}
\Cref{tab:lr} re-tunes the learning rate of DFedAvg, AoI + IPW, \sfg{} and \sfgl{} separately, over $\{0.02,0.05,0.1,0.2\}$ on four
processes. It compares each method at its own best learning rate.

\begin{table}[h]
\centering
\caption{\textbf{Per-method learning-rate tuning} (bursty $u{=}0.8$, heterogeneous $u{=}0.6$, FLASH strict, STUNner, seeds 1--3).
Best lr: the learning rate with the highest mean accuracy over the four processes, chosen per method and dataset. Gain: paired
difference to DFedAvg at its own best learning rate (pp). Bold: best gain in each row.}
\label{tab:lr}
\small
\setlength{\tabcolsep}{4pt}
\adjustbox{max width=\linewidth}{%
\begin{tabular}{l cc cc cc cc}
\toprule
 & \multicolumn{2}{c}{DFedAvg} & \multicolumn{2}{c}{AoI + IPW} & \multicolumn{2}{c}{\sfg{}} & \multicolumn{2}{c}{\sfg-L} \\
\cmidrule(lr){2-3}\cmidrule(lr){4-5}\cmidrule(lr){6-7}\cmidrule(lr){8-9}
Dataset & best lr & acc. & best lr & gain & best lr & gain & best lr & gain \\
\midrule
CIFAR-10 & 0.05 & 65.2 & 0.05 & +1.64 & 0.05 & +2.35 & 0.05 & \textbf{+4.31} \\
CIFAR-100 & 0.1 & 24.2 & 0.1 & +1.89 & 0.1 & +2.47 & 0.05 & \textbf{+4.24} \\
FEMNIST & 0.1 & 79.8 & 0.1 & +1.50 & 0.1 & +1.98 & 0.1 & \textbf{+3.69} \\
\midrule
Mean gain (wins) & & & & +1.68 (36/36) & & +2.27 (36/36) & & \textbf{+4.08} (36/36) \\
\bottomrule
\end{tabular}
}
\end{table}

\paragraph{Accuracy over time and offline clients.}
\label{app:time}
The main results evaluate every client, available or not, after the last round. On the real traces, this snapshot falls at about
02:00 local time. Depending on the process, 57--81\% of the clients are offline at that moment (\cref{tab:time}), so they are
evaluated with frozen models. \Cref{tab:time} therefore also averages accuracy over the last 100 rounds and over the whole run.

The ranking is stable. The steep rules on inherited staleness (limited merging with $L=2$, \sfg{} on counts, \sfgl{}) gain
3.5--3.8\,pp under every summary, ahead of every rule on ages. With IPW, they gain 4.1--4.3\,pp. DFL-AA and limited merging with
$L=10$ lose a little over time. \sfg{}, which trails the latter at the final round, therefore overtakes it.

\begin{table}[h]
\centering
\caption{\textbf{Accuracy over time.} Gain over DFedAvg (pp) on the main grid at the final round, averaged over the evaluations in
the last 100 rounds (rounds 210--300), and averaged over all evaluations (every 10 rounds). At the final evaluation, 57--60\% of
clients are offline under the $u=0.6$ processes and FLASH relaxed, and 75--81\% under the others. Bold: best value in each column, separately for rules without and with catch-up training (extra computation).}
\label{tab:time}
\small
\adjustbox{max width=\linewidth}{%
\begin{tabular}{lccc}
\toprule
Rule & final round & mean over rounds 210--300 & mean over all rounds \\
\midrule
AoI + IPW & +1.77 & +1.76 & +1.83 \\
DFL-AA & +1.51 & +1.47 & +1.39 \\
Limited merging, $L{=}10$ & +2.55 & +2.46 & +2.41 \\
Limited merging, $L{=}2$ & +3.70 & +3.68 & +3.81 \\
\sfg{} weights on counts & +3.66 & +3.65 & +3.84 \\
\sfg{} weights on counts + IPW & +4.27 & +4.19 & +4.15 \\
\sfg{} & +2.44 & +2.51 & +2.83 \\
\sfg{} + IPW & +2.97 & +2.97 & +3.08 \\
\sfg-L & +3.49 & +3.52 & +3.71 \\
\sfg-L + IPW & +4.21 & +4.15 & +4.06 \\
\sfgcut{}, $L{=}1$ & +3.97 & +3.92 & +4.15 \\
\sfgcut{} + IPW & \textbf{+4.39} & \textbf{+4.34} & \textbf{+4.43} \\
\sfg-L + catch-up & +4.65 & +4.80 & +5.69 \\
\sfgcut{} + catch-up & \textbf{+5.06} & \textbf{+5.18} & \textbf{+6.07} \\
\bottomrule
\end{tabular}
}
\end{table}

\paragraph{Who benefits.}
\label{app:quartiles}
\Cref{tab:quartiles} splits the gain over DFedAvg by the client's availability quartile within each run. It covers all processes
with heterogeneous availability.

\begin{table}[h]
\centering
\caption{\textbf{Gain by availability quartile} (pp against DFedAvg). Quartiles are of per-client availability within each run.
Heterogeneous rates: 2 processes $\times$ 3 datasets $\times$ 5 seeds. Real traces: 3 $\times$ 3 $\times$ 5. Bold: best rule in each column.}
\label{tab:quartiles}
\small
\setlength{\tabcolsep}{4pt}
\adjustbox{max width=\linewidth}{%
\begin{tabular}{l cccc cccc}
\toprule
 & \multicolumn{4}{c}{Heterogeneous rates ($u{=}0.6$, $0.77$)} & \multicolumn{4}{c}{Real traces (FLASH, STUNner)} \\
\cmidrule(lr){2-5}\cmidrule(lr){6-9}
Rule & least & Q2 & Q3 & most & least & Q2 & Q3 & most \\
\midrule
AoI & $-$0.18 & $-$0.01 & +0.23 & +1.12 & +0.96 & +1.07 & +1.30 & +1.67 \\
AoI + IPW & +1.33 & +1.38 & +1.51 & +1.83 & +1.18 & +1.27 & +1.35 & +1.52 \\
DFL-AA & +1.46 & +1.60 & +1.75 & +1.65 & +0.84 & +0.75 & +0.74 & +0.99 \\
Limited, $L{=}2$ & +1.83 & +1.92 & +2.12 & +2.47 & +3.79 & +3.92 & +3.85 & +3.98 \\
\sfg{} on counts & +1.25 & +1.46 & +1.88 & +2.32 & +3.74 & +3.96 & +3.83 & +3.90 \\
\sfg{} on counts + IPW & +2.80 & +2.55 & \textbf{+3.01} & +3.12 & +4.29 & +4.36 & \textbf{+4.20} & \textbf{+4.33} \\
\sfg{} & +1.27 & +1.53 & +1.75 & +2.48 & +1.73 & +2.00 & +2.17 & +2.35 \\
\sfg{} + IPW & +2.50 & +2.64 & +2.90 & \textbf{+3.21} & +2.11 & +2.28 & +2.41 & +2.48 \\
\sfg-L & +1.35 & +1.30 & +1.61 & +1.99 & +3.74 & +3.93 & +3.77 & +3.91 \\
\sfg-L + IPW & \textbf{+3.19} & \textbf{+2.78} & +2.99 & +3.15 & +4.28 & \textbf{+4.42} & +4.13 & +4.15 \\
\sfgcut{} & +1.21 & +1.34 & +1.71 & +2.08 & +3.92 & +3.92 & +3.92 & +3.98 \\
\sfgcut{} + IPW & +2.19 & +2.17 & +2.40 & +2.61 & \textbf{+4.45} & +4.34 & +4.17 & +4.21 \\
\bottomrule
\end{tabular}
}
\end{table}

\paragraph{Six days of STUNner.}
\label{app:week}
\Cref{tab:week} repeats the comparison on a six-day STUNner window with 30-minute rounds (Appendix~\ref{app:details}). Outages
there last up to 143 hours. The self-inclusive rules on inherited staleness keep, and extend, their lead over age-based rules.

\begin{table}[h]
\centering
\caption{\textbf{STUNner, six days} (300 rounds of 30 minutes, 3 datasets $\times$ 5 seeds). Mean accuracy (\%), gain over DFedAvg
(pp) and paired wins. Bold: best value in each column, separately for rules without and with catch-up training (extra computation).}
\label{tab:week}
\small
\adjustbox{max width=\linewidth}{%
\begin{tabular}{lccccc}
\toprule
Rule & CIFAR-10 & CIFAR-100 & FEMNIST & Gain (pp) & Wins \\
\midrule
Local only & 62.5 & 21.9 & 72.3 & $-$2.31 & 1/15 \\
DFedAvg & 64.3 & 22.3 & 77.1 & -- & -- \\
Model memory & 49.7 & 10.9 & 24.9 & $-$26.09 & 0/15 \\
Update amplification & 64.1 & 22.5 & 76.7 & $-$0.12 & 8/15 \\
Gossip-learning model age & 66.0 & 24.1 & 79.2 & +1.90 & 15/15 \\
Limited merging, $L{=}2$ & 69.2 & 28.2 & 82.5 & +5.41 & 15/15 \\
Limited merging, $L{=}5$ & 69.0 & 28.0 & 82.2 & +5.19 & 15/15 \\
Limited merging, $L{=}10$ & 68.9 & 27.6 & 82.2 & +5.02 & 15/15 \\
Limited merging, $L{=}50$ & 66.9 & 25.3 & 80.0 & +2.83 & 15/15 \\
\sfg{} weights on counts & 69.2 & 28.2 & 82.4 & +5.38 & 15/15 \\
\sfg{} weights on counts + IPW & \textbf{69.6} & 28.5 & \textbf{82.8} & +5.76 & 15/15 \\
Relative-variance combiner & 62.5 & 21.2 & 73.5 & $-$2.13 & 0/15 \\
AoI & 65.3 & 23.5 & 78.0 & +1.01 & 15/15 \\
DFL-AA & 64.8 & 23.1 & 77.5 & +0.59 & 14/15 \\
AoI + IPW & 65.3 & 23.8 & 78.3 & +1.26 & 15/15 \\
PWG (relative ages) & 66.1 & 24.3 & 79.1 & +1.96 & 15/15 \\
Freshest only ($\gamma{=}0$) & 66.6 & 24.7 & 79.1 & +2.23 & 15/15 \\
Reset on return & 65.1 & 23.2 & 78.0 & +0.90 & 15/15 \\
\sfg{} & 66.6 & 24.9 & 79.2 & +2.34 & 15/15 \\
PWG-L & 68.5 & 26.8 & 81.6 & +4.41 & 15/15 \\
Least-lagged only & 69.1 & 28.1 & 82.3 & +5.28 & 15/15 \\
\sfg-L & 69.2 & 28.0 & 82.4 & +5.30 & 15/15 \\
\sfg-L + IPW & 69.5 & 28.6 & 82.6 & +5.67 & 15/15 \\
\sfgcut{}, $L{=}1$ & 69.2 & 28.1 & 82.3 & +5.31 & 15/15 \\
\sfgcut{}, $L{=}2$ & 69.1 & 28.0 & 82.3 & +5.26 & 15/15 \\
\sfgcut{} + IPW & \textbf{69.6} & \textbf{28.7} & 82.7 & \textbf{+5.79} & 15/15 \\
Catch-up & 65.4 & 23.7 & 78.4 & +1.28 & 15/15 \\
\sfg-L + catch-up & \textbf{69.9} & \textbf{28.8} & \textbf{83.1} & \textbf{+6.02} & 15/15 \\
\bottomrule
\end{tabular}
}
\end{table}

\end{document}